\PassOptionsToPackage{titlenumbered, ruled, algo2e}{algorithm2e}
\PassOptionsToPackage{colorlinks,hyperfootnotes=false}{hyperref}
\PassOptionsToPackage{nosubfloats}{jmlrutils}
\documentclass[12pt]{article} %

\newcommand{\conf}{arxiv}

\input{resources/headers/packages.tex}

\newcommand{\1}{\mathds{1}}

\newcommand{\ol}{\overline}
\newcommand{\id}{\text{id}}

\newcommand{\loss}{\mathbf{l}}

\nojmlr{
  \newtheorem{theorem}{Theorem}[section]
  \newtheorem{proposition}[theorem]{Proposition}
  \newtheorem{corollary}[theorem]{Corollary}
  \newtheorem{lemma}[theorem]{Lemma}
  \newtheorem{remark}[theorem]{Remark}

  \newtheorem{definition}{Definition}
  \newtheorem{example}{Example}
}

\newcommand{\R}{\mathbb{R}}

\newcommand{\N}{\mathbb{N}}
\newcommand*{\mX}{\mathcal{X}}
\newcommand*{\mY}{\mathcal{Y}}
\newcommand*{\mZ}{\mathcal{Z}}
\newcommand{\setX}{{\bm{X}}}
\newcommand{\setY}{{\bm{Y}}}
\newcommand{\setZ}{{\bm{Z}}}

\DeclarePairedDelimiter\floor{\lfloor}{\rfloor}

\newcommand\inner[2]{\langle #1, #2 \rangle}
\DeclarePairedDelimiter\abs{\lvert}{\rvert}
\DeclarePairedDelimiter\norm{\|}{\|}
\newcommand{\indep}{\perp\!\!\!\!\perp} %

\DeclareMathOperator{\supp}{supp} %

\DeclareMathOperator*{\argmax}{arg\,max}

\newcommand*{\prob}{\mathcal{P}}
\renewcommand{\Pr}{\mathbb{P}} %
\newcommand{\E}{\mathbb{E}} %

\NewDocumentCommand{\expect}{ e{^} s o >{\SplitArgument{1}{|}}m }{%
  \E%
  \IfValueT{#1}{{\!}^{#1}}%
  \IfBooleanTF{#2}{%
    \expectarg*{\expectvar#4}%
  }{%
    \IfNoValueTF{#3}{%
      \expectarg{\expectvar#4}%
    }{%
      \expectarg[#3]{\expectvar#4}%
    }%
  }%
}
\NewDocumentCommand{\expectvar}{mm}{%
  #1\IfValueT{#2}{\nonscript\;\delimsize\vert\nonscript\;#2}%
}
\DeclarePairedDelimiterX{\expectarg}[1]{(}{)}{#1}\newcommand*\diff{\mathop{}\!\mathrm{d}}

\DeclarePairedDelimiter\distanceargs{(}{)}

\newcommand{\law}{\mathrm{law}}

\newcommand{\dX}{d_{\setX}}

\NewDocumentCommand{\dW}{e{^} d<> m}{%
  \IfNoValueTF{#1}{d_{\glssymbol{dW}}}{d_{\glssymbol{sinkhorn}}^{#1}}
  \distanceargs*{#3\IfValueT{#2}{;#2}}%
}

\NewDocumentCommand{\dOTC}{d<> m}{%
  d_{\glssymbol{dOTC}}%
  \ifthenelse{\equal{#2}{}}%
  {\IfValueT{#1}{\distanceargs*{\bullet, \bullet; #1}}}%
  {\distanceargs*{#2\IfValueT{#1}{;#1}}}%
}

\NewDocumentCommand{\dWL}{d<> o o m m o}{%
  \IfValueTF{#2}{%
    \IfValueTF{#3}{%
    d_{\glssymbol{dWLepsilondelta}, #2, #3}%
    }{%
    d_{\glssymbol{dWLdelta}, #2}%
    }
  }{%
    d_{\glssymbol{dWL}}%
  }%
  ^{\scriptscriptstyle (#4)}%
  \ifthenelse{\equal{#5}{}}{%
    \IfValueT{#1}{\distanceargs*{\bullet, \bullet; #1}}%
  }{%
  \distanceargs*{#5\IfValueT{#1}{;#1}}%
  }
}

\NewDocumentCommand{\dOTM}{d<> m m}{%
  d_{\glssymbol{dOTM}}^{\scriptscriptstyle #2}
  \ifthenelse{\equal{#3}{}}{%
    \IfValueT{#1}{\distanceargs*{\bullet, \bullet; #1}}%
  }{\distanceargs*{#3\IfValueT{#1}{;#1}}}%
}

\definecolor{darkblue}{rgb}{0.0, 0.0, 0.8}
\definecolor{darkred}{rgb}{0.8, 0.0, 0.0}
\definecolor{darkgreen}{rgb}{0.0, 0.8, 0.0}
\definecolor{purple}{RGB}{153,50,204}

\newcommand{\ready}{Yes} %

\ifthenelse{\equal{\ready}{No}}{%

}%
{

}

\newcommand{\denselist}{\itemsep 0pt\parsep=1pt\partopsep 0pt}

\makeatletter
\newcommand*{\glsplainhyperlink}[2]{%
  \colorlet{currenttext}{.}%
  \colorlet{currentlink}{\@linkcolor}%
  \hypersetup{linkcolor=currenttext}%
  \hyperlink{#1}{#2}%
  \hypersetup{linkcolor=currentlink}%
}
\let\@glslink\glsplainhyperlink
\makeatother

\nojmlr{
\newcommand{\sectionref}[1]{Section~\ref{#1}}
\newcommand{\appendixref}[1]{Appendix~\ref{#1}}
\newcommand{\lemmaref}[1]{Lemma~\ref{#1}}
\newcommand{\algorithmref}[1]{Algorithm~\ref{#1}}
\newcommand{\definitionref}[1]{Definition~\ref{#1}}
\newcommand{\theoremref}[1]{Theorem~\ref{#1}}

\newcommand{\figureref}[1]{Figure~\ref{#1}}
\newcommand{\tableref}[1]{Table~\ref{#1}}
\newcommand{\equationref}[1]{Eq.~\ref{#1}}
}
\newcommand{\propositionref}[1]{Proposition~\ref{#1}} %

\arxivonly{
  \NewDocumentCommand{\email}{m}{\href{mailto:#1}{\texttt{#1}}}
}

\makeglossaries

\NewDocumentCommand{\myglsentry}{mmoom}{%
  \newglossaryentry{#1}{
    text={#3},%
    long={\IfValueT{#4}{#4}},%
    description={#5},%
    first={\IfValueTF{#4}{#4 (#3)}{#3}},%
    name={\ensuremath{d_{\glsentrysymbol{#1}}}},%
    symbol={\ensuremath{\mathrm{#2}}},%
  }
}

\myglsentry{dWL}{WL}[WL distance][Weisfeiler-Lehman distance]{%
The Weisfeiler-Lehman distance as defined by \citet{chen2022weisfeilerlehman}. It is defined by \cref{eq:dwlk}.
\begin{equation*}
\dWL<C>{k}{\mX, \mY}:= \inf_{(X_t,Y_t)_{t\in\N} \in \Pi(\mX, \mY)} \mathbb{E}\,C(X_k, Y_k).
\end{equation*}
More details in Section~\ref{par:wl}.%
}

\myglsentry{dOTC}{OTC}[OTC distance][Optimal Transport Coupling distance]{%
The Optimal Transport Coupling distance as defined by \citet{o2022optimal}. It is defined in \cref{eq:dOTC}.
\begin{equation*}
 \dOTC<C>{\mathcal{X},\mathcal{Y}}:=\inf_{\substack{(X_t,Y_t)_{t\in\N}\in\Pi_\mathrm{H}(\mathcal{X},\mathcal{Y})\\
\law((X_0,Y_0))\text{ is stationary}}} \mathbb{E}\, C(X_0,Y_0),
\end{equation*}
More details in Section~\ref{par:dOTC}.%
} 

\myglsentry{dOTM}{OTM}[OTM distance][Generalized Optimal Transport Markov Distance]{%
Generalized Optimal Transport Markov Distances are a class of 
    Markov distances we define that encompasses or has as limit points 
    \glspl{dWL}, the \gls{dOTC}, and \glspl{dWLdelta}.
    More details in Section~\ref{generalized-optimal-transport-Markov-distances}.
    They are parameterized by a distribution $p$ over the integers, 
    and defined by \cref{eq:dotm}.
    \begin{equation*}
\dOTM{p}{\mX, \mY} = \inf_{(X_t, Y_t)_{t\in\N}}
  \E\, C(X_T,Y_T),
    \end{equation*}
    where $T\sim p$.%
}

\newglossaryentry{dW}{%
  name={\ensuremath{d_{\glsentrysymbol{dW}}}},
  text={Wasserstein distance},
  description={The Wasserstein distance is defined as the solution of the optimal transport between two measures with a cost matrix, in \cref{eq:wassersteindef}.
\begin{equation*}\dW<C>{\alpha,\beta}:=\inf_{(X,Y)}\E\; C(X,Y)\end{equation*}
More details in Section~\ref{par:wasserstein}.
  }, 
  symbol={\ensuremath{\mathrm{W}}}
}
\newcommand*{\wasserstein}{\gls{dW}}

\newglossaryentry{sinkhorn}{%
  name={\ensuremath{d_{\glsentrysymbol{dW}}^\epsilon}},
  text={Sinkhorn distance},
  description={The Sinkhorn distance is the entropy-regularized 
    version of the \gls{dW}. 
    We use it because of its smoothness properties. it is also faster to compute than the \gls{dW}. It is defined in \cref{def:regOT} of the Appendix.
\begin{equation*}
\dW^{\epsilon}<C>{\alpha,\beta}:=\min_{(X, Y)\in \mathcal{C}(\alpha, \beta)} \E\; C(X,Y) - \epsilon H(X, Y) 
\end{equation*}
  }, 
  symbol={\ensuremath{\mathrm{W}}},
}

\newglossaryentry{dWLdelta}{%
    name={\ensuremath{d_{\glsentrysymbol{dWLdelta}, \delta}}}, 
    long={\(\delta\)-discounted Weisfeiler-Lehman distance}, 
    text={\(\delta\)-discounted WL distance}, 
    symbol={\ensuremath{\mathrm{WL}}}, 
    description={ Our \(\delta\)-discounted WL distance. 
      It is a regularization of the original \gls{dWL}.
      More details in Section~\ref{sec:wldelta}.
      They are defined as a parametric class of \gls{dOTM}, 
      parameterized by the distributions defined in Section~\ref{subsec:dwldelta}.
      For $k \in \N$:
      \begin{equation*}
      \dWL[\delta]{k}{\mathcal{X},\mathcal{Y}} := 
      \inf_{(X_t,Y_t)_{t\in\N}\in\Pi(\mathcal{X},\mathcal{Y})} 
        \expect*{ \sum_{t = 0}^{k-1} \delta (1 - \delta)^t C(X_t,Y_t)
        +(1 - \delta)^k C(X_{k},Y_{k})}
      \end{equation*}
      and 
      \begin{equation*}
      \dWL[\delta]{\infty}{\mathcal{X},\mathcal{Y}} := 
      \inf_{(X_t,Y_t)_{t\in\N}\in\Pi(\mathcal{X},\mathcal{Y})} 
      \expect*{ \sum_{t = 0}^{\infty} \delta (1 - \delta)^t C(X_{t},Y_{t})}
      \end{equation*}
     }
}

\newglossaryentry{dWLepsilondelta}{%
    name={\ensuremath{d_{\glsentrysymbol{dWLdelta},\delta,\epsilon}}}, 
    text={Entropy-regularized $\delta$-discounted WL distance}, 
    symbol={\ensuremath{\mathrm{WL}}}, 
    description={ The \gls{dWLepsilondelta} is obtained by replacing all 
      \glspl{dW} by \glspl{sinkhorn} in the computation of \gls{dWLdelta}.
      This operation makes our \gls{dWLdelta} into a smooth distance that can 
      be used for learning, using the formulae developed in Section~\ref{differentiation}.
      It is defined in Definition~\ref{def:entropy dWL}.
  }
}

\newglossaryentry{sets}{%
    name={\ensuremath{\setX, \setY, \setZ}}, 
    text={sets}, 
    symbol={\ensuremath{\setX}}, 
    description={In this paper, we use boldface letters, such as $\setX$, to denote finite sets.
  }
}

\newglossaryentry{markovchains}{%
  name={\ensuremath{\mX, \mY, \mZ, \nu^{\setX}, \nu^{\setY}, \nu^{\setZ}, m^{\setX}_\bullet, m^{\setY}_\bullet, m^{\setZ}_\bullet}}, 
    text={Markov chain}, 
    description={We denote
      $\mX$ (resp. $\mY$) a Markov chain over \glssymbol{sets} (resp. $\setY$). 
    It is defined by its initial distribution $\nu^{\setX}$ (resp. $\nu^{\setY}$) and its transition kernel $m^{\setX}_\bullet$ (resp. $m^{\setY}_\bullet$).
  }
}

\newglossaryentry{markovrealisations}{%
  name={\ensuremath{(X_t)_{t\in\N}, (Y_t)_{t\in\N}, (Z_t)_{t\in\N}}}, 
    text={realization}, 
    description={We denote $(X_t)_{t\in\N}$ (resp. $(Y_t)_{t\in\N}$) a realisation of $\mX$ (resp. $\mY$)
}
}

\newglossaryentry{costmatrix}{%
  name={\ensuremath{C}}, 
    text={cost function}, 
    description={We denote $C: \setX \times \setY \to \R_+$ a cost function.
}
}

\title{Distances for Markov Chains, and Their Differentiation}

\author{Tristan Brugère\footnotemark[1]~\footnotemark[2] \\ \email{tbrugere@ucsd.edu} 
\and Zhengchao Wan\footnotemark[2] \\ \email{zcwan@ucsd.edu}
\and Yusu Wang\footnotemark[2] \\ \email{yusuwang@ucsd.edu}
}

\begin{document}

\maketitle
\renewcommand{\thefootnote}{\fnsymbol{footnote}}
\footnotetext[1]{corresponding author}
\footnotetext[2]{
Halıcıoğlu Data Science Institute —
University of California, San Diego —
9500 Gilman Dr. La Jolla, CA 92093
}
\renewcommand{\thefootnote}{\arabic{footnote}}
\setcounter{footnote}{0}

\begin{abstract}%
    (Directed) graphs with node attributes are a common type of data in various applications and there is a vast literature on developing metrics and efficient algorithms for comparing them. Recently, in the graph learning and optimization communities, a range of new approaches have been developed for comparing graphs with node attributes, leveraging ideas such as the Optimal Transport (OT) and the Weisfeiler-Lehman (WL) graph isomorphism test. Two state-of-the-art representatives are the OTC distance proposed in \citep{o2022optimal} and the WL distance in \citep{chen2022weisfeilerlehman}. Interestingly, while these two distances are developed based on different ideas, we observe that they both view graphs as Markov chains, and are deeply connected. 
Indeed, in this paper, we propose a unified framework to generate distances for Markov chains (thus including (directed) graphs with node attributes), which we call the \emph{Optimal Transport Markov (OTM)} distances, that encompass both the OTC and the WL distances. We further introduce a special one-parameter family of distances within our OTM framework, called the \emph{discounted WL distance}. We show that the discounted WL distance has nice theoretical properties and can address several limitations of the existing OTC and WL distances. Furthermore, contrary to the OTC and the WL distances, our new discounted WL distance can be differentiated {\bf after} a entropy-regularization similar to the Sinkhorn distance, making it suitable to use in learning frameworks, e.g., as the reconstruction loss in a graph generative model. 

\end{abstract}

\section{Introduction}
\label{section:introduction}

Graph data is ubiquitous across various application domains, e.g., molecules viewed as node-attribute graphs, citation networks as directed graphs. 
Developing metrics and efficient algorithms to compare them have been traditionally studied in fields such as graph theory and theoretical computer science. 
In the last two decades, this problem also received tremendous attention in the graph learning and optimization community, especially for comparing (directed) graphs with node attributes (which we will call {\bf labeled graphs}). 
In particular, two ideas have become prominent in the modern treatment of labeled graphs. The first idea is to leverage the so-called Weisfeiler-Lehman (WL) graph isomorphism test \citep{leman1968reduction}, which is a classic graph isomorphism test that, in linear time, can distinguish a large family of graphs \citep{babai1979canonical,babai1983canonical}. It has recently gained renewed interest both in designing WL-inspired graph kernels \citep{shervashidze2011weisfeiler,togninalli2019wasserstein} and as a tool for analyzing Message Passing Graph Neural Networks (MP-GNNs) \citep{xu2018powerful,azizian2020expressive}. 
The second idea in modern treatment of graphs is to treat labeled graphs (or related structured data) as suitable discrete measure spaces and then use the idea of Optimal Transport (OT) to compare them. Examples include the Wasserstein WL (WWL) kernel \citep{togninalli2019wasserstein}, the Fused Gromov-Wasserstein (FGW) distance \citep{titouan2019optimal}, and the WL test based tree-mover distance \citep{chuangtree}.

Very recently, several studies took a less combinatorial approach and viewed graphs as Markov chains: \cite{chen2022weisfeilerlehman} introduced the Weisfeiler-Lehman (WL) distance, which generalizes the graph comparing problem to the Markov chain comparison problem through a WL-like process in a natural way. This distance has been found to be more discriminative than the previously popular WWL graph kernel.
Around the same time, the optimal transition coupling (OTC) distance was proposed by \citet{o2022optimal} for comparing stationary Markov chains, i.e., Markov chains with stationary distributions and the study was followed by \cite{yi2021alignment} with applications in comparing graphs.
 
The WL distance proposed in \citep{chen2022weisfeilerlehman} and the OTC distance in \citep{o2022optimal} represent two SOTA approaches in comparing labeled graphs (i.e., graphs with node attributes). In fact, both of them compare more general Markov chains like objects. The Markov chain perspective not only relieves the difficulty in handling combinatorial structures of graphs but also provides a natural and unified way of modelling both directed and undirected graphs. To broaden the use of these distances, especially in graph learning and optimization (e.g., to use such distance as graph reconstruction loss in a generative model), it is crucial that we are able to differentiate such distances w.r.t. changes in input graphs. However, differentiating these distances appears to be challenging.

\paragraph{Our contributions.}
We propose in \sectionref{generalized-optimal-transport-Markov-distances}  a unified framework to generate distances between Markov chains (and thus also for labeled graphs), which we call the \emph{Optimal Transport Markov (OTM)} distances. 
This framework of OTM distances encompasses both the WL distance and the OTC distance and in particular, we prove that the two distances serve as extreme points in the family of OTM distances.
We further identify a special one-parameter family of distances within our general framework of OTM distances, and we call our new distance \emph{the \gls{dWLdelta}} (for a parameter $\delta \in[0,1]$) in \sectionref{sec:wldelta}. 
Not only do we unveil succinct connections between our discounted WL distance and both the WL and the OTC distances, but we also show that the discounted WL distance has better theoretical properties than the other two distances:
\begin{enumerate}\denselist
    \item  Contrary to the WL and the OTC distances, the discounted WL distance can be used to compare non-stationary Markov chains.
    \item The discounted WL distance has the same discriminative power as the OTC distance and possibly stronger discriminative power than the WL distance. %
    \item All the three types of distances are computed via iterative schemes. We devise an algorithm of the discounted WL distance which converges provably faster than the one for the WL distance introduced in \citep{chen2022weisfeilerlehman}; whereas to the best of our knowledge, there is no known study on convergence rate of the OTC distance. 
    \item Furthermore, contrary to both the OTC and the WL distances, a regularized version of the \gls{dWLdelta}s 
can be differentiated against its parameters, 
enabling a range of possible applications as a loss in machine learning or in other optimization tasks.
In \sectionref{differentiation}, we give a simple formula to compute its gradients.
\end{enumerate}

Note that the effectiveness of the WL distance was already shown in \citep{chen2022weisfeilerlehman} where it compared favorably with other graph kernels. Our discounted WL distance is provably more discriminative (e.g, Proposition \ref{prop:dwl_lower_bound}), and thus we expect it will lead to even better practical performance.

\paragraph{Relation to the fused-GW (FGW) distance of \citet{titouan2019optimal}.} 
The fused-GW (FGW) distance also leverages the optimal transport idea, and in fact, uses the Gromov-Wasserstein distance to compute two graphs (equipped with metric structures at nodes). The authors also developed a heuristic algorithm to approximate this algorithm in practice. While the algorithms work well in practice \citep{vincent2021online,titouan2019optimal}, there are no theoretical guarantees for them and in fact, the FGW algorithm is only proven to converge to a local minimum (of a provably non-convex function). Current methods on optimizing to minimize FGW as a loss relies on a kind of block coordinate descent, updating alternatively the
OT matching (using the FGW algorithm) the parameters by gradient descent with fixed matching~\citep{vincent2021online,brogat2022learning, xia2023implicit}.  
In contrast, we can compute our $\delta$-discounted WL distance –and its gradient in the case of the regularized version– exactly, allowing us to easily optimize it. 
We further remark that the FGW distance and our OTM distance adopt fundamentally different points of view: FGW stems from the interpretation of graphs as metric spaces, whereas OTM distances rise from random walks on graphs viewed as probabilistic objects (Markov Chains).

\section{Preliminaries}

We include an appendix within the supplementary material that contains all the detailed proofs, algorithms and experimental details.
We provide a Glossary of all notations we use in \appendixref{main}. 

\subsection{Probability Measures and Markov Chains}\label{sec:markovchains}
In this paper, we use boldface letters, such as $\setX$, to denote finite \gls{sets}. We let $\mathcal{P}(X)$ to denote the space of all probability measures on $\setX$.

A \emph{finite \gls{markovchains}} $\mX = (\setX,m_\bullet^{\setX, (\bullet)},\nu^{\setX})$ consists of 
a finite state space $\setX$, 
a Markov transition kernel  $m^{\setX, (\bullet)}_\bullet:x, t \in \setX \times \N \rightarrow m^{\setX, (t)}_x \in \mathcal{P}(\setX)$,
and an initial distribution $\nu^{\setX}$.
A \emph{\gls{markovrealisations} of a Markov chain} $\mX$ is a sequence of random variables $(X_t:\Omega\rightarrow\setX)_{t\in\N}$ on a common probability space $(\Omega,\mathbb{P})$ such that
$\mathrm{law}(X_0)=\nu^{\setX}$
and $\Pr(X_{t+1}=x'|X_t=x)=m_x^{\setX, (t)}( x')$ for any $x,x'\in\setX.$ 
If $m^{\setX, (t)}_\bullet$ is independent of the time $t$, we call the chain $\mX$ \emph{time homogeneous}. In that case, we will omit the $t$ parameter and write $m^{\setX}_\bullet$. 
We will also use the notation $m_{xx'}^{\setX}:=m_x^{\setX}( x')$ for compactness later.
If the initial distribution $\nu^{\setX}$ is stationary, then we call $\mX$ a \emph{stationary Markov chain}.\looseness-1

\paragraph{Couplings.}
Let $\setX, \setY$ be two finite sets and let $\alpha \in \prob(\setX), \beta\in\prob(\setY)$. We call $\mu\in\mathcal{P}(\setX\times \setY)$ \emph{a coupling between $\alpha$ and $\beta$} if for any $A\subseteq \setX$ and $B\subseteq \setY$, one has that $\mu(A\times \setY)=\alpha(A)$ and $\mu(\setX\times B)=\beta(B)$. We let $\mathcal{C}(\alpha,\beta)$ denote the set of all couplings between $\alpha$ and $\beta$.
Couplings can be also interpreted via random variables. Let $X:\Omega\rightarrow\setX$ and $Y:\Omega\rightarrow\setY$ be two random variables from some same probability space $(\Omega,\Pr)$ into the spaces $\setX$ and $\setY$, respectively, such that $\law(X)=\alpha$ (we also write $X\sim\alpha$) and $\law(Y)=\beta$. Then, it is easy to check that $\law((X,Y))\in\mathcal{C}(\alpha,\beta)$ and that any coupling in $\mathcal{C}(\alpha,\beta)$ can be obtained in this way. We hence also write $(X,Y)\in\mathcal{C}(\alpha,\beta)$.

\paragraph{Markovian couplings.}
Given two Markov chains $\mathcal{X}$ and $\mathcal{Y}$, a stochastic process $(X_t,Y_t)_{t\in\N}$ on a probability space $(\Omega,\Pr)$ is \emph{a Markovian coupling} \citep{chen2023wl} between them if 
\begin{itemize}\denselist
    \item $(X_t,Y_t)_{t\in\N}$ satisfies the (time inhomogeneous) Markov property: for any $t\in\N$,
    \begin{align*}
    \Pr((X_{t+1},Y_{t+1})|(X_t,Y_t),\ldots,(X_0,Y_0))
    &=\Pr((X_{t+1},Y_{t+1})|(X_t,Y_t)).
    \end{align*}
    \item For any $t\in\N$, any $x\in \setX$ and $y\in \setY$, $m_{xy}^{\setX\setY,(t)}$ defined below belongs to $\mathcal{C}(m_x^{\setX},m_y^{\setY})$: 
    \[m_{xy}^{\setX\setY,(t)}:=\Pr((X_{t+1},Y_{t+1})|(X_t,Y_t)=(x,y))\in\mathcal{P}(\setX\times \setY).\]  
    \item The initial distribution is a coupling: $\mathrm{law}((X_0,Y_0))\in\mathcal{C}(\nu^{\setX},\nu^{\setY})$.
\end{itemize}
We let $\Pi(\mathcal{X},\mathcal{Y})$ denote the collection of all Markovian couplings. Given a Markovian coupling $(X_t,Y_t)_{t\in\N}$, if for each $t\in\N$ and each $x\in \setX$ and $y\in \setY$, one has that $m_{xy}^{\setX\setY,(t)}=m_{xy}^{\setX\setY,(1)}$, then we say $(X_t,Y_t)_{t\in\N}$ is a \emph{time homogeneous} Markovian coupling. We denote by $\Pi_\mathrm{H}(\mathcal{X},\mathcal{Y})$ the collection of all time homogeneous Markovian couplings {w.r.t. $\mathcal{X}$ and $\mathcal{Y}$.}

\subsection{Optimal Transport and Distances between Markov Chains}
\label{subsec:OPT} 
\paragraph{The \gls{dW}.}\label{par:wasserstein}

Let $\setX$ and $\setY$ be two finite sets. Assume that $\alpha, \beta$ are probability measures over $\setX$ and $\setY$, respectively. 
We call any function $C: \setX \times \setY \to \R_+$ a \emph{\gls{costmatrix}} between $\setX$ and $\setY$. 
Then, the \emph{Optimal Transport (OT) distance} between $\alpha$ and $\beta$ is defined as follows:
\begin{equation}\dW<C>{\alpha,\beta}:=\inf_{(X,Y) \in \mathcal{C}(\alpha,\beta)}\E\; C(X,Y),\label{eq:wassersteindef}\end{equation}
where $\E$ denotes the expectation.
When $\setX=\setY$ and $C:=d_\setX$ is a distance function on $\setX$, the quantity $\dW{\alpha,\beta;d_\setX}$ is also called the \emph{Wasserstein distance} between $\alpha$ and $\beta$ as $d_\mathrm{W}$ is a metric distance on  $\mathcal{P}(\setX)$.

\paragraph{The \Gls{dWL}.}\label{par:wl}
Consider two finite \emph{stationary} Markov chains (i.e., Markov chains with stationary initial distributions) $\mX$ and $\mY$, with a cost function \(C: \setX \times \setY \rightarrow \R_+\).
{Inspired by \citep{chen2022weisfeilerlehman},} given any $k\in\N$, the \emph{depth-$k$ Weisfeiler-Lehman (WL) distance} between them
is defined as follows:

\begin{equation}\label{eq:dwlk}
\dWL<C>{k}{\mX, \mY}:= \inf_{(X_t,Y_t)_{t\in\N}} \mathbb{E}\,C(X_k, Y_k), 
\end{equation}
where the infimum is taken over all possible Markovian couplings $(X_t,Y_t)_{t\in\N}\in\Pi(\mathcal{X},\mathcal{Y})$.
Then, the Weisfeiler-Lehman distance is defined as
\begin{equation}\label{eq:alternatewldefinition}
\dWL<C>{\infty}{\mX, \mY}:=\sup_{k\in\N} \dWL<C>{k}{\mX,\mY}.
\end{equation}

\begin{remark}[Nuance in definition]\label{rmk:different dwl}
The definition of the WL distance above is based on a characterization of the WL distance in \cite{chen2023wl}. 
The (depth-$k$) WL distance was originally inspired by the classical Weisfeiler-Lehman graph isomorphism test. Specifically, the depth-$k$ WL distance was designed to emulate the $k$th iteration of the WL test\footnote{In the WL test literature, the index $k$ typically denotes the order of the test. However, in this context, we use it to denote the depth.}.
Our definition is slightly more general than the one in \cite{chen2023wl}: the (depth-$k$) \gls{dWL} was originally defined between two Markov chains endowed with label functions $\ell_{\setX} : \setX \to \mathbf{Z}$ and $\ell_{\setY}: \setY \to \mathbf{Z}$ into a common metric space $(\mathbf{Z},d_{\mathbf{Z}})$. Using our language, this is equivalent to saying that the cost function involved is of the form $C(x,y) = d_{\mathbf{Z}}(l_{\setX}(x), l_{\setY}(y))$.
\end{remark}

The stationary assumption is redundant when $k<\infty$, and hence $\dWL<C>{k}{\mX,\mY}$ is defined for any Markov chains. When $k=\infty$, the situation becomes subtle; see more discussion in \appendixref{sec: more on dWL}.

\paragraph{The \gls{dOTC}.}\label{par:dOTC}
Consider two finite stationary Markov chains $\mathcal{X}$ and $\mathcal{Y}$. 
Let $C:\setX\times \setY\rightarrow \R_+$ be any cost function.  
Then, the \emph{optimal transition coupling (OTC) distance} \citep{o2022optimal}, which we denote by $\dOTC{}$, is defined as 
\begin{equation}\label{eq:dOTC}
 \dOTC<C>{\mathcal{X},\mathcal{Y}}:=\inf_{\substack{(X_t,Y_t)_{t\in\N}\in\Pi_\mathrm{H}(\mathcal{X},\mathcal{Y})\\
\law((X_0,Y_0))\text{ is stationary}}} \mathbb{E}\, C(X_0,Y_0),
\end{equation}
where infimum is over \emph{time homogeneous} Markovian couplings with \emph{stationary} initial distributions.

\begin{remark}[A note on symbols]
For simplicity of presentation, in what follows, we will sometimes omit the cost matrix $C$ in our notation of distances when its choice is clear.  
For example, we may write $\dWL{k}{\mX, \mY}$ instead of $\dWL<C>{k}{\mX, \mY}$. 
\end{remark}
\section{Optimal Transport Markov Distances}\label{generalized-optimal-transport-Markov-distances}
Note the similarity between \cref{eq:dwlk} and \cref{eq:dOTC}: they are both infimizing certain expected costs between random walk paths.
Motivated by this similarity, in this section, we devise a general framework for constructing distances between Markov chains with \emph{arbitrary} initial distributions in contrast to the stationary condition for the two distances mentioned above. 
We will show how this framework incorporates the (depth-$k$) \gls{dWL} and admits the \gls{dOTC} as a limit point, and how these two distances appear as lower and upper bounds for this family of distances.
All proofs of results in this section can be found in \appendixref{sec:OTM proofs}.

In what follows, we assume $\mathcal{X}$ and $\mathcal{Y}$ are two (not necessarily stationary) finite Markov chains endowed with any cost function $C:\setX\times \setY\rightarrow \R_+$. 

\begin{definition}[\gls{dOTM}]\label{def:otm_distance} 
Let \(p\in\mathcal{P}(\N)\) be a distribution over all non-negative integers, $T\sim p$ be a random variable.
We define the \emph{Optimal Transport Markov (OTM) distance}
associated to \(p\), between two Markov chains \(\mX\) and \(\mY\) as:
\begin{equation}
\dOTM<C>{p}{\mX, \mY} = \inf_{(X_t, Y_t)_{t\in\N}}
  \E\, C(X_T,Y_T),
\label{eq:dotm}\end{equation}
where the infimum is taken over all Markovian couplings $(X_t, Y_t)_{t\in\N}$ \emph{independent} of $T$.
\end{definition}

\begin{remark}[Optimal Markovian couplings exist]\label{rmk: optimal markovian coupling}
    In fact, the infimum in \cref{eq:dotm} can be replaced by a \emph{minimum}: there exists a Markovian coupling 
    $(X_t, Y_t)_{t\in\N}$ such that 
    $\dOTM<C>{p}{\mX, \mY} = \E\, C(X_T,Y_T)$. 
    We refer the reader to \appendixref{proof:optimalmarkovcoupling} for a proof.
\end{remark}

We also remark that, just as the Optimal Transport problem gives rise to the \wasserstein{} between probability measures on the same underlying metric space, the OTM distance above becomes a pseudometric on the collection of Markov chains with a pseudometric space $(\setX,d_\setX)$ being their common state space; we refer the reader to \appendixref{sec:OTM is a metric} for details.

\begin{example}[\(\dWL{k}{}\) is an OTM distance]\label{example:dwlk_otm} 
Let $\delta_k$ denote the Dirac delta measure at $k\in\N$. Then, it is obvious that
$\dWL{k}{\mathcal{X}, \mathcal{Y}} = \dOTM{\delta_k}{\mathcal{X}, \mathcal{Y}}.$
In this way, although \(\dWL{\infty}{}\) is not an instance of OTM
distances, it is actually the limit of a sequence of OTM distances.
\end{example}

Besides the example above, one can establish the following bounds for the OTM distance utilizing the (depth-$k$) \gls{dWL} and the \gls{dOTC}. 

\begin{proposition}[A $\dWL{k}{}$-based lower bound]\label{prop:dwl_lower_bound}
For any distribution $p$ on $\N$, one has that 
\begin{equation*}
\dOTM{p}{\mX, \mY} \geq \E_{T\sim p}({\dWL{T}{\mX, \mY}})=\sum_{k\in\N} p(k)\dWL{k}{\mX, \mY}.
\end{equation*}
\end{proposition}
\begin{proposition}[$\dOTC{}$ is an upper bound]\label{prop:dotc_upper_bound}
For all distributions $p$ on $\N$, and \emph{stationary} Markov chains $\mX, \mY$, one has that:
\(
\dOTM{p}{\mX, \mY} \leq \dOTC{\mX, \mY}.
\)
\end{proposition}

The above two propositions suggest that the (depth-$k$) WL-distance and the \gls{dOTC} can be viewed as the two extremes of the family of OTM distances. %
In fact, in \cref{sssection:interpolation} later, we will show that the OTC distance turns out to be a \emph{tight} upper bound of the family of OTM distances. Furthermore, although the \gls{dOTC} serves as an upper bound, the following result states that a large family of OTM distances has the same discriminative power as the \gls{dOTC}.
\begin{proposition}[Zero-sets]\label{thm:zerosets}
Suppose that $\mX, \mY$ are \emph{stationary} and that $p$ is fully supported (i.e., $\forall t \in \N , p(t) > 0$). 
Then, $\dOTC{\mX, \mY} = 0$ iff $\dOTM{p}{\mX, \mY} = 0$, implying that these two distances ``distinguish'' the same sets of stationary Markov chains.
\end{proposition}

The OTM distances have many interesting theoretical properties. For example, any OTM distance is indeed a ``pesudo-distance'' satisfying the triangle inequality under suitable conditions (cf. \propositionref{prop:OTM_is_distance}). Furthermore, the OTM distance is continuous with respect to the probability measure $p$ (cf. \lemmaref{lm:convergence of dOTM}). Interested readers are referred to \appendixref{sec:OTM_details} for more details.

We also refer the reader to the discussion provided in \appendixref{sec:OTM_computation} for a general way of computing OTM distances with finitely-supported distribution $p$, and a detailed view of how exactly the geometric distribution – and specifically its memoryless property – is used to obtain the simplified computation. This motivates our study of the \emph{discounted WL distances} (whose distribution $p$ is a geometric law) in \sectionref{sec:wldelta}, which will be our focus in the rest of the paper. 
These distances can be computed more easily than general OTM distances by using a fixed-point algorithm. Furthermore, an entropy-regularized version of them can be differentiated.
These nice properties make them a natural choice of OTM distances for applications which we explore in \Cref{sec:experiment}. 

\section{The Discounted WL Distance}\label{sec:wldelta}

The OTM distances we introduced above encompass the WL distance and the OTC distance at the two extremes, and are a (significant) generalization of both distances. However, one might wonder why it is useful to consider this general formulation. 
In this section we first note some limitations of the two distances and then propose the \emph{discounted WL distance}, which is a special instance of our OTM distance, as a remedy of those limitations. Indeed, we will see that the discounted WL distance can compare more general Markov chains, is more efficient to compute, and more importantly, has a relaxed form that can be differentiated (whereas the \gls{dWL} and the \gls{dOTC} cannot). 

All missing proofs from this section are in \appendixref{seq:wldeltaproofs}.

\subsection{Limitations of the WL Distance and the OTC Distance}\label{sec: limitation}

\paragraph{Stationary initial distributions.} Both the WL distance and the \gls{dOTC} are only defined for Markov chains with stationary initial distributions.
This assumption is quite limited and, in general, does not even accommodate the uniform measure, which assigns equal weight to all states.
Note that while the definition of the (depth-$k$) WL distance could be extended to Markov chains with any initial distributions, on irreducible and aperiodic Markov chains, depth-$\infty$ WL distance turns out to be independent on the initial distribution. Hence, the extension of the WL distance to non-stationary case is meaningless; see \appendixref{sec: more on dWL} for more details.

\paragraph{Rate of convergence.} 
The depth-$k$ WL distance converges as $k\rightarrow\infty$ (see our discussion in \appendixref{sec:convergence_wl}) provided that the cost is of the form given in Remark~\ref{rmk:different dwl}. However, this convergence is based on the convergence to a non-unique fixed point of a map. Due to the non-uniqueness feature, this convergence may be vulnerable to numerical errors.
We have further established an estimate of the rate of convergence for $\dWL{k}{}\rightarrow\dWL{\infty}{}$ as $k\rightarrow \infty$. Although this limit converges exponentially fast, the rate depends on the input Markov chains. See \appendixref{sec:convergence_wl} for details.\looseness-1

The algorithm for computing the OTC distance is through the policy iteration on an average-cost Markov Decision Process \citep{o2022optimal}. Although the policy iteration terminates in a finite number of iterations, as far as we know, this number does not have a reasonable bound (except of the obvious upper bound of \(n_\text{actions} ^ {n_\text{states}}\) of the number of possible policies)\looseness-1

This all together motivates us to seek for a distance that can be computed via a stable iterative algorithm which can provably converge faster than the one for the WL distance.

\paragraph{Differentiation.}
To the best of our knowledge, there is no known way to compute the derivative of those two distances with respect to input Markov chains and cost functions. 
Although the WL distance can be formulated as a fixed point to a certain map, this map lacks the desired contracting property to guarantee uniqueness and smoothness of fixed points. For the OTC distance, {differentiating it seems to be even more challenging.} Further we are not aware of any way to formulate the OTC distance to a Banach fixed point and thus our strategy for differentiating the discounted WL distance to be introduced (in \sectionref{differentiation}) does not apply to differentiate the OTC distance.

\subsection{The $\delta$-Discounted WL Distance}\label{subsec:dwldelta}
We now introduce a one parameter family of instances of OTM distances which has close relationship with the WL distance and the \gls{dOTC}. This family of distances addresses those limitations mentioned above in \Cref{sec: limitation}. 

The distance that we define next, called the \gls{dWLdelta}, is essentially a regularized version of the WL distance {(this view will be more evident by considering Proposition~\ref{prop:wlreg_recursive} later)}.
This regularization enables us to compute the new distance via solving a Banach fixed point problem. This approach is very tractable, addressing the limitations of the original WL distance, 
and providing the ability to differentiate our distance.

For the purpose of introducing our discounted WL distance, we consider the following two types of distributions on $\N$ given $\delta\in[0,1]$. %

\noindent\emph{Geometric distribution $p_\delta^\infty$:} if we let a RV \(T^\infty_\delta \sim p_\delta^\infty\), then \(\Pr(T_\delta^\infty = t) = \delta (1-\delta)^t,\,\forall t\in\N;\)

\noindent\emph{Truncated geometric distribution $p_\delta^k$ for $k\in\N$:} let \(T^k_\delta :=\min(T^\infty_\delta,k) \), then $T^k_\delta\sim p_\delta^k$: 

\[\Pr(T^k_\delta = t) = \begin{cases}
  \delta(1-\delta)^t,& t\leq k-1\\
  (1-\delta)^k,&t=k\\
  0, &t>k
\end{cases}.\]

\begin{definition}[\gls{dWLdelta}]\label{corollary:probabilistic} 
For any $k\in\N\cup\{\infty\}$,
the \emph{depth-$k$ \gls{dWLdelta}} is defined as follows %
 \begin{equation}\label{eq:wlregdef}
\dWL<C>[\delta]{k}{\mathcal{X},\mathcal{Y}} := 
\dOTM<C>{p_\delta^k}{\mX, \mY}~~\text{and more explicitly,}
\end{equation}
\[
\dWL[\delta]{k}{\mathcal{X},\mathcal{Y}} := 
\inf_{(X_t,Y_t)_{t\in\N}\in\Pi(\mathcal{X},\mathcal{Y})} 
\begin{cases}\expect*{ \sum_{t = 0}^{k-1} \delta (1 - \delta)^t C(X_{t},Y_{t})+(1 - \delta)^k C(X_{k},Y_{k})}, &k<\infty\\
\expect*{ \sum_{t = 0}^{\infty} \delta (1 - \delta)^t C(X_{t},Y_{t})}, &k=\infty
\end{cases}.
\]
\end{definition}

\begin{remark}[$k=\infty$]\label{rmk:bicaual ot}
$\dWL[\delta]{\infty}{}$ is closely related to the \emph{bicausal optimal transport} distance from
    \cite{moulos2021bicausal}: the bicausal optimal transport distance, with a discount factor of $(1-\delta)$ and a \emph{binary} cost matrix $C$, is the same as our $\dWL[\delta]{\infty}{}$ up to a multiplicative constant $\delta$.
\end{remark}

\begin{remark}[$\delta=0$]\label{rmk:delta0}
    Note that for any finite $k$, $\dWL[0]{k}{\mathcal{X},\mathcal{Y}}=\dWL{k}{\mathcal{X},\mathcal{Y}}$.
\end{remark}

The $\delta$-discounted WL distance behaves nicely when $k$ approaches $\infty$: 
\begin{proposition}[Convergence w.r.t. $k$]\label{prop:wlreginfty}
For any Markov chains $\mathcal{X}$ and $\mathcal{Y}$, any cost function $C$, and any $\delta\in[0,1]$, one has that  $\dWL[\delta]{\infty}{\mathcal{X},\mathcal{Y}} = \lim_{k\rightarrow\infty} \dWL[\delta]{k}{\mathcal{X},\mathcal{Y}}.$
\end{proposition}
We will also establish later in Proposition \ref{prop:wlreg_recursive} a convergence rate result.

{A second nice property of this distance is,} although defined via general Markovian couplings, an optimal Markovian coupling %
can be chosen to be \emph{time-homogeneous} for $\dWL[\delta]{\infty}{}$.
\begin{proposition}[Optimal Markovian coupling]\label{prop: optimal coupling} 
Recall that $\Pi_\mathrm{H}(\mathcal{X},\mathcal{Y})$ denotes the collection of all time homogeneous Markovian couplings between $\mX$ and $\mY$. Then, for any $\delta> 0$, one has that
 \[
\dWL[\delta]{\infty}{\mathcal{X},\mathcal{Y}} = 
\min_{(X_t,Y_t)_{t\in\N}\in\Pi_\mathrm{H}(\mathcal{X},\mathcal{Y})}
\E\, C(X_{T_\delta^\infty},Y_{T_\delta^\infty}).
\]    
\end{proposition}

\subsection{Relationship with the WL Distance and the OTC Distance}\label{sssection:interpolation}
Recall from Remark \ref{rmk:delta0} that $\dWL{k}{\mathcal{X},\mathcal{Y}} =  \dWL[0]{k}{\mathcal{X},\mathcal{Y}}$ for any finite $k$. In fact, we have the following stronger result, showing that $\dWL{k}{}$ is an appropriate limit of $\dWL[\delta]{k}{}$ : 
\begin{theorem}
For any Markov chains $\mathcal{X}$ and $\mathcal{Y}$, one has that 
\begin{equation}\label{eq:coro dwl}
\dWL{k}{\mathcal{X},\mathcal{Y}} = \lim_{\delta\rightarrow 0} \dWL[\delta]{k}{\mathcal{X},\mathcal{Y}} ~~\text{and hence} ~~ \dWL{\infty}{\mathcal{X},\mathcal{Y}} = \lim_{k\rightarrow\infty}\lim_{\delta\rightarrow 0} \dWL[\delta]{k}{\mathcal{X},\mathcal{Y}}.
\end{equation}
\end{theorem}

{Interestingly, it turns out that if we fix $k=\infty$, then $\dWL[\delta]{\infty}{}$ converges to $\dOTC{}$ as $\delta\rightarrow 0$.}  
This closes the loop for our previous claim about the OTM distances, 
showing that $\dOTC{}$ can also be expressed as a limit of OTM distances. 

\begin{theorem}\label{thm:regwlcv} 
For any \emph{stationary} Markov chains $\mathcal{X}$ and $\mathcal{Y}$, one has that
\begin{equation}\label{eq:coro dotc}
\dOTC{\mathcal{X},\mathcal{Y}} = \lim_{\delta\rightarrow 0} \dWL[\delta]{\infty}{\mathcal{X},\mathcal{Y}}~~\text{and hence} ~~ \dOTC{\mathcal{X},\mathcal{Y}} = \lim_{\delta\rightarrow 0}\lim_{k\rightarrow\infty} \dWL[\delta]{k}{\mathcal{X},\mathcal{Y}}.
\end{equation}
\end{theorem}

Besides this convergence result, we note that by Proposition \ref{thm:zerosets} the OTC distance and $\dWL[\delta]{\infty}{}$ have the same zero-sets for any $\delta>0$. This implies that although the OTC distance is an upper bound for $\dWL[\delta]{\infty}{}$, our new construction has the same discriminative power as the OTC distance.

From \cref{eq:coro dwl} and \cref{eq:coro dotc}, it is tempting to ask whether the order of the limits can be switched and whether $\dWL{\infty}{}=\dOTC{}$. Although we empirically observe that $\dWL{\infty}{}\neq\dOTC{}$ in general, due to the approximation nature of the algorithms implemented, we do not know for sure whether $\dWL{\infty}{}=\dOTC{}$ or not, and we leave this for future study.

\subsection{Algorithm and Convergence}\label{sec:algorithm and convergence}
As mentioned earlier, the discounted WL distance is a regularized version of the original WL distance, in this section we will elucidate on this claim and provide a recursive algorithm for computing the  (depth-$k$) discounted WL distance. 

In \cite{chen2022weisfeilerlehman}, a recursive algorithm was proposed to compute the depth-$k$ WL distance. We provide a $\delta$-regularized version of that algorithm in the proposition below. 
The $\delta$-regularization results in a unique fixed point solution to the recursive algorithm (as opposed to the original WL algorithm), which enables differentiation.

\begin{proposition}[Recursive computation]\label{prop:wlreg_recursive} 
Given any $k\in\N$, we recursively define matrices $C^{\delta, (l)}$ for $l=0,\ldots,k$ as follows:
\begin{equation}
C^{\delta, (0)}_{ij} =  C_{ij},\quad
C^{\delta, (l+1)}_{ij} = \delta C_{ij} + (1 - \delta)\, \dW<C^{\delta, (l)}>{m^{\setX}_i, m^{\setY}_j}. \label{eq:wlreg_costmatrix}
\end{equation}
Then, the depth-$k$ \gls{dWLdelta} can be computed as follows
\begin{equation}
\dWL<C>[\delta]{k}{\mX,\mY} =  \dW<C^{\delta, (k)}>{\nu^{\setX}, \nu^{\setY}}.
\label{eq:wlregk}\end{equation}

\end{proposition}

  \noindent The matrices $C^{\delta, (l)}$ depend on the Markov kernels $m^\setX_\bullet, m^\setY_\bullet$
  and on the cost matrix $C$. When making this dependency apparent is needed, we will use the notation
    $C^{\delta, (l)}(m^\setX_\bullet, m^\setY_\bullet, C)$. 

Using Proposition \ref{prop:wlreg_recursive}, one can devise an algorithm to compute $\dWL[\delta]{k}{\mX,\mY}$ for any finite $k$. A similar but more intricate recursive computation exists for general OTM distances. The simplicity of the discounted WL distance's recursive computation within the OTM family stems from the memoryless nature of the (truncated) geometric distribution.  See \Cref{sec:OTM_computation} for more details.

A natural question is whether $\dWL[\delta]{k}{\mX,\mY}$ can be a good approximation for $\dWL[\delta]{\infty}{\mX,\mY}$. We aim to answer this question below based on the observation that \cref{eq:wlreg_costmatrix} is a fixed point iteration, which enables us to use Banach fixed point theorem to prove convergence, and other properties. \looseness-1

\begin{proposition}[Convergence of \(C^{\delta,(k)}\)]\label{prop:wlregnotconstant} 
When $\delta>0$, \(C^{\delta,(k)}\) converges to the \emph{unique fixed point} $C^{\delta,(\infty)}$ of \cref{eq:wlreg_costmatrix} which is not a constant matrix (unless \(C\) is a constant matrix itself) such that
    \begin{equation}
\dWL[\delta]{\infty}{\mX,\mY} =  \dW<C^{\delta, (\infty)}>{\nu^{\setX}, \nu^{\setY}}.
\label{eq:wlregk-infty}\end{equation}
Moreover, $C^{\delta,(k)}$ converges at rate   $\norm{C^{\delta,(k)} - C^{\delta,(\infty)}}_\infty \leq \frac{2(1 - \delta)^k}{\delta} 
\norm{C}_\infty$. Consequentially, 
\[|\dWL[\delta]{k}{\mX,\mY} - \dWL[\delta]{\infty}{\mX,\mY}| \leq \frac{2(1 - \delta)^k}{\delta} 
\norm{C}_\infty.\]
When $\delta=0$ and $m^\setX_\bullet,m^\setY_\bullet$ are \emph{irreducible} and \emph{aperiodic}, then \(C^{\delta,(k)}\) converges to a constant matrix. 
\end{proposition}

As the \gls{dWL} corresponds to the case when $\delta=0$, this proposition actually implies that the WL distance $\dWL{\infty}{\mX, \mY} $ is independent of the initial distributions of $\mX$ and $\mY$ (see also Proposition~\ref{prop:cvwlinit} in \appendixref{sec: more on dWL}). 
When $\delta>0$, 
\(\dWL[\delta]{\infty}{\mX, \mY} \) behaves completely differently and it of course depends on the initial distributions of $\mX$ and $\mY$ since $\dW<C^{\delta, (\infty)}>{\nu^{\setX}, \nu^{\setY}}$ depends on $\nu^{\setX}$ and $\nu^{\setY}$ when $C^{\delta, (\infty)}$ is not constant.
Together with the fact that $\dOTC{}$ is only  defined for stationary Markov chains, we conclude that $\dWL[\delta]{\infty}{}$ distinguishes more Markov chains than both the WL and the OTC distances. \looseness-1

Finally, we remark that when $\delta>0$, the last step for the computation of $\dWL[\delta]{\infty}{}$ (\cref{eq:wlregk-infty}), involves solving for a meaningful (i.e., non-constant) optimal coupling between $\nu^{\setX}$ and $\nu^{\setY}$ that minimizes this cost. That coupling provides a matching between the state spaces of $\mX$ and $\mY$ which can be used for some applications. Note that, when $\delta=0$ (corresponding to the WL distance), as $C^{0,(\infty)}$ is a constant matrix, no meaningful coupling/matching can be obtained.

In \appendixref{sec:algorithm} we provide pseudo-codes for computing $\dWL[\delta]{k}{\mX,\mY}$ for both finite and infinite $k$ based on the two propositions above as well as a complexity analysis. We also provide certain acceleration techniques in \appendixref{sec:algorithm}, including a faster (in terms of complexity) algorithm in the case where both transition kernel matrices are sparse in \algorithmref{alg:computation_accel_k}, and techniques to empirically accelerate the computation.

\section{Differentiation of the Discounted WL Distance}\label{differentiation}
Recall from \cref{sec:algorithm and convergence} that \(d_{\mathrm{WL}}^{\delta, (k)}\) can be computed recursively for any finite $k$ and \(d_{\mathrm{WL}}^{\delta, (\infty)}\) can hence be approximated efficiently. 
However, in many applications such as graph learning, one requires that the distances involved can be differentiated.  
This motivates us to devise in this section an algorithm to differentiate
\(\dWL<C>[\delta]{\infty}{\mX, \mY}\) w.r.t. change in parameters $\mathcal{X},\mathcal{Y}$ or the cost $C$ when \(\delta > 0\). All missing proofs and details in this section are in \appendixref{sec:regularized discounted WL}. This section gives the key results needed to compute the gradient. The detailed steps of the computation of the backwards pass are laid out in \algorithmref{alg:backwardpass_simple} (with the other algorithms in \appendixref{sec:algorithm}).

\subsection{Sinkhorn Approximation}

To differentiate our distance, 
we want the steps to be differentiable. 
Optimal transport as defined in \cref{eq:wassersteindef} 
is not a differentiable problem.
In the literature, differentiability is achieved by replacing it with a smooth approximation, 
called the Sinkhorn distance, originally introduced in  \citet{cuturi2013sinkhorn}:
Using the same notations as in \cref{eq:wassersteindef}, given $\epsilon\geq 0$, the ($\epsilon$-)regularized OT problem is defined as:
\(
\dW^{\epsilon}<C>{\alpha,\beta}:=\min_{(X, Y)\in \mathcal{C}(\alpha, \beta)} \E\; C(X,Y) - \epsilon H(X, Y). \label{eq:sinkhorndef}
\)
Here $H$ denotes the entropy function, i.e., $H(X,Y) := -\sum_{i\in\setX,j\in\setY}P_{ij}\log(P_{ij})$, where $P_{ij}:=\Pr(X=i,Y=j)$.

We now define the entropy-regularized version of our discounted WL distance, denoted by $\dWL<C>[\delta][\epsilon]{k}{\mX,\mY}$, via formulas shown in \cref{eq:wlreg_costmatrix} and \cref{eq:wlregk} with the optimal transport distance $d_\mathrm{W}$ all replaced by the $\epsilon$ regularized optimal transport distance $d_\mathrm{W}^\epsilon$. We then set $\dWL<C>[\delta][\epsilon]{\infty}{\mX,\mY}:=\lim_{k\rightarrow\infty}\dWL<C>[\delta][\epsilon]{k}{\mX,\mY}$. See \appendixref{sec:regularized discounted WL} for the precise definition of $\dWL<C>[\delta][\epsilon]{k}{\mX,\mY}$ and the well-definedness of $\dWL<C>[\delta][\epsilon]{\infty}{\mX,\mY}$. It turns out that the entropy-regularized discounted WL distance is indeed an approximate of our original discounted WL distance.

\begin{theorem}[Convergence of the entropy-regularized distance]\label{thm:wlsinkhorncv}
For any Markov chains $\mX$, $\mY$ over a finite number of states, and cost matrix  $C$ between these two Markov chains and any $k\in\N\cup\{\infty\}$, one has that
$\lim_{\epsilon\rightarrow 0} \dWL<C>[\delta][\epsilon]{k}{\mX,\mY} = \dWL<C>[\delta]{k}{\mX,\mY}.$ Moreover, one has the following convergence rate:
\[|\dWL<C>[\delta][\epsilon]{k}{\mX,\mY}-\dWL<C>[\delta]{k}{\mX,\mY}|\leq \frac{\epsilon}{\delta}\log(|\setX||\setY|).\]
\end{theorem}

\subsection{Differentiation of $\dWL[\delta][\epsilon]{\infty}{\mX,\mY}$}
Fixing the underlying sets $\setX$ and $\setY$, the distance $\dWL<C>[\delta][\epsilon]{\infty}{\mX,\mY}$ can be written down explicitly as a function 
$\dWL<C>[\delta][\epsilon]{\infty}{m^{\setX}_{\bullet},m^{\setY}_{\bullet},\nu^{\setX},\nu^{\setY}}$
which depends on $m^{\setX}_{\bullet},m^{\setY}_{\bullet},\nu^{\setX},\nu^{\setY}$ and $C$.
Furthermore, to compute $\dWL[\delta][\epsilon]{\infty}{\mX,\mY}$, by Definition \ref{def:entropy dWL}, one first needs to compute the matrix $C^{\epsilon, \delta, (\infty)}$ which is a function
$C^{\epsilon, \delta, (\infty)}(m^{\setX}_{\bullet},m^{\setY}_{\bullet},C)$ depending on $m^{\setX}_{\bullet},m^{\setY}_{\bullet}$ and $C$. 
We now devise an algorithm to compute the gradient of $C^{\epsilon, \delta, (\infty)}$. 
Based on this, the gradient for $\dWL[\delta][\epsilon]{\infty}{\mX,\mY}$ can then be computed using the chain rule and the differentiation method for entropy-regularized OT
\citep[Proposition 4.6 and 9.2]{peyreComputationalOT2018}.

We use the following tensor notation to represent the target gradient of $C^{\epsilon, \delta, (\infty)}$:
\[\Delta := \left(\Delta_{ij}^{kl}\right){}_{1 \leq i \leq n, 1 \leq j \leq m }^{ 1 \leq k \leq n, 1 \leq l \leq m}, \,\Gamma := \left(\Gamma_{ij}^{kk'}\right){}_{1 \leq i \leq n, 1 \leq j \leq m }^{ 1 \leq k \leq n, 1 \leq k' \leq n}, \,\Theta := \left(\Theta_{ij}^{ll'}\right){}_{1 \leq i \leq n, 1 \leq j \leq m }^{ 1 \leq l \leq m, 1 \leq l' \leq m},\]
 where   $\Delta_{ij}^{kl} := \frac{\partial C^{\epsilon,\delta, (\infty)}_{ij}}{\partial C_{kl}}, \,\Gamma_{ij}^{kk'} := \frac{\partial C^{\epsilon,\delta, (\infty)}_{ij}}{\partial m^{\setX}_{kk'}}, \,\Theta_{ij}^{ll'} := \frac{\partial C^{\epsilon,\delta, (\infty)}_{ij}}{\partial m^{\setY}_{ll'}}$
and $m^{\setX}_{kk'}$ (resp. $m^{\setY}_{ll'}$) is the transition probability from state $k$ to state $k'$ (resp. from state $l$ to state $l'$).
For each $i,j$, given the matrix $C^{\epsilon,\delta, (\infty)}_{ij}$ (approximated by $C^{\epsilon,\delta, (k)}_{ij}$ in practice; see also Proposition \ref{prop:wlregnotconstant} for an analysis of convergence rate), we solve the regularized optimal transport problem
$\dW^\epsilon<C^{\epsilon,\delta, (\infty)}_{ij}>{m^{\setX}_i, m^{\setY}_j}\label{eq:OTij}$ to obtain the following data (defined in \definitionref{def:regOT} in Appendix):
\begin{itemize}\denselist
\item
  an optimal transport matrix (also called the primal solution) \(P_{ij}=\left(P_{ij}^{kl}\right){}^{ 1 \leq k \leq n, 1 \leq l \leq m}\);
\item
  and two dual solutions \(f_{ij}=\left(f_{ij}^k\right){}^{ 1 \leq k \leq n}\)
  and
  \(g_{ij}=\left(g_{ij}^l\right){}^{ 1 \leq l \leq m}\).
\end{itemize}

These give rise to the following tensors when considering all $i$ and $j$:
\begin{align*}
  P := \left(P_{ij}^{kl}\right){}_{1 \leq i \leq n, 1 \leq j \leq m }^{ 1 \leq k \leq n, 1 \leq l \leq m},\, F := \left(f_{ij}^{k'}\1_{i = k}\right){}_{1 \leq i \leq n, 1 \leq j \leq m }^{ 1 \leq k \leq n, 1 \leq k' \leq m},\, G := \left(g_{ij}^{l'}\1_{i = l}\right){}_{1 \leq i \leq n, 1 \leq j \leq m }^{ 1 \leq l \leq m, 1 \leq l' \leq m}.
\end{align*}

Now that we have computed $P$, $F$ and $G$, it turns out that we can use them to directly compute $\Delta,\Gamma$ and $\Theta$, {which contain all necessary gradients for $C^{\epsilon, \delta, (\infty)}$.}

\begin{theorem}[Explicit computation of the gradients]\label{thm:computation of gradients}
View the tensors defined above as  matrices by flattening their dimensions (and resp. codimensions) together — for example $P$ becomes an $nm \times nm$ square matrix. Let $I_{nm}$ denote the identity matrix of size $nm\times nm$. Then, one has 
$\Delta = \delta\left(I_{nm} - (1 - \delta)P \right)^{-1}$, 
  $\Gamma = (1 - \delta)\left(I_{nm} - (1 - \delta)P \right)^{-1} F$ and 
  $\Theta = (1 - \delta)\left(I_{nm} - (1 - \delta)P \right)^{-1} G.$
\end{theorem}

\begin{figure}
  \centering
      \raisebox{-\height}{\begin{subfigure}[t]{.23\textwidth}
          \centering
          \includegraphics[width=\textwidth]{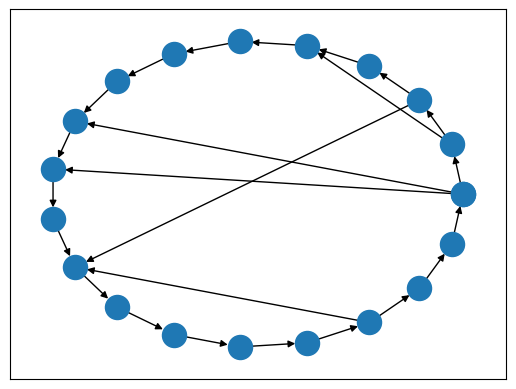}
          \caption{A noisy circle graph}
      \end{subfigure}}
      \raisebox{-\height}{\begin{subfigure}[t]{.23\textwidth}
          \centering
          \includegraphics[width=\textwidth]{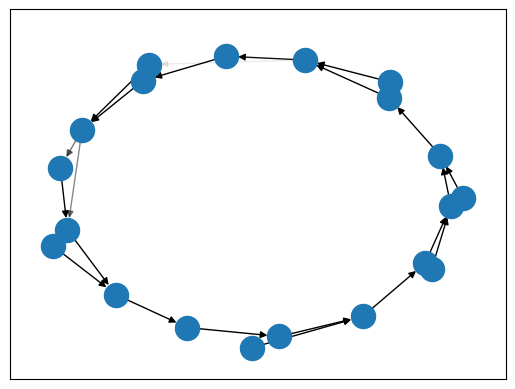}
          \caption{$\dWL[\epsilon][\delta]{\infty}{}$ barycenter}
      \end{subfigure}}
      \raisebox{-\height}{\begin{subfigure}[t]{.23\textwidth}
          \centering
          \includegraphics[width=\textwidth]{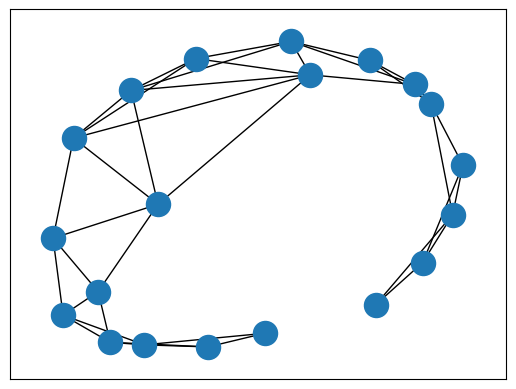}
          \caption{FGW barycenter}
      \end{subfigure}}
  
  \caption{Barycenter computation of 30 noisy circle graphs}
  \label{fig:bary}
  \end{figure}

Please refer to \appendixref{sec:algorithm} for the pseudo-code that implements gradient computation based on the above theorem as well as analysis on computational complexity.

\section{Experiments}\label{sec:experiment}
In this section, we employ the discounted WL distance for graph classification tasks and the computation of graph barycenters. It is important to highlight that, for computing graph barycenters, we deploy the gradient descent method to minimize the Fréchet functional. This approach necessitates the differentiability of our distance. We also demonstrate how our distance can be used to compute graph coarsening via gradient descent in \appendixref{par:coarsening}.
\paragraph{Graph Classification.}\label{par:classification}

We compared our distance to the fused GW distance (FGW) \cite{titouan2019optimal} on classification benchmarks on real world datasets from the \href{http://www.graphlearning.io/}{TUDataset repository}~\citep{tudataset} (see \tableref{tab:classification}).
The SVM for both FGW and OTM distances are learnt using rbf kernels (with cross-valitated regularization parameter). FGW is run with $\alpha=0.5$.
When attributes are discrete, the cost (distance) used is the Dirac cost ($0$ if the attributes are the same, $1$ otherwise).

1-NN results suggest FGW is superior or similar when labels carry low information, and our distance performs better when the label carries more information. SVM results are mitigated, and suggest similar results as other methods.

\begin{table}[htb]
    \resizebox{\textwidth}{!}{%
\begin{tabular}{llllll}
dataset                                           & PROTEINS                    & PTC\_MR                     & PROTEINS\_full              & ENZYMES                 \\
\hline
classes                                           & 2                           & 2                           & 2                           & 6                       \\
attributes                                        & discrete label              & discrete label              & 29                          & 18                      \\
\hline
FGW 1-NN                                          & $\mathbf{65.1\% \pm 4.6\%}$ & $57.6\% \pm 5.0\%$          & $69.5\% \pm 4.0\%$          & $66.3\% \pm 6.4\%$\\
$d_{WL, \delta}^{(\infty)}$ ($\delta = 0.2$) 1-NN & $61.4\% \pm 4.0\%$          & $\mathbf{61.3\% \pm 7.6\%}$ & $\mathbf{70.0\% \pm 4.5\%}$ & $\mathbf{74.7\% \pm 6.2\%}$\\
\hline
FGW SVM                                           & $70.5\% \pm 2.9\%$          & $57.6\% \pm 4.6\%$          & $\mathbf{75.0\% \pm 3.8\%}$ & $42.7\% \pm 13.5\%$\\
$d_{WL, \delta}^{(\infty)}$ ($\delta = 0.2$) SVM  & $\mathbf{76.4\% \pm 5.3\%}$ & $\mathbf{61.3\% \pm 5.9\%}$ & $73.5\% \pm 3.1\%$          & $\mathbf{68.3\% \pm 4.1\%}$\\
\end{tabular}%
}%
\caption{Results of the classification experiment}\label{tab:classification}
\end{table}

\paragraph{Barycenter Computation.} 
In order to show the effectiveness of our proposed OTM distances as optimization targets for learning tasks on directed graphs, 
we compute a simple graph barycenter. 
With random noisy cycle graphs $G_1, \ldots , G_n$ as input, we compute the barycenter graph $G_{\text{bar}}$ by minimizing the following objective function: $\sum_i \dWL<C>[\delta]{\infty}{G_{\text{bar}},G_i}$ where the cost $C$ is based on the euclidean distance over the labels.
The detailed results (including comparison with the barycenter computed by using fused-GW of \citet{titouan2019optimal}) and setup can be found in \appendixref{app:barycenter_computation}, and one example is shown in \figureref{fig:bary} where the discounted WL distance achieves a better barycenter than the fused GW distance.

\section{Concluding Remarks}\label{sec:conclusion}
Our paper provides a novel framework of OTM distances comparing Markov chains and hence directed graphs. 
As our discount-WL distance can be differentiated, 
the natural next step is to apply our distances to various learning problems, 
such as to provide effective statistical analysis in the space of graphs (equipped with this metric), 
or to provide loss for learning models (e.g. graph generative models) over complex networks. 
In order to make this endeavour easier, we provide the code to compute it\footnote{\url{https://github.com/YusuLab/ot_markov_distances}}, 
in the form of a packaged python library\footnote{\url{https://pypi.org/project/ot-markov-distances/}}.
We are particularly interested in exploring the use of the discounted WL distance (or variants) to study directed networks, 
where our current available tool-box has been more limited.

\paragraph{Limitations.} On a practical front, the computation of our new distance can be slow on large graphs, although technical optimizations presented in \appendixref{sec:algorithm} mitigate that to some extent. The hyperparameters (e.g, $\epsilon$, $\delta$) also require careful handling. 
Though our distance calculation is empirically much slower than the approximate fused-GW distance (See \appendixref{app:classification} for comparisons), 
it is polynomial-time computable, unlike the NP-hard exact FGW. 
This difference allows for acceleration techniques, potentially enhancing efficiency. 
We posit that methods like neural OT~\citep{makkuva20optimal,korotin2022neural,chen2023neural} could be integrated into our framework for further gains.

\section*{Acknowledgments}
This work is partially supported by NSF under grants CCF-2112665, CCF-2217058, and  CCF-2310411.

\bibliography{resources/dWL}

\begin{thebibliography}{35}
\providecommand{\natexlab}[1]{#1}
\providecommand{\url}[1]{\texttt{#1}}
\expandafter\ifx\csname urlstyle\endcsname\relax
  \providecommand{\doi}[1]{doi: #1}\else
  \providecommand{\doi}{doi: \begingroup \urlstyle{rm}\Url}\fi

\bibitem[Anderson et~al.(1999)Anderson, Bai, Bischof, Blackford, Demmel,
  Dongarra, Du~Croz, Greenbaum, Hammarling, McKenney, and Sorensen]{lapack}
E.~Anderson, Z.~Bai, C.~Bischof, S.~Blackford, J.~Demmel, J.~Dongarra,
  J.~Du~Croz, A.~Greenbaum, S.~Hammarling, A.~McKenney, and D.~Sorensen.
\newblock \emph{{LAPACK} Users' Guide}.
\newblock Society for Industrial and Applied Mathematics, Philadelphia, PA,
  third edition, 1999.
\newblock ISBN 0-89871-447-8 (paperback).

\bibitem[Azizian and Lelarge(2021)]{azizian2020expressive}
W.~Azizian and M.~Lelarge.
\newblock Expressive power of invariant and equivariant graph neural networks.
\newblock In \emph{International Conference on Learning Representations}, 2021.

\bibitem[Babai and Kucera(1979)]{babai1979canonical}
L.~Babai and L.~Kucera.
\newblock Canonical labelling of graphs in linear average time.
\newblock In \emph{20th Annual Symposium on Foundations of Computer Science
  (sfcs 1979)}, pages 39--46. IEEE, 1979.

\bibitem[Babai and Luks(1983)]{babai1983canonical}
L.~Babai and E.~M. Luks.
\newblock Canonical labeling of graphs.
\newblock In \emph{Proceedings of the fifteenth annual ACM symposium on Theory
  of computing}, pages 171--183, 1983.

\bibitem[Brogat-Motte et~al.(2022)Brogat-Motte, Flamary, Brouard, Rousu, and
  d’Alch{\'e} Buc]{brogat2022learning}
L.~Brogat-Motte, R.~Flamary, C.~Brouard, J.~Rousu, and F.~d’Alch{\'e} Buc.
\newblock Learning to predict graphs with fused {G}romov-{W}asserstein
  barycenters.
\newblock In \emph{International Conference on Machine Learning}, pages
  2321--2335. PMLR, 2022.

\bibitem[Bunch and Hopcroft(1974)]{bunch1974}
J.~R. Bunch and J.~E. Hopcroft.
\newblock Triangular factorization and inversion by fast matrix multiplication.
\newblock \emph{Mathematics of Computation}, 28\penalty0 (125):\penalty0
  231--236, 1974.
\newblock ISSN 00255718, 10886842.

\bibitem[Chen and Wang(2023)]{chen2023neural}
S.~Chen and Y.~Wang.
\newblock Neural approximation of wasserstein distance via a universal
  architecture for symmetric and factorwise group invariant functions.
\newblock In \emph{Thirty-seventh Conference on Neural Information Processing
  Systems}, 2023.

\bibitem[Chen et~al.(2022)Chen, Lim, Mémoli, Wan, and
  Wang]{chen2022weisfeilerlehman}
S.~Chen, S.~Lim, F.~Mémoli, Z.~Wan, and Y.~Wang.
\newblock Weisfeiler-{L}ehman meets {G}romov-{W}asserstein.
\newblock In \emph{International Conference on Machine Learning (ICML)}, pages
  3371--3416. PMLR, 2022.

\bibitem[Chen et~al.(2023)Chen, Lim, Mémoli, Wan, and Wang]{chen2023wl}
S.~Chen, S.~Lim, F.~Mémoli, Z.~Wan, and Y.~Wang.
\newblock The {W}eisfeiler-{L}ehman distance: reinterpretation and connection
  with {GNN}s.
\newblock \emph{ICML workshop: Topology, Algebra, and Geometry in Machine
  Learning (2023)}, 2023.

\bibitem[Chuang and Jegelka(2022)]{chuangtree}
C.-Y. Chuang and S.~Jegelka.
\newblock Tree mover's distance: Bridging graph metrics and stability of graph
  neural networks.
\newblock In \emph{Advances in Neural Information Processing Systems}, 2022.

\bibitem[Cuturi(2013)]{cuturi2013sinkhorn}
M.~Cuturi.
\newblock Sinkhorn distances: Lightspeed computation of optimal transport.
\newblock \emph{Advances in neural information processing systems}, 26, 2013.

\bibitem[Duan et~al.(2022)Duan, Wu, and Zhou]{duan2022faster}
R.~Duan, H.~Wu, and R.~Zhou.
\newblock Faster matrix multiplication via asymmetric hashing, 2022.

\bibitem[Feydy et~al.(2019)Feydy, S{\'e}journ{\'e}, Vialard, Amari, Trouv{\'e},
  and Peyr{\'e}]{feydyInterpolatingOptimalTransport2019}
J.~Feydy, T.~S{\'e}journ{\'e}, F.-X. Vialard, S.-i. Amari, A.~Trouv{\'e}, and
  G.~Peyr{\'e}.
\newblock Interpolating between optimal transport and mmd using sinkhorn
  divergences.
\newblock In \emph{The 22nd International Conference on Artificial Intelligence
  and Statistics}, pages 2681--2690. PMLR, 2019.

\bibitem[Griffeath(1975)]{Griffeath1975}
D.~Griffeath.
\newblock A maximal coupling for {M}arkov chains.
\newblock \emph{Zeitschrift f{\"u}r Wahrscheinlichkeitstheorie und verwandte
  Gebiete}, 31\penalty0 (2):\penalty0 95--106, 1975.

\bibitem[Kolmogorov and Bharucha-Reid(2018)]{kolmogorov2018foundations}
A.~N. Kolmogorov and A.~T. Bharucha-Reid.
\newblock \emph{Foundations of the theory of probability: Second English
  Edition}.
\newblock Courier Dover Publications, 2018.

\bibitem[Korotin et~al.(2022)Korotin, Selikhanovych, and
  Burnaev]{korotin2022neural}
A.~Korotin, D.~Selikhanovych, and E.~Burnaev.
\newblock Neural optimal transport.
\newblock In \emph{The Eleventh International Conference on Learning
  Representations}, 2022.

\bibitem[Lehman and Weisfeiler(1968)]{leman1968reduction}
A.~Lehman and B.~Weisfeiler.
\newblock A reduction of a graph to a canonical form and an algebra arising
  during this reduction.
\newblock \emph{Nauchno-Technicheskaya Informatsiya}, 2\penalty0 (9):\penalty0
  12--16, 1968.

\bibitem[Levin and Peres(2017)]{levin2017markov}
D.~A. Levin and Y.~Peres.
\newblock \emph{Markov chains and mixing times}, volume 107.
\newblock American Mathematical Soc., 2017.

\bibitem[Makkuva et~al.(2020)Makkuva, Taghvaei, Oh, and Lee]{makkuva20optimal}
A.~Makkuva, A.~Taghvaei, S.~Oh, and J.~Lee.
\newblock Optimal transport mapping via input convex neural networks.
\newblock In H.~D. III and A.~Singh, editors, \emph{Proceedings of the 37th
  International Conference on Machine Learning}, volume 119 of
  \emph{Proceedings of Machine Learning Research}, pages 6672--6681. PMLR,
  2020.

\bibitem[Morris et~al.(2020)Morris, Kriege, Bause, Kersting, Mutzel, and
  Neumann]{tudataset}
C.~Morris, N.~M. Kriege, F.~Bause, K.~Kersting, P.~Mutzel, and M.~Neumann.
\newblock Tudataset: A collection of benchmark datasets for learning with
  graphs.
\newblock In \emph{ICML 2020 Workshop on Graph Representation Learning and
  Beyond (GRL+ 2020)}, 2020.
\newblock URL \url{www.graphlearning.io}.

\bibitem[Moulos(2021)]{moulos2021bicausal}
V.~Moulos.
\newblock Bicausal optimal transport for {M}arkov chains via dynamic
  programming.
\newblock In \emph{2021 IEEE International Symposium on Information Theory
  (ISIT)}, pages 1688--1693. IEEE, 2021.

\bibitem[O'Connor et~al.(2022)O'Connor, McGoff, and Nobel]{o2022optimal}
K.~O'Connor, K.~McGoff, and A.~B. Nobel.
\newblock Optimal transport for stationary {M}arkov chains via policy
  iteration.
\newblock \emph{Journal of Machine Learning Research}, 23:\penalty0 45--1,
  2022.

\bibitem[Paszke et~al.(2019)Paszke, Gross, Massa, Lerer, Bradbury, Chanan,
  Killeen, Lin, Gimelshein, Antiga, et~al.]{pytorch}
A.~Paszke, S.~Gross, F.~Massa, A.~Lerer, J.~Bradbury, G.~Chanan, T.~Killeen,
  Z.~Lin, N.~Gimelshein, L.~Antiga, et~al.
\newblock Pytorch: An imperative style, high-performance deep learning library.
\newblock \emph{Advances in neural information processing systems}, 32, 2019.

\bibitem[Pata et~al.(2019)]{pata2019fixed}
V.~Pata et~al.
\newblock \emph{Fixed point theorems and applications}, volume 116.
\newblock Springer, 2019.

\bibitem[Peyr{\'e} et~al.(2019)Peyr{\'e}, Cuturi,
  et~al.]{peyreComputationalOT2018}
G.~Peyr{\'e}, M.~Cuturi, et~al.
\newblock Computational optimal transport: With applications to data science.
\newblock \emph{Foundations and Trends{\textregistered} in Machine Learning},
  11\penalty0 (5-6):\penalty0 355--607, 2019.

\bibitem[Puterman(2014)]{puterman2014markov}
M.~L. Puterman.
\newblock \emph{Markov decision processes: discrete stochastic dynamic
  programming}.
\newblock John Wiley \& Sons, 2014.

\bibitem[Shervashidze et~al.(2011)Shervashidze, Schweitzer, Van~Leeuwen,
  Mehlhorn, and Borgwardt]{shervashidze2011weisfeiler}
N.~Shervashidze, P.~Schweitzer, E.~J. Van~Leeuwen, K.~Mehlhorn, and K.~M.
  Borgwardt.
\newblock Weisfeiler-{L}ehman graph kernels.
\newblock \emph{Journal of Machine Learning Research}, 12\penalty0 (9), 2011.

\bibitem[Sutton and Barto(2018)]{sutton2018reinforcement}
R.~S. Sutton and A.~G. Barto.
\newblock \emph{Reinforcement learning: An introduction}.
\newblock MIT press, 2018.

\bibitem[Togninalli et~al.(2019)Togninalli, Ghisu, Llinares-L{\'o}pez, Rieck,
  and Borgwardt]{togninalli2019wasserstein}
M.~Togninalli, E.~Ghisu, F.~Llinares-L{\'o}pez, B.~Rieck, and K.~Borgwardt.
\newblock Wasserstein {W}eisfeiler-{L}ehman graph kernels.
\newblock \emph{Advances in Neural Information Processing Systems},
  32:\penalty0 6439--6449, 2019.

\bibitem[Vayer et~al.(2019)Vayer, Courty, Tavenard, and
  Flamary]{titouan2019optimal}
T.~Vayer, N.~Courty, R.~Tavenard, and R.~Flamary.
\newblock Optimal transport for structured data with application on graphs.
\newblock In \emph{International Conference on Machine Learning}, pages
  6275--6284. PMLR, 2019.

\bibitem[Villani et~al.(2009)]{villani2009optimal}
C.~Villani et~al.
\newblock \emph{Optimal transport: old and new}, volume 338.
\newblock Springer, 2009.

\bibitem[Vincent-Cuaz et~al.(2021)Vincent-Cuaz, Vayer, Flamary, Corneli, and
  Courty]{vincent2021online}
C.~Vincent-Cuaz, T.~Vayer, R.~Flamary, M.~Corneli, and N.~Courty.
\newblock Online graph dictionary learning.
\newblock In \emph{International Conference on Machine Learning}, pages
  10564--10574. PMLR, 2021.

\bibitem[Xia et~al.(2023)Xia, Mishne, and Wang]{xia2023implicit}
X.~Xia, G.~Mishne, and Y.~Wang.
\newblock Implicit graphon neural representation.
\newblock In F.~Ruiz, J.~Dy, and J.-W. van~de Meent, editors, \emph{Proceedings
  of The 26th International Conference on Artificial Intelligence and
  Statistics}, volume 206 of \emph{Proceedings of Machine Learning Research},
  pages 10619--10634. PMLR, 25--27 Apr 2023.

\bibitem[Xu et~al.(2018)Xu, Hu, Leskovec, and Jegelka]{xu2018powerful}
K.~Xu, W.~Hu, J.~Leskovec, and S.~Jegelka.
\newblock How powerful are graph neural networks?
\newblock In \emph{International Conference on Learning Representations}, 2018.

\bibitem[Yi et~al.(2021)Yi, O'Connor, McGoff, and Nobel]{yi2021alignment}
B.~Yi, K.~O'Connor, K.~McGoff, and A.~B. Nobel.
\newblock Alignment and comparison of directed networks via transition
  couplings of random walks.
\newblock \emph{arXiv preprint arXiv:2106.07106}, 2021.

\end{thebibliography}

\newpage 

\appendix

\section{Experiment details}\label{sec:experiment_details}

The code to run experiments is available on GitHub\footnote{\url{https://github.com/YusuLab/ot_markov_distances}}.

\subsection{Barycenter Computation}\label{app:barycenter_computation}

\paragraph{Target graphs} The goal of this experiment is to show that $\dWL[\delta][\epsilon]{\infty}{}$ produces barycenters that are meaningful with regard to the structure of the input graphs. Here we use simple data: oriented circles with 20 nodes, which we perturb through Erdős-Rényi noise of equal edge addition and deletion probability $p$. Examples of such data are shown in \cref{subfig:bary:data}. The attributes of the nodes of the circles are their $(x, y)$ positions.
Our goal is to check that the barycenter approximately recovers the original circle.

\paragraph{Parametric Markov kernels}\label{par:parametric_markov_kernels}
In our experiments, when learning a Markov kernel, it is crucial to ensure that all transition probabilities retain their properties as probability distributions throughout the training process, meaning they remain non-negative and continue to sum to one.
We could have used projected gradient descent, but due to better empirical results, we decided to use a \emph{parametric Markov kernel}.
An $n\times n$ Markov kernel $M$ is parameterized by an $n\times n$ matrix $\Theta\in\R_+^{n\times n}$
using the parameterization

\begin{equation}
  M_i = \text{Softmax}\left(\frac{\Theta_i}{\text{heat}}\right)
\end{equation}
where $\text{heat}$ is a positive floating point parameter, and $M_i$ (resp $\Theta_i$) is the $i$-th row of $M$ (resp $\Theta$).

This choice of parameterization is both theoretically grounded and practically motivated:
\begin{enumerate}
    \item \textbf{Universality}: This parameterization reaches all dense (without 0 entry) transition matrices and approximates all others.
    \item \textbf{Standardization}: This aligns with common machine learning practices, where softmax is used to output probability distributions. As a Markov kernel is a collection of probability distributions, this approach is logical. It thus illustrates how our distance could interface with outputs of neural networks.
    \item \textbf{Convenience}:
    This method avoids issues like projected gradient descent and degenerated gradients, and is compatible with frameworks like PyTorch~\citep{pytorch}.
    \item \textbf{Sparsity Encouragement}: Paradoxically, this parameterization encourages relatively sparse transition matrices via exponentials and certain choice of threshold.
\end{enumerate}

\paragraph{Setup}
In this experiment, all initial distributions are taken to be uniform.

Let $G^1, \ldots , G^n$ denote the graphs whose barycenter we want to compute.
Let $M^1, \ldots , M^n$ denote the transition matrices of the random walks on those graphs, respectively, defined as follows:

\begin{equation}
  M^i = (D^i)^{-1} A^i
\end{equation}

where $A^i$ is the adjacency matrix of $G^i$ and $D^i$ is the diagonal matrix of degrees of $G^i$.
Finally, we let $l^{1} \in \R^{s_1 \times d}, \ldots , l^{n} \in \R^{s_n \times d}$ denote
the labels of the graphs, 
where $s_i$ is the number of nodes of $G^i$, and $d$ is the label size.

Since the size (number of vertices) of the input graphs is not necessarily the same, 
we need to define the size of the barycenter graph.
We leave it as a hyperparameter, and denote it by $s$.
\begin{itemize}
  \item the barycenter Markov kernel $M^{\text{bar}} \in \R_{+}^{s\times s}$ 
  \item the labels $l^\text{bar} \in \R^{s \times d}$ of the barycenter graph
\end{itemize}

We encode $M^{\text{bar}}$ as a parametric Markov matrix as described in the previous paragraph, 
and $l^\text{bar}$ directly as a matrix of parameters.

We then minimize the following objective function: 
\begin{equation}
f(M, l) = \sum_i \dWL<C^{i}(l)>[\delta][\epsilon]{\infty}{M, M^i}
\end{equation}

where $C^{i}(l)$ is the cost matrix defined as 

\begin{equation}
  C^{i}(l)_{u,v} = \norm{l_u - l^i_v}_2^2
\end{equation}

This objective may appear unconventional; however, it is equal to the following:
\begin{equation}
f(M, l) = \sum_i h(M, M^i)^2
\end{equation}
where $h$ is the pseudometric defined as in \propositionref{prop:alpha_OTM_is_distance}, 
with $\alpha = 2$ and where the pseudodistance 
is the $L_2$ distance between labels. 
This is the so-called Fréchet variance for the space with pseudometric $h$, 
and a minimizer of it as called a Fréchet mean.

And we use the ``Adam'' optimizer (with the implementation from Pytorch~\cite{pytorch}) 
to minimize the objective function.

The parameters for the $\dWL[\delta][\epsilon]{\infty}{}$ distance we chose are $\delta = 0.5$ and $\epsilon = 0.05$. 

The Fused Gromov-Wasserstein barycenters are computed using the official implementation from \citet{titouan2019optimal}.
The method is the one described in [Vayer et al, 2019], ie Block Coordinate Descent (BCD). The parameters used are the following
\begin{itemize}
    \item The tradeoff parameter is $\alpha = 0.95$ (heavily skewed towards the structural loss rather than the attribute loss)
    \item The weights are not learnt, but fixed to the uniform distribution. This is the same setting as for the delta-discounted WL distance barycenter.
\end{itemize}

\paragraph{Computational power}
Each barycenter computation takes about 2.5 to 11 minutes on an Nvidia RTX A6000 GPU depending on the number of target graphs(ranging from 1 to 50). This computation involves 1000 steps with a learning rate of $10^{-2}$. Although the time can be reduced by decreasing the number of steps, increasing the learning rate, or increasing either $\delta$ or $\epsilon$, these adjustments might degrade the quality of the results.

For a theoretical analysis of the complexity, please refer to \cref{sec:algorithm}. Comprehensive performance benchmarking can be found in the \lstinline{Performance.ipynb} notebook included in the appended code. The results of this benchmarking are presented in \cref{fig:performance}.

As a comparison computing one FGW barycenter for this experiment takes between $0.005$s and $4.15s$ with an average of $0.67$s on CPU (using the code provided by \citet{titouan2019optimal}). We acknowledge the lack of competitiveness of our method in terms of time complexity, as mentioned in \sectionref{sec:conclusion}. We hope that advantages of our distance outweight this problem, and that subsequent work will allow for more compute-efficient approximations.

\paragraph{Results} We compare the produced barycenters (in \cref{subfig:bary:results}) with the ones produced by the state-of-the-art graph distance, Fused Gromov-Wasserstein distance~\cite{titouan2019optimal} (in \cref{subfig:bary:fgw}). We observe that for higher noise values ($p=0.01$), our distance recovers the structure significantly better.
It is interesting to see that for very high noise ($p=0.1$), our distance and FGW fail in very similar way: they create one or several "accumulation nodes" that are in the middle (matched with several original nodes) and linked to and from a lot of nodes.

\begin{figure}[htbp]
\begin{center}
  \begin{subfigure}{0.5\textwidth}
      \includegraphics[width=\textwidth]{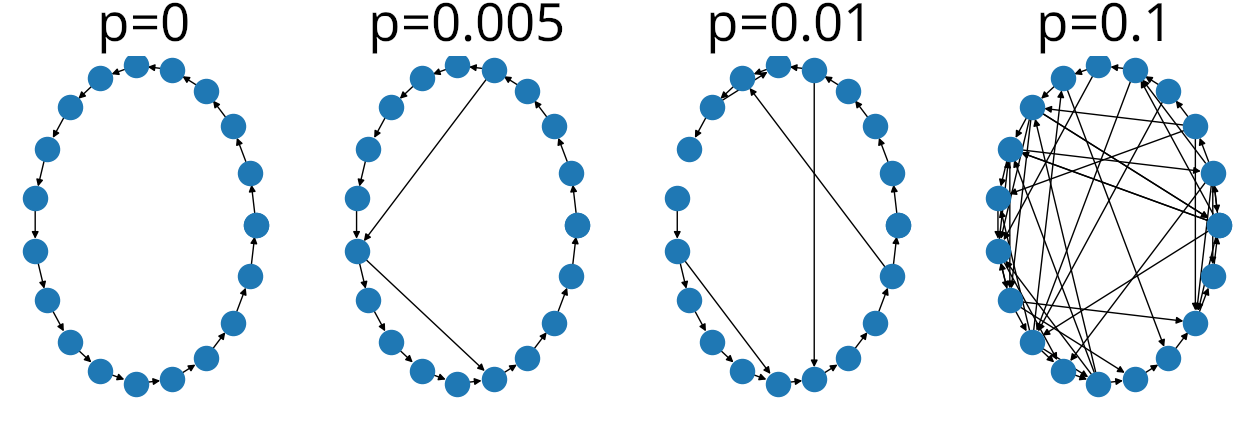}
    \caption{Example of perturbed circles depending on the noise level $p$}%
    \label{subfig:bary:data}
  \end{subfigure}\\
  \begin{subfigure}{0.9\textwidth}
  \includegraphics[width=\textwidth]{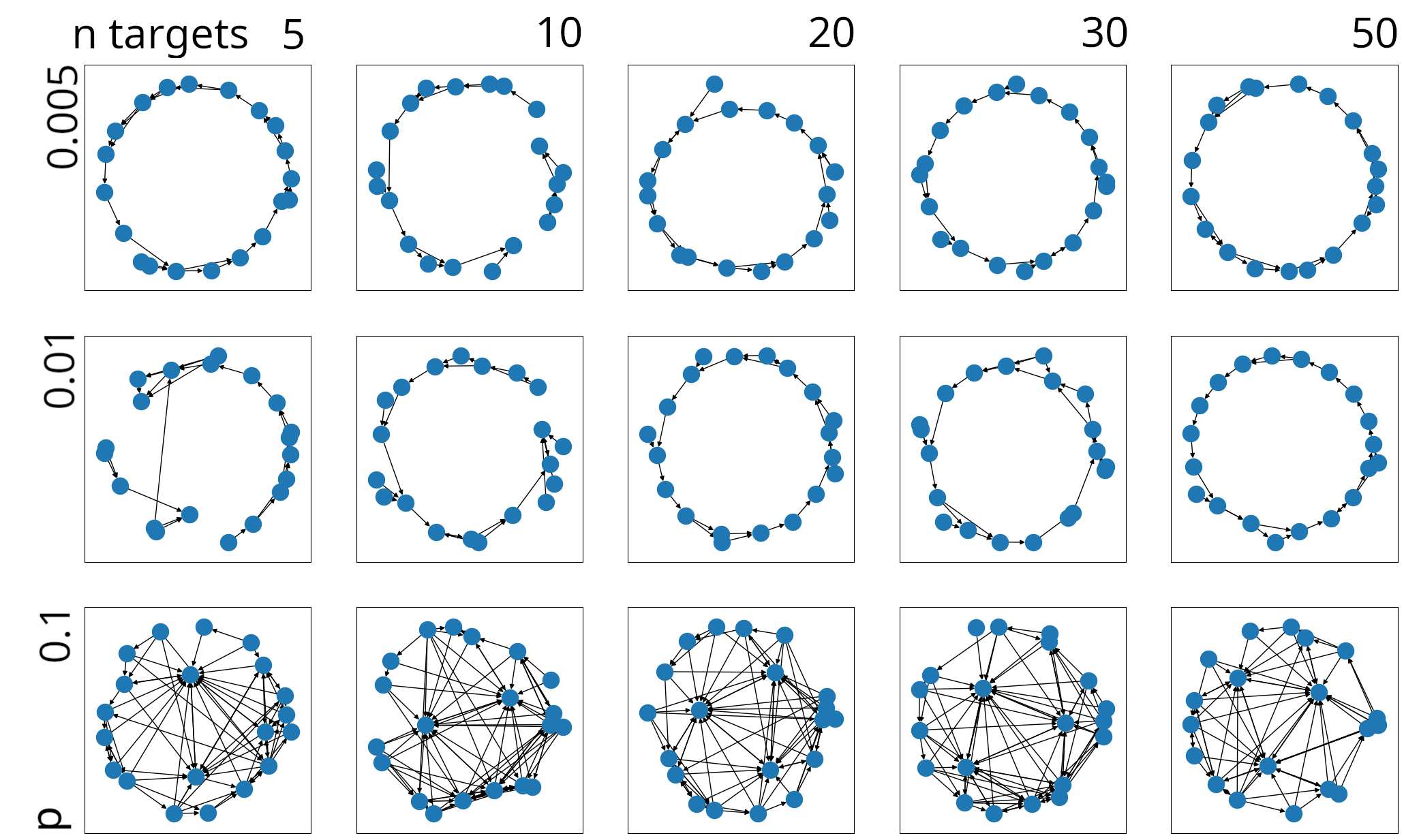}
  \caption{Barycenter computed with different values of $p$ (in ordinate) and for different number of graphs (in abscissa), using our distance}%
    \label{subfig:bary:results}
\end{subfigure}%
\hfill
\begin{subfigure}{0.9\textwidth}
  \includegraphics[width=\textwidth]{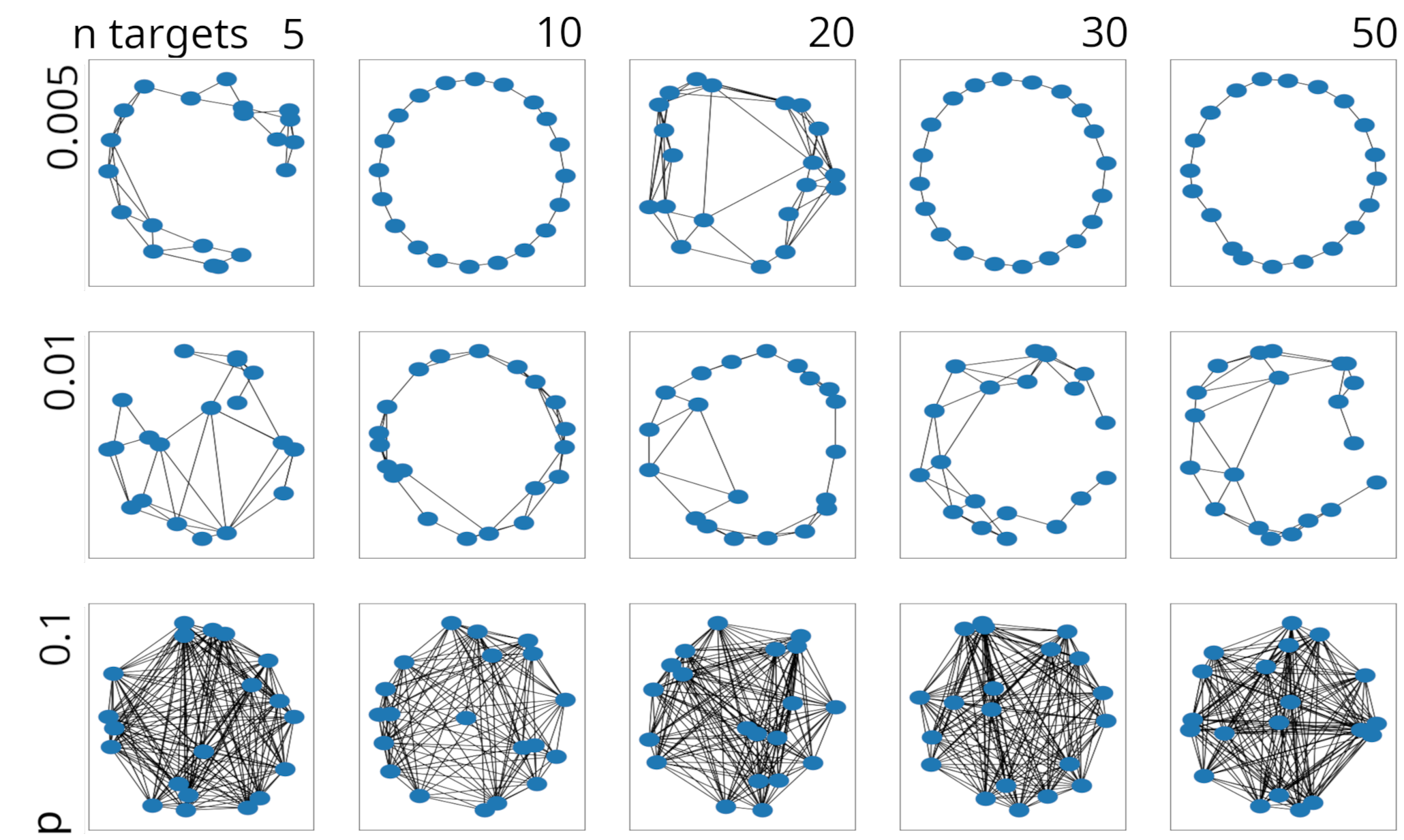}
    \caption{Barycenter computed on the same graphs without orientation, using the FGW barycenters from \citet{titouan2019optimal}}%
    \label{subfig:bary:fgw}
\end{subfigure}
\end{center}
\caption{Barycenter experiment}\label{fig:barycenter}
\end{figure}

\begin{figure}[htbp]
\begin{center}
  \begin{subfigure}{0.5\textwidth}
      \includegraphics[width=\textwidth]{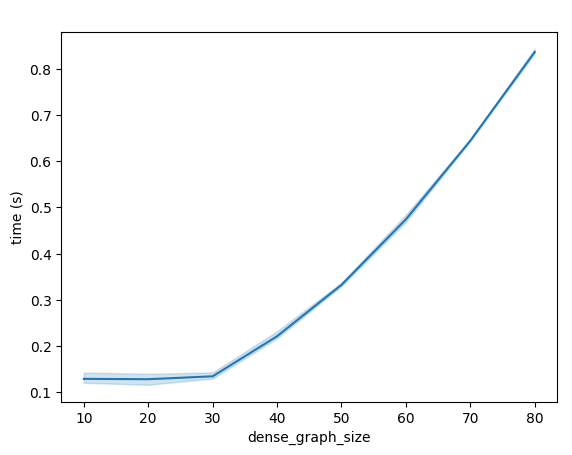}
      \caption{Average time for computing 
        $\dWL[\delta][\epsilon]{\infty}{G_\text{dense}, G_\text{sparse}}$, 
      where $G_\text{sparse}$ is of size 64 and degree 5 and the size of 
      $G_\text{dense}$ is on the abscissa.
    }%
    \label{subfig:performance:dense_sparse}
  \end{subfigure}\\
  \begin{subfigure}{0.49\textwidth}
  \includegraphics[width=\textwidth]{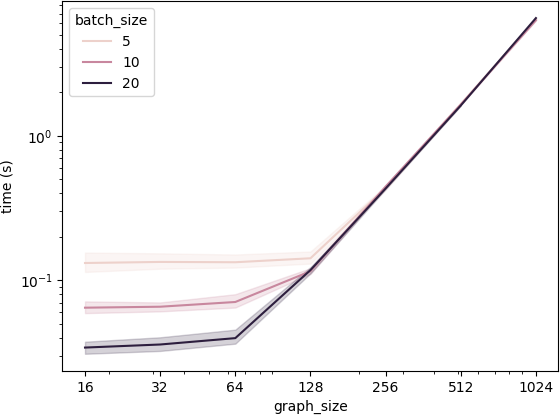}
  \caption{Average computing time of $\dWL[\delta][\epsilon]{\infty}{}$ with two sparse graphs of varying size with different batch sizes}%
    \label{subfig:performance:varying_batch}
\end{subfigure}%
\hfill
\begin{subfigure}{0.49\textwidth}
  \includegraphics[width=\textwidth]{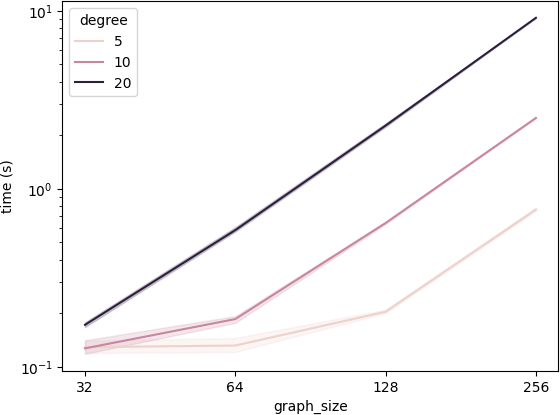}
  \caption{Average computing time of $\dWL[\delta][\epsilon]{\infty}{}$ with two sparse graphs of varying degree and size}%
    \label{subfig:performance:varying_degree}
\end{subfigure}
\end{center}
\caption{Performance analysis results on an Nvidia RTX A6000 GPU}\label{fig:performance}
\end{figure}

\paragraph{Influence of parameters} In this paragraph, we study the influence of the $\delta$ and $\epsilon$ parameters on the result of this experiment. We run this barycenter expeeriment while varying the values of $\epsilon$ and $\delta$, the results are shown in \figureref{fig:epsilon-delta}
This shows degenerated cases :
\begin{itemize}
\item $\delta = 1$ Our distance degenerates to the (regularized) Wasserstein distance between node label sets. Positions are learned through a Wasserstein barycenter problem, but the Markov kernels remain unlearned, with the resulting graph reflecting only random initialization.
\item high $\epsilon$ Overregularization occurs with high values of , hindering the learning process. The most extreme manifestation of this can be observed in the lower-right part of the grid, where all points are matched equally to each other, resulting in a graph with all nodes clustered at the center.
\item low $\delta$ A low value of $\delta$ introduces instability in learning. This is evident in the upper-left corner of the grid, where the learning process appears erratic.
\end{itemize}

\begin{figure}[htb]
    \centering
    \includegraphics[width=0.6\textwidth]{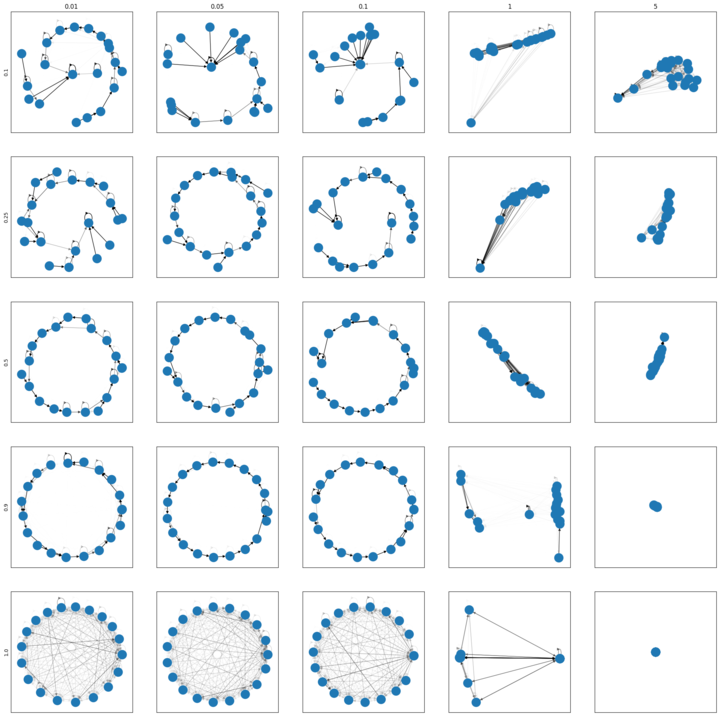}
    \caption{Same barycenter experiment ($n\_targets=20$, $p=0.01$), run with different values of $\epsilon$ (in abscissa) and $\delta$ in ordinate}
    \label{fig:epsilon-delta}
\end{figure}

\subsection{Graph Classification}\label{app:classification} 
The reported accuracy and error margins are the average and standard deviation from a stratified k-fold with 5 splits.
In terms of runtime, we measured an average of $0.18s$ to compute one discounted WL distance with parameters $\delta=0.5$, $\epsilon=0.1$ (using an Nvidia RTX A6000 graphics card, and an AMD EPYC 7452 32-Core Processor).
As a comparison, on the same dataset and hardware, FGW takes an average of $0.0099s$ (approximately $19$ times faster). Note that our algorithm runs on a GPU because it is easy to parallelize the many independent optimal transport computations while FGW runs on CPU.

\subsection{Graph Coarsening}\label{par:coarsening} 
We also carry out an experiment as a proof of concept on how the discounted WL distance can be used to coarsen graphs. The goal is to coarsen a simple oriented circle of $n=30$ nodes as in the barycenter experiment into a graph with a given number $m$ of nodes. In order to obtain a coarsening of size $m$, we minimize an objective on the space of Markov chains of size $m$, similarly to the barycenter experiment.

Let $M^\text{target}$ denote the Markov matrix of the target graph and let $M^\text{coarsened}$ denote the Markov matrix of the coarsened graph.
A natural objective would be to minimize the $\delta$-discounted WL distance between the original graph and the coarsened graph. 
This naive approach, however, does not yield good results. An explanation is the following:
if the coarsened graph is 4 times smaller (in terms of the number of nodes), then one step of random walk in the coarsened graph should intuitively correspond to 4 steps of random walks in the original graph.
In this way, one should think of a coarsened graph as a Markov chain with a 
larger ``time step'' than the original graph and hence one should think of the coarsened graph 
and the original graph induce Markov chains with different time scales.

To adapt to the different time scales, we propose to instead minimize the following objective
$$ L = \dWL<C(l)>[\delta][\epsilon]{\infty}{(M^\text{target})^k, M^\text{coarsened}}$$
where $k = \floor*{\frac{n}{m}}$ is the coarsening factor.

Where $l$ is the set of labels $x_i, y_i$ given to the nodes of the target, and $C(l)$ is the cost matrix so that$C(l)(i, j) = \norm{(x_i, y_i) - (x^\text{target}_j, y^\text{target}_j) }$

We use the same parametric markov kernels as in the barycenter experiment (\appendixref{par:parametric_markov_kernels}) and we minimize the objective using the Adam optimizer with a learning rate of $0.005$ and $3000$ iterations.
The results are outlined in \figureref{fig:coarsening}. We observe that the algorithm gives better results when the coarsened size $m$ is a divisor of the original size $n$. This hence implies an interesting question of how to coarsen graphs into arbitrary sizes. We leave this as a future work.

\begin{figure}
\begin{center}
  \includegraphics[width=\textwidth]{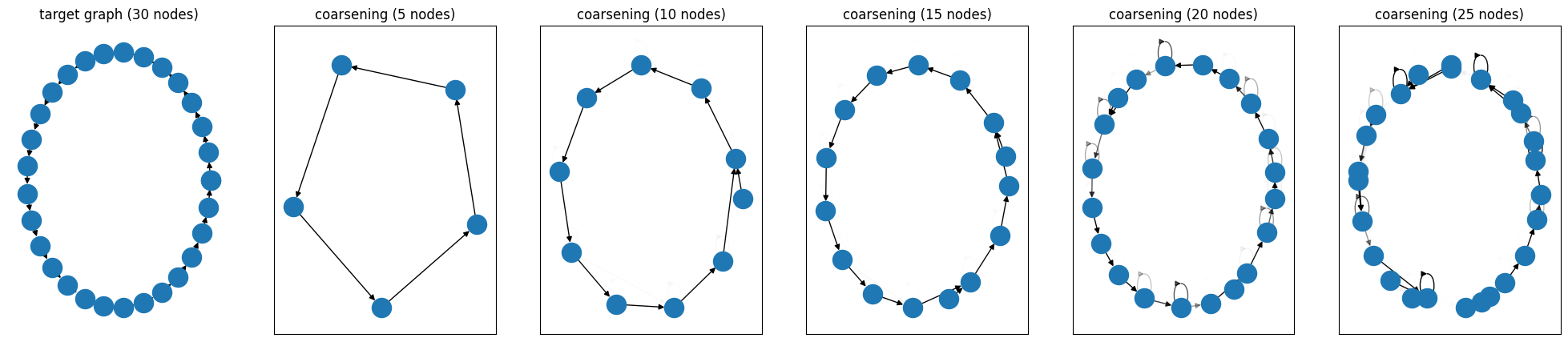}
\end{center}
\caption{Coarsening results on a circle graph of size 30. The original graph is on the left, the subsequent graphs are coarsenings of different sizes.}
\label{fig:coarsening}
\end{figure}

\section{New Results on the WL Distance}\label{sec: more on dWL}

In this section, we introduce some new results on the WL distance introduced by \citet{chen2022weisfeilerlehman}.
Those results justify our motivation for introducing new distances by showing some of the flaws we mentioned in Section~\ref{sec: limitation}.

Although the original WL distance is defined for Markov chains with stationary initial distributions, \cref{eq:dwlk} can be adapted to define a quantity for Markov chains with arbitrary initial distributions.
We can thus define the depth-$k$ WL distance for any Markov chains.

It turns out that the depth-$k$ WL distance can be also computed iteratively. We first introduce some constructions.

\begin{definition}\label{def:wlck}
Given $k\in\N$, we define $C^{(l)}$ for $l=0,\ldots,k$ recursively as follows.
    \begin{equation}C^{(0)}_{ij} = C_{ij},\label{eq:wl_costmatrix0}\end{equation}
    \begin{equation}C^{(l)}_{ij} = \dW<C^{(l-1)}>{m^{\setX}_i, m^{\setY}_j}. \label{eq:wl_costmatrix}\end{equation}
Note that $C^{(l)}=C^{\delta, (l)}$ when $\delta = 0$, where $C^{\delta, (l)}$ is the cost matrix involved in the computation of the discounted WL distance (see Proposition \ref{prop:wlregnotconstant}). 
\end{definition}
Those matrices $C^{(l)}$ coincides with the cost matrix computed in the $l$th iteration in \citep[Algorithm 1]{chen2022weisfeilerlehman} to compute the depth-$k$ WL distance (for a special type of initial cost function $C$) in the following way:
\[\dWL{k}{\mX,\mY}= \dW<C^{(k)}>{\nu^\setX,\nu^\setY}.\]

Notice that, in fact, those matrices above are themselves WL distances in a certain way. More precisely, for any $k\in\N$, one has that
\begin{equation} \dWL<C>{k}{(\setX,m^{\setX}_\bullet, \delta_i), (\setY,m^{\setY}_\bullet, \delta_j)}= \dW<C^{(k)}>{\delta_i,\delta_j}= C^{(k)}_{ij}.\label{eq:Cij}\end{equation}
We analyze $\dWL{k}{}$ when $k$ approaches $\infty$ as follows. 

\begin{proposition}[Convergence of $\dWL{k}{}$ is independent of initial distributions]\label{prop:cvwlinit}
  Given a finite set $\setX = \setY$, assume that $m^{\setX}_\bullet$ and $m^{\setY}_\bullet$ are \emph{irreducible} and \emph{aperiodic} Markov transition kernels. Assume also that the cost is defined as a pseudometric in $\setX$ (This is for example true if the cost is defined as in \citet{chen2022weisfeilerlehman} or in Remark~\ref{rmk:different dwl}.)
Then, for any $\nu^{\setX}\in\mathcal{P}(\setX)$ and $\nu^{\setY}\in\mathcal{P}(\setY)$, the limit
\begin{equation} \lim_{k\rightarrow\infty}\dWL{k}{(\setX,m^{\setX}_{\bullet}, \nu^{\setX}), (\setY,m^{\setY}_{\bullet}, \nu^{\setY})} \label{eq:indepdistrib}\end{equation}
exists and is independent of choices of $\nu^{\setX}\in\mathcal{P}(\setX)$ and $\nu^{\setY}\in\mathcal{P}(\setY)$.
\end{proposition}

Note that this property is only true for \emph{irreducible} and \emph{aperiodic} Markov chains. 
This is a common property in the study of finite Markov chains. See for example \citep{levin2017markov} for a good introduction to that theory.
A Markov chain is irreducible if all states can be reached from any other state (including itself) in finite positive number of steps (or equivalently if the graph) with positive probability.
A Markov chain is aperiodic if for any state $s$, one has that $\gcd \{k\in\N:\,  s \text{ can be reached from } s \text{ in time } k  \} = 1$.
Irreducibility and aperiodicity ensure the existence of a unique stationary distribution for a finite Markov chain~\citep[Corollary 1.17]{levin2017markov} and its convergence towards that distribution~\citep[Theorem 4.9]{levin2017markov}. 

In this way, for any irreducible and aperiodic finite Markov chains $\mX,\mY$, we redefine $\dWL{\infty}{}$ (introduced in \cref{eq:alternatewldefinition}) as
\[\dWL<C>{\infty}{\mX,\mY}:=\lim_{k\rightarrow \infty}\dWL<C>{k}{\mX,\mY}.\]
By the proposition above, we know that $\dWL{\infty}{}$ is independent of choice of initial distributions. In particular,
when the initial distributions are stationary, this new definition coincides with the original definition in \cite{chen2022weisfeilerlehman} since in this case $\dWL{k}{\mX,\mY}$ is increasing w.r.t. $k$.

\subsection{Proof of Proposition~\ref{prop:cvwlinit}}

In this section, we prove Proposition~\ref{prop:cvwlinit}. 
The proof of Proposition~\ref{prop:cvwlinit} is based on the following observation.

\begin{lemma}[$\dWL{\infty}{}$ does not distinguish initial distributions with the same transitions]\label{thm:wlinit} 
Let $\setX$ be a finite set and let $m^{\setX}_\bullet$ denote an irreducible and aperiodic Markov transition kernel on $\setX$. Assume the assumptions of Proposition~\ref{prop:cvwlinit}.

Then, for any  \(\nu_1,\nu_2\in\mathcal{P}(\setX)\), one has that
\[\lim_{k\rightarrow\infty}\dWL<C>{k}{(\setX,m^{\setX}_\bullet, \nu_1), (\setX,m^{\setX}_\bullet, \nu_2)} = 0,\]
where the cost matrix is defined under the assumptions of Proposition~\ref{prop:cvwlinit} as $C(x,x') = d_{\bm{X}}(x, x')$
\end{lemma}

\begin{proof}\label{proof:wlinit} 
We use symbols $(X_t)_{t\in\N}$ and $(Y_t)_{t\in\N}$ to denote realizations for the two Markov chains $(\setX,m^{\setX}_\bullet, \nu_1)$ and $ (\setX,m^{\setX}_\bullet, \nu_2)$, respectively.
Then, by definition we have that
\begin{equation}
\dWL{k}{(\setX,m^{\setX}_\bullet, \nu_1), (\setX,m^{\setX}_\bullet, \nu_2)} = 
\inf_{\text{Markovian coupling }(X_t,Y_t)_{t\in\N}}
\E C(X_k,Y_k).
\end{equation}

Consider a stochastic process \((X_t, Y_t)_{t\in\N}\) defined as follows:
\[
\begin{cases}
\text{if } X_t \neq Y_t, &\text{ then } \begin{cases} 
  X_{t+1} \sim m^{\setX}_{X_t} \\
  Y_{t+1} \sim m^{\setX}_{Y_t}
  \end{cases}  \text{ independently;}
  \\
\text{if } X_t = Y_t, &\text{ then } X_{t+1} = Y_{t+1} \sim m^{\setX}_{X_t}.
\end{cases}
\]
Then, \((X_t, Y_t)_{t\in\N}\) is clearly a time homogeneous Markovian coupling. This coupling has been used for studying convergence of Markov chains, 
and often called the "classical coupling" (for example, by \citet{Griffeath1975}) since it has the following property:

\begin{equation}
\lim_{k\rightarrow\infty}\Pr(X_k \neq Y_k) = 0.\label{eq:stationarycv}\end{equation}
For completeness we provide a proof of the equation above later. 
In fact, we prove something stronger: there exists $0 \leq \rho \leq 1,t_0\in N,0<\epsilon \leq 1$ depending on the two Markov chains such that
\begin{equation}\label{eq:exponential growth}
    \Pr(X_{t} \neq Y_{t})<(1-\epsilon)^{\lfloor\frac{t}{t_0}\rfloor}\rho.
\end{equation}
This equation implies that \(\lim_{t\rightarrow\infty}\Pr(X_{t} \neq Y_{t})\) converges to 0
exponentially fast. 
Given this equation, we have that
\begin{align}
    &\dWL{k}{(\setX,m^{\setX}_\bullet, \nu_1), (\setX,m^{\setX}_\bullet, \nu_2)} 
\leq \E C(X_k, Y_k) \\
&=\Pr(X_k = Y_k)\times 0 + {\Pr(X_k \neq Y_k)} 
\E({C(X_k,Y_k)} 
\mid X_k \neq Y_k)\\
&\leq (1-\epsilon)^{\lfloor\frac{k}{t_0}\rfloor}\rho\cdot \|C\|_\infty.\label{eq:dwl exp growth}
\end{align}
Hence, 
\[\lim_{k\rightarrow\infty}\dWL{k}{(\setX,m^{\setX}_\bullet, \nu_1), (\setX,m^{\setX}_\bullet, \nu_2)} = 0.\]
Now, we finish the proof by proving \cref{eq:stationarycv}.

\begin{proof}[Proof of \cref{eq:exponential growth}]\label{proof:eqstationarycv}
Let
\(S = \setX \times \setX\) denote the state space of the Markovian
coupling \((X_t, Y_t)_{t\in\N}\).
Let \(E := \{ (x, x):\, x\in \setX\} \subseteq S\).
For any state \(s_0 = (x_0, y_0) \in S\), we define a stopping time as follows
\[T_S := \inf\{t:\, X_t = Y_t \}.\]
We know that if \((X_t, Y_t) \in E\), then all
the subsequent elements also are (by definition of the coupling).

Let \(t_0^{s_0} := \inf\{t:\, \Pr(T_S \leq t \mid (X_0, Y_0) = s_0) > 0\}\) for any $s_0\in S$. 
Note that
\begin{itemize}
\item
  if \(s^0\) in \(E\), then \(t_0^{s_0} = 0\) trivially;
\item
 since $m^\setX_\bullet$ is irreducible and aperiodic, one has that \(t_0^{s_0} < \infty\) for any $s_0\in S$: suppose (by contradiction) that \(t_0^{s_0} = \infty\) for some $s_0=(x_0,y_0)\in S\backslash E$. Let $x\in \setX$ be any state. Given that $(X_0, Y_0) = s_0$, $X_t\neq Y_t$ for all $t\in \N$ almost surely. Then, by \citep[Proposition 1.7]{levin2017markov},
there exists $r_0\in\N$ so that $\forall t \geq r_0, \Pr(X_t = x | X_0 = x_0) > 0 ~\text{and}~ \Pr(Y_t = x | Y_0 = y_0) > 0$.
By definition, since $X_t\neq Y_t$ for all $t\geq 1$, $X_t$ is independent of $Y_t$. Then,
\[\Pr(X_t=x, Y_t=x | (X_0, Y_0) = s_0)= \Pr(X_t=x | (X_0, Y_0) = s_0)\Pr(Y_t=x | (X_0, Y_0) = s_0)> 0.\]
This contradicts the fact that $X_t\neq Y_t$ for all $t\in \N$ almost surely. Hence, $t_0^{s_0}<\infty$ for all $s_0$.
\end{itemize}

Let \(t_0 := \max_{s_0} t_0^{s_0}\), and
\(\epsilon := \inf_{s_0} \Pr(T_S \leq t_0 \mid (X_0, Y_0) = s_0)\). Note that $\epsilon>0$ since
\(t_0\geq t_0^{s_0}\). {Furthermore, $t_0$ and $\epsilon$ are independent of initial distributions.}
Then, for any \(n\in\N\), we have that
\begin{align*}
    &\Pr(X_{(n + 1)t_0} \neq Y_{(n+1)t_0})\\
    &=\underbrace{\Pr(X_{(n + 1)t_0} \neq Y_{(n+1)t_0} \mid 
  X_{nt_0} = Y_{nt_0})}_{=0} \times\Pr(X_{nt_0} = Y_{nt_0}) \\
  & + {\Pr(X_{(n + 1)t_0} \neq Y_{(n+1)t_0} \mid 
  X_{nt_0} \neq Y_{nt_0})} \times\Pr(X_{nt_0} \neq Y_{nt_0})\\
  & =  \sum_{s_0\in S\backslash E} \frac{\Pr(X_{(n + 1)t_0} \neq Y_{(n+1)t_0} \mid 
  (X_{nt_0}, Y_{nt_0})=s_0) \times \Pr( (X_{nt_0}, Y_{nt_0})=s_0) }{\Pr(X_{nt_0} \neq Y_{nt_0}) }\times\Pr(X_{nt_0} \neq Y_{nt_0}) \\
  & =  \sum_{s_0\in S\backslash E} {\Pr(T_S>t_0 \mid 
  (X_{0}, Y_{0})=s_0) \times \Pr( (X_{nt_0}, Y_{nt_0})=s_0) } \\
  &\;< (1 - \epsilon) \Pr(X_{nt_0} \neq Y_{nt_0}).
\end{align*}
We let $\rho:=\Pr(X_{0} \neq Y_{0})$. Assume that $\rho>0$ (otherwise $\nu_1=\nu_2$ and the conclusion holds trivially).
Then, by the computation above, one has that
\begin{equation*}
    \Pr(X_{nt_0} \neq Y_{nt_0})<(1-\epsilon)^n\rho.
\end{equation*}
By our construction of the Markovian coupling $(X_t,Y_t)_{t\in\N}$, we know that \(\Pr(X_t \neq Y_t)\) is decreasing in $t$ since
\(X_t = Y_t \implies X_{t+1} = Y_{t+1}\). Hence, for any $t\in\N$, we have that 
\begin{equation*}
    \Pr(X_{t} \neq Y_{t})<(1-\epsilon)^{\lfloor\frac{t}{t_0}\rfloor}\rho.
\end{equation*}
This concludes the proof.
\end{proof}
\end{proof}

Now Proposition~\ref{prop:cvwlinit} follows directly from the lemma above.
\begin{proof}[Proof of Proposition~\ref{prop:cvwlinit}] 
Let $\mu^{\setX}$ and $\mu^{\setY}$ be the unique stationary distributions for $m^{\setX}_\bullet$ and $m^{\setY}_\bullet$, respectively.
Then, by the triangle inequality we have that
\begin{align*}
    \limsup_{k\rightarrow \infty}\dWL{k}{(\setX,m^{\setX}_{\bullet}, \nu^{\setX}), (\setY,m^{\setY}_{\bullet}, \nu^{\setY})} &\leq \underbrace{\limsup_{k\rightarrow \infty}\dWL{k}{(\setX,m^{\setX}_{\bullet}, \nu^{\setX}), (\setX,m^{\setX}_{\bullet}, \mu^{\setY})}}_{0}\\
 &\,+ \limsup_{k\rightarrow \infty}\dWL{k}{(\setX,m^{\setX}_{\bullet}, \mu^{\setX}), (\setY,m^{\setY}_{\bullet}, \mu^{\setY})}\\
 &\,+ \underbrace{\limsup_{k\rightarrow \infty}\dWL{k}{(\setY,m^{\setY}_{\bullet}, \mu^{\setY}), (\setY,m^{\setY}_{\bullet}, \nu^{\setY})}}_{0} \\
 &=\dWL{\infty}{(\setX,m^{\setX}_{\bullet}, \mu^{\setX}), (\setY,m^{\setY}_{\bullet}, \mu^{\setY})}.
\end{align*}
Similarly,
\begin{align*}
    \dWL{\infty}{(\setX,m^{\setX}_{\bullet}, \mu^{\setX}), (\setY,m^{\setY}_{\bullet}, \mu^{\setY})} &=\liminf_{k\rightarrow \infty}\dWL{k}{(\setX,m^{\setX}_{\bullet}, \mu^{\setX}), (\setY,m^{\setY}_{\bullet}, \mu^{\setY})}\\ &\leq \underbrace{\liminf_{k\rightarrow \infty}\dWL{k}{(\setX,m^{\setX}_{\bullet}, \mu^{\setX}), (\setX,m^{\setX}_{\bullet}, \nu^{\setY})}}_{0}\\
 &\,+ \liminf_{k\rightarrow \infty}\dWL{k}{(\setX,m^{\setX}_{\bullet}, \nu^{\setX}), (\setY,m^{\setY}_{\bullet}, \nu^{\setY})}\\
 &\,+ \underbrace{\liminf_{k\rightarrow \infty}\dWL{k}{(\setY,m^{\setY}_{\bullet}, \nu^{\setY}), (\setY,m^{\setY}_{\bullet}, \mu^{\setY})}}_{0} \\
 &=\liminf_{k\rightarrow\infty}\dWL{\infty}{(\setX,m^{\setX}_{\bullet}, \nu^{\setX}), (\setY,m^{\setY}_{\bullet}, \nu^{\setY})}.
\end{align*}
Therefore, $\lim_{k\rightarrow \infty}\dWL{k}{(\setX,m^{\setX}_{\bullet}, \nu^{\setX}), (\setY,m^{\setY}_{\bullet}, \nu^{\setY})} $ exists furthermore, for any $\nu^{\setX}$ and $\nu^{\setY}$, we have that
\[\lim_{k\rightarrow \infty}\dWL{k}{(\setX,m^{\setX}_{\bullet}, \nu^{\setX}), (\setY,m^{\setY}_{\bullet}, \nu^{\setY})}=\dWL{\infty}{(\setX,m^{\setX}_{\bullet}, \mu^{\setX}), (\setY,m^{\setY}_{\bullet}, \mu^{\setY})}.\]

\end{proof}

\hypertarget{consequences}{%
\subsection{Convergence of the WL Distance}\label{sec:convergence_wl}}
In this section, we establish that some convergence results on $\dWL{k}{}$ as $k$ increases.

We assume that $m^{\setX}_{\bullet}$ and $m^{\setX}_{\bullet}$ are irreducible and aperiodic Markov transition kernels on $\setX$ and $\setY$, respectively. Let $\mu^\setX$ and $\mu^\setY$ denote their respective unique stationary distributions. We also assume that the cost matrix is defined as in Proposition~\ref{prop:cvwlinit}.
\begin{proposition}\label{prop:Cconverges} 
\((C^{(k)})_{k\in\N}\) converges to a constant matrix.
\end{proposition}
\begin{proof}
    By \cref{eq:Cij} and Proposition~\ref{prop:cvwlinit}, one has that for any $i,j$,
    \[\lim_{k\rightarrow\infty}C^{(k)}_{ij}=\lim_{k\rightarrow \infty}\dWL{k}{(\setX,m^{\setX}_{\bullet}, \delta_i), (\setY,m^{\setY}_{\bullet}, \delta_j)}=\dWL{\infty}{(\setX,m^{\setX}_{\bullet}, \mu^{\setX}), (\setY,m^{\setY}_{\bullet}, \mu^{\setY})}.\]
\end{proof}

In fact, we can provide an estimate of the convergence rate of $\dWL{k}{}$ in the case when $C$ is a pseudometric.
\begin{theorem}[$C^{(k)}$ converges exponentially with a pseudometric cost]\label{thm:speedofconvergence}
Suppose that $\setX = \setY$ is a pseudometric space, and that $C := d_\setX$ is the pseudometric on $\setX$.
If we let 
\[c:=\dWL{\infty}{(\setX,m^{\setX}_{\bullet}, \mu^{\setX}), (\setY,m^{\setY}_{\bullet}, \mu^{\setY})},\]
then there exists a rate of convergence $0 \leq  \rho < 1 $ dependent on $m_\bullet^\setX$ and $m_\bullet^\setY$ such that
\begin{equation}
  \forall (i, j), \,|C^{(k)}_{ij} -  c| \leq 2\rho^k \norm{C}_\infty. \label{eq:speedofconvergence}
\end{equation}
\end{theorem}

As a direct consequence, we have that
\begin{corollary}
For any initial distributions $\nu^{\setX}$ and $\nu^{\setY}$, the quantity
$\dWL{\infty}{(\setX,m^{\setX}_{\bullet}, \nu^{\setX}), (\setY,m^{\setY}_{\bullet}, \nu^{\setY})}$ converges to $\dWL{\infty}{(\setX,m^{\setX}_{\bullet}, \mu^{\setX}), (\setY,m^{\setY}_{\bullet}, \mu^{\setY})}$ \emph{exponentially} fast.
\end{corollary}

\begin{proof}[Proof of theorem \ref{thm:speedofconvergence}]
  Since $C$ is assumed to be a pseudometric, and $\dWL{k}{}$ is an OTM distance, using Proposition~\ref{prop:OTM_is_distance}, we can use the triangular inequality on $\dWL{k}{}$.
Let $1 \leq i, u \leq n$, $1 \leq j, l \leq m$, and $k>0$. Then, one has that
\begin{align*}
  \mathllap{|C^{(k)}_{ij} - C^{(k)}_{ul}|}& =\mid \dWL{k}{(\setX,m^{\setX}_{\bullet},\delta_i), (\setY,m^{\setY}_{\bullet}, \delta_j)} - \dWL{k}{(\setX,m^{\setX}_{\bullet},\delta_u), (\setY,m^{\setY}_{\bullet}, \delta_l)}\mid \\
&\leq   \dWL{k}{(\setX,m^{\setX}_{\bullet},\delta_i), (\setX,m^{\setX}_{\bullet}, \delta_u)} + \dWL{k}{(\setY,m^{\setY}_{\bullet},\delta_j), (\setY,m^{\setY}_{\bullet}, , \delta_l)}.
\end{align*}

Then, from \cref{eq:dwl exp growth}, we have that there exists  $t_0\in \N,0<\epsilon \leq 1$ such that

\begin{equation*}
  \abs{C^{(k)}_{ij} - C^{(k)}_{ul}} < 2(1-\epsilon)^{\lfloor\frac{k}{t_{0}}\rfloor}\cdot \|C\|_\infty.
\end{equation*}

If we let $\rho := (1 - \epsilon)^{\frac{1}{ t_0}}$, then for any $k\in\N$, one has that
\begin{equation*}
  \abs{C^{(k)}_{ij} - C^{(k)}_{ul}} < 2\rho^k \norm{C}_\infty.
\end{equation*}
Therefore,
\begin{equation}
  \max_{ij}C^{(k)}_{ij} - \min_{ij} C^{(k)}_{ij} < 2\rho^k \norm{C}_\infty.
  \label{eq:maxminspeedcv}
\end{equation}

Using the shorthand $\min C^{(k)} = \min_{ij} C^{(k)}_{ij}$ and $\max C^{(k)} = \max {ij} C^{(k)}_{ij}$, then we have the following inequalities:
\begin{equation}
  \min C^{(0)} \leq \min C^{(1)} \leq \ldots \leq \min C^{(\infty)} = 
  c = \max C^{(\infty)} \leq \ldots \leq \max C^{(1)} \leq \max C^{(0)}.
  \label{eq:orderofmins}
\end{equation}
This is a direct consequence of the following inequalities: for any $k>0$,
\begin{equation}
  \min C^{(k)} \leq \min C^{(k+1)} \leq  \max C^{(k+1)} \leq \max C^{(k)}.
  \label{eq:orderofmins2}
\end{equation}
We prove \cref{eq:orderofmins2} as follows. Let $k>0$ and let $i, j, u, v$ be such that
$\min C^{(k+1)}= C^{(k+1)}_{ij}$ 
and $\max C^{(k+1)} = C^{(k+1)}_{uv}$. Then,
\begin{equation}
  C^{(k+1)}_{ij} = \dW<C^{(k)}>{m^\setX_i, m^\setY_j} \geq \min C^{(k)}
  \label{eq:orderofmins3}
\end{equation}
and 
\begin{equation}
  C^{(k+1)}_{ul} = \dW<C^{(k)}>{m^\setX_u, m^\setY_l}\leq \max C^{(k)}.
  \label{eq:orderofmins4}
\end{equation}
where the inequalities in \cref{eq:orderofmins3} and \cref{eq:orderofmins4} follows directly from the definition of optimal transport (the optimal transport cost is smaller than the maximal cost and bigger than the minimal cost).

Then, using \cref{eq:orderofmins} and  \cref{eq:maxminspeedcv}, we conclude the proof for \cref{eq:speedofconvergence}.
\end{proof}

\section{Theoretical Details on OTM Distances}\label{sec:OTM_details}

In this section, we add conceptual details that can help getting a detailed understanding of OTM distances, 
and in particular the discounted WL distance.

\subsection{Computation of Finite-Time OTM Distances and Usage of the Memoryless Property in the  Discounted WL Distance}\label{sec:OTM_computation}

While we did not include it in the main text for the sake of simplicity, 
it is possible to compute exactly the value of any finite-time OTM distance,
in a way that is similar to the computation of the discounted WL distance.

\begin{theorem}\label{thm:OTM_computation}
  Let $p$ be a finitely supported probability distribution on $\N$.
  Define $n$ as the maximal value in the support of $p$ (i.e., $n=\max\{t:p(t)>0\}$). Consider two Markov chains
  $\mX$ and $\mY$ as well as a cost matrix $C$ on $\setX\times \setY$. Then, $\dOTM<C>{p}{\mX,\mY}$ can be computed in the following way:

\begin{equation}
  \dOTM<C>{p}{\mX,\mY} = \dW{\nu^\setX, \nu^\setY, C^{p, (n)}}
  \label{eq:OTM_computation}
\end{equation}
where the iterated cost matrices are defined as
\begin{equation}
C^{p, (t)}_{i,j} = \inf_{\substack{(X,Y) \sim \Pi(\mX,\mY)}} 
\expect{C (X_T, Y_T) |T \geq n - t\,\text{and}\,(X_{n-t}, Y_{n-t}) = (i, j)}
  \label{eq:OTM_costmatrixdef}
\end{equation}
where $T\sim p$ is independent of $X$ and $Y$.
The cost matrices can be computed recursively as follows:
\begin{gather}
C^{p, (0)}_{i,j} = C(i,j) \\
C^{p, (t)}_{i,j} = \delta_t C (i, j) + (1 - \delta_t) \dW<C^{p, (t-1)}> {m^\setX_i, m^\setY_j}
\label{eq:OTM_costmatrixcomp}
\end{gather}
where $\delta_t = \mathbb{P} (T = n - t \mid T \geq n - t )$.

\end{theorem}

A proof of this theorem will be provided at the end of this section. 

We remark that \Cref{eq:OTM_computation} has a simpler form in the case of the geometric distribution 
(ie. in the case of the $\delta$-discounted WL distance).

First note that, in the case of the (truncated) geometric distribution $p_\delta^k$, $\delta_t= \delta$ for all $t=1,\ldots,n$.

Second, the geometric distribution has the so-called \emph{memoryless} property: Recall the notation $p_\delta^\infty(t) = \delta(1-\delta)^t$ for the geometric distribution with parameter $\delta$.
Assume $T \sim p_\delta^\infty$.
Then, 
\begin{equation}
  \Pr(T = t+s | T \geq t) = \delta (1 - \delta)^{t} = \Pr(T = s)
  \label{eq:memoryless}
\end{equation}

Now, for $k\in\N$, let $T^k = \min(T, k) \sim p_\delta^k$.
In this case, a variant of the memoryless property holds as long as $k \geq t+s$ where $t,s\in\N$:
\begin{equation}
  \Pr(T^k = t+s | T^k \geq t) = \Pr(T^{k-t} = s).
  \label{eq:memoryless_truncated}
\end{equation}
The proof is provided later. One immediate consequence of the formula which we will use soon is the following formula:
\begin{equation}
  \Pr(T^k = t+s | T^k \geq t) = \Pr(T^{k+1} = t+1+s | T^{k+1} \geq t+1).
  \label{eq:independence of k}
\end{equation}

Now, based on the formula above, one has, for $k \geq t$:
\begin{align*}
  & C^{p_\delta^k, (t)}_{i,j} \\
  &\qquad= \inf_{(X,Y) \sim \Pi(\mX,\mY)} 
      \expect{C (X_{T^k}, Y_{T^k}) |T^k \geq k - t \cap (X_{k-t}, Y_{k-t}) = (i, j)}
  \\ &\qquad= \inf_{(X,Y) \sim \Pi(\mX,\mY)} 
  \expect{C (X_{T^{k+1} - 1}, Y_{T^{k+1} - 1}) |T^{k+1} \geq k + 1 - t \cap (X_{k-t}, Y_{k-t}) = (i, j)} \text{ (\Cref{eq:independence of k})}
  \\ &\qquad= \inf_{(X,Y) \sim \Pi(\mX,\mY)} 
  \expect{C (X_{T^{k+1}}, Y_{T^{k+1}}) |T^{k+1} \geq k + 1 - t \cap (X_{k +1 -t}, Y_{k + 1 -t}) = (i, j)} \text{ (Markov property) }
  \\ &\qquad= C^{p_\delta^{k+1}, (t)}_{i,j},
\end{align*}

This shows that the cost matrices $C^{p_\delta^k, (t)}$ are independent of $k$ for $k \geq t$. Thus, we can define a single $C^{\delta, (t)}$ for all truncated geometric distributions $p_\delta^k$, ($C^{\delta, (t)} := C^{p_\delta^k, (t)}$ for any $k\geq t$).
Since $\delta_t = \delta$ is independent of $k$ and $t$ as well, we get the simplified formula for computing the discounted WL distance presented in \propositionref{prop:wlreg_recursive}:
\begin{gather*}
C^{\delta, (0)}_{ij} =  C_{ij}\\
C^{\delta, (l+1)}_{ij} = \delta C_{ij} + (1 - \delta)\, \dW<C^{\delta, (l)}>{m^{\setX}_i, m^{\setY}_j}. \label{eq:wlreg_costmatrix_bis}
\end{gather*}

Moreover, in particular $C^{\delta, (t)} = C^{p_\delta^t, (t)}$ for all $t$. 
Thus $\dWL[\delta]{t}{\mX, \mY}$ can be computed directly from $C^{\delta, (t)}$. This allows us to iterate all the way to convergence to $\dWL[\delta]{\infty}{\mX, \mY}$.
We could not use the same technique to compute the OTM distance with other infinitely-supported distribution $q$ than the geometric distribution $p_\delta^\infty$: 
if we tried to compute $\dOTM{q}{\mX, \mY}$ by computing it 
sequentially for truncated versions of $q$, we could not reuse the cost matrix of the previous iteration for the subsequent one 
because the  updates in \equationref{eq:OTM_costmatrixcomp} would be different depending on the cutoff.

Additionally, the simplification means the update is always the same, making this computation into a fixed point iteration. 
This in turn enables the differentiation (that is justified through fixed point properties).

\begin{proof}[Proof of \equationref{eq:memoryless_truncated}]\label{proof:memoryless_truncated}
First we assume that $k > t+s$:
\begin{align*}
  \Pr(T^k = t+s | T^k \geq t) &= \Pr(\min(T, k) = t+s | \min(T, k) \geq t) 
                           \\ &= \Pr(\min(T, k) = t+s | T \geq t) 
                           \\ &= \Pr(T = t+s | T \geq t) 
                           \\ &= \Pr(T = s) 
                           \\ &= \Pr(\min(T, k-t)= s) 
                           \\ &= \Pr(T^{k-t}= s).
\end{align*}

Then, we consider the case when $k = t+s$:
\begin{align*}
  \Pr(T^k = t+s | T^k \geq t) &= \Pr(\min(T, k) = t+s | \min(T, k) \geq t) 
                           \\ &= \Pr(\min(T, k) = t+s | T \geq t) 
                           \\ &= \Pr(T \geq t+s | T \geq t) 
                           \\ &= \Pr(T \geq s) 
                           \\ &= \Pr(\min(T, s)= s) 
                           \\ &= \Pr(T^{s} = s) 
                           \\ &= \Pr(T^{k-t} = s).
\end{align*}
\end{proof}

\begin{proof}[Proof of \theoremref{thm:OTM_computation}]
  First note that for $t = n$, \equationref{eq:OTM_costmatrixdef} reduces to 
\begin{equation}
  C^{p, (n)}_{i,j} = \inf_{\substack{(X,Y) \sim \Pi(\mX, \mY)}}\expect{C (X_T, Y_T)|(X_0, Y_0) = (i, j)}.\label{eq:OTM_costmatrixdef_stepn}
\end{equation}

Thus 
\begin{align*}
   &\dW<C^{p, (n)}>{\nu^\setX, \nu^\setY}
\\ &\qquad= \inf_{X_0' \sim \nu^\setX, Y_0' \sim \nu^\setY} \expect{C^{p, (n)} (X_0', Y_0')}
\\ &\qquad= \inf_{X_0' \sim \nu^\setX, Y_0' \sim \nu^\setY} \expect*{\inf_{\substack{(X,Y) \sim \Pi(\mX,\mY)}}\expect{C (X_T, Y_T)|(X_0, Y_0) = (X_0', Y_0')}}
\\ &\qquad= \inf_{(X,Y) \sim \Pi(\mX, \mY)}\expect{C (X_T, Y_T)} \text{ (Markov property)}
\\ &\qquad= \dOTM<C>{p}{\mX, \mY}.
\end{align*}

Now we prove the recursive formula (\equationref{eq:OTM_costmatrixcomp}):

\begin{align*}
  C^{p, (t+1)}_{i,j} &= \inf_{\substack{(X,Y) \sim \Pi(\mX,\mY)}} 
    \expect{C (X_T, Y_T) |T \geq n - t -1\cap (X_{n-t - 1}, Y_{n-t -1}) = (i, j)}
                  \\ &= \inf_{\substack{(X,Y) \sim \Pi(\mX,\mY)}}
    \Pr(T = n - t - 1 | T \geq n - t - 1) \expect{C(X_{n-t - 1}, Y_{n-t -1})|(X_{n-t - 1}, Y_{n-t -1}) = (i, j)} \\
    &\qquad\qquad\quad\,\,\,\,+ 
    \Pr(T \geq n - t | T \geq n - t - 1) \expect{C (X_T, Y_T) |T \geq n - t\cap (X_{n-t - 1}, Y_{n-t -1}) = (i, j)}
                  \\ &=\inf_{\substack{(X,Y) \sim \Pi(\mX,\mY)}}
    \delta_{t + 1} C(i, j) + (1 - \delta_{t + 1}) \expect{C (X_T, Y_T) |T \geq n - t\cap (X_{n-t - 1}, Y_{n-t -1}) = (i, j)}
                  \\ &= \delta_{t + 1} C(i, j) + (1 - \delta_{t + 1}) \inf_{\substack{(X,Y) \sim \Pi(\mX,\mY) }} \expect{C (X_T, Y_T) |T \geq n - t\cap (X_{n-t - 1}, Y_{n-t -1}) = (i, j)}
                  \\ &= \delta_{t + 1} C(i, j) + (1 - \delta_{t + 1}) \inf_{\substack{(X,Y) \sim \Pi(\mX,\mY) }} \expect{C (X_T, Y_T) |T \geq n - t\cap (X_{n-t}, Y_{n-t}) \sim (m^\setX_i, m^\setY_j)}
  \\ &= \delta_{t + 1} C(i, j) + (1 - \delta_{t + 1}) \dW<C^{p, (t)}>{m^\setX_i, m^\setY_j}.
\end{align*}
      
\end{proof}

\subsection{A MDP Interpretation of the Discounted WL Distance}\label{sec:MDP}

In this section, we interpret the algorithm for solving the discounted WL distance (see Proposition \ref{prop:wlreg_recursive}) as a value iteration process on a certain kind of Markov Decision Process (MDP). 
This interpretation is not novel, and follows from previous literature \citep{moulos2021bicausal,o2022optimal}, but it gives a different point of view on our iterative algorithm in Proposition \ref{prop:wlreg_recursive}.

Markov decision processes (MDPs) are a type of model at the center of reinforcement learning theory \citep{sutton2018reinforcement}. 
An MDP is defined by a state space $S$, a collection of action spaces $(A_s)_{s\in S}$, a transition distribution $P: \coprod_{s\in S} A_s \to \mathcal{P}(S)$, and a transition cost $c: S \times A \times S\rightarrow \R$.

A policy (or a strategy) is a map $\pi: (s, t) \in S \times \N \mapsto a \in A_s$ which sends a state (and a time) to an action.

Then, given such a policy $\pi$ and an initial distribution $\nu^S \in \mathcal{P}(S)$,
we can define a (time-inhomogeneous) Markov chain $\mathcal{S}_\pi = (S,(m_{\pi}^{S,(t)})_{t\in\N}, \nu^S)$ whose transition kernels are defined as follows:
\[
  m_{\pi}^{S,(t)} (s) = P(\pi(s, t)),\quad\forall t\in\N
\]

The goal of MDP theory is to find a policy that minimizes an expected cost $\bar{c} = \E_{S \sim \mathcal{S_\pi}} c(S)$ where $c(S)$ is defined based on the different transitions made in S, and the transition cost $c(s_t, a_t, s_{t+1})$.
Examples of the most frequent costs include:
\begin{itemize}
  \item finite or infinite-time discounted cost 
    (with length $k\in\N\cup\{\infty\}$ and discount factor $\beta \in \N$):
    \[c(S) = \sum_{t=0}^{k-1} \beta^t c(S_t, \pi(S_t), S_{t+1}); \]

  \item finite time average cost
    (with length $k\in\N$):
    \[c(S) = \frac 1k\sum_{t=0}^{k-1} c(S_t, \pi(S_t), S_{t+1}); \]

  \item infinite time average cost
    \[c(S) = \lim_k \frac 1k\sum_{t=0}^{k-1} c(S_t, \pi(S_t), S_{t+1}); \]

  \item and others such as finite or infinite time total costs.
\end{itemize}

Such problems can generally be solved by two techniques: 
value iteration and policy iteration.

In fact, \(C^{\delta, (l)}\) defined in Proposition \ref{prop:wlreg_recursive} is the $l$th step cost matrix
of a value iteration on the discounted cost Markov Decision Process
\((S, (A_s)_{s\in S}, P, c)\)
where
\begin{itemize}
\item
  \(S = {\setX} \times {\setY};\)
\item
  \(A_{x, y} = \mathcal{C}(m^{\setX}_x, m^{\setY}_y);\)
\item
  \(P( (x', y')|(x, y), a) = a(x, y);\)
\item
  \(c((x, y), a) = C(x,y)\).
\end{itemize}

We note that this interpretation is not novel, and we list it here only for illustration purpose: As mentioned in Remark \ref{rmk:bicaual ot}, \(\dWL[\delta]{\infty}{}\) coincides with the bicausal optimal transport problem studied in \citet{moulos2021bicausal} when binary cost and discount factor \((1 - \delta)\) is considered. \cite{moulos2021bicausal} solves this bicausal optimal transport problem by interpreting it as the same MDP as we described above. 
Furthermore, this MDP interpretation has also been used in  \citet{o2022optimal} for computing the \gls{dOTC}. However, \cite{o2022optimal} used an average cost (instead of discounted cost) for devising their algorithm.

\begin{remark}
  Theorem \ref{thm:regwlcv} can be interpreted as a convergence result between discounted and average cost MDPs: since $\dWL[\delta]{\infty}{}$ is the expected discounted cost of the MDP defined above, 
  and $\dOTC{}$ is defined as the average cost of an MDP, this result can be interpreted as analogous to the classical result that the cost of a discounted-cost finite MDP converges to that of an average-cost MDP if the discount factor approaches zero; see, for example, \citet{puterman2014markov} for more details on this result.
  Note, nevertheless, that this classical result holds only for finite state and finite {action spaces}, thus it was not applicable here since our action space is infinite; see \appendixref{sec:MDP} for more details.
\end{remark}

\section{Algorithm and Complexity}\label{sec:algorithm}

In this section, we give algorithmic details on how we compute the discounted WL distance (with Sinkhorn regularization), and its gradient.

We make extensive use of the PyTorch framework \citep{pytorch} to accelerate our code, and we integrate our gradient algorithm into its automatic differentiation engine, 
so as to make it usable as an optimization target.

The GPU-accelerated code used to compute the distance and its gradient is available as a python library\footnote{\url{https://pypi.org/project/ot-markov-distances/}}, installable with \verb|$ pip install ot_markov_distances|

\subsection{Computation and Differentiation of the Discounted WL Distance}

\paragraph{Forward pass – computation of the distance} The method we use to compute the depth-$\infty$ discounted WL distance is based on Proposition~\ref{prop:wlreg_recursive}.
It is described in a simplified way in Algorithms \ref{alg:computation_simple_k} and \ref{alg:computation_simple}.

\begin{algorithm2e}[htb!]
\LinesNumbered
\TitleOfAlgo{Computation of depth-$k$ discounted WL distance}\label{alg:computation_simple_k}
\KwIn{ %
  \(m^\setX\):  [n, n] float array \;
  \(m^\setY\):  [m, m] float array \;
  \(\nu^\setX\): [n] float array \;
  \(\nu^\setY\): [m] float array \;
  \(C\): [n, m] float array \;
  \(\delta\): The discount parameter for the WL distance.\;
  \(\epsilon\): Parameter for the Sinkhorn divergence\;
  \(k\): depth \;
}
\KwOut{%
  \(\dWL<C>[\delta][\epsilon]{k}{(\setX,m^\setX, \nu^\setX), (\setY,m^\setY, \nu^\setY)}\) \;
}
\(C_\text{current} = C\) \;
\ForEach{$0 \leq l < k$}{
  \(C_\text{new} = (0)_{0\leq i < n, 0 \leq j < m}\) \;
  \ForEach{\(0\leq i < n, 0 \leq j < m\)}{\label{algline:computation_simple_k_loop}
  \(C_\text{new}[i, j] = \delta C[i,j] + (1-\delta) \dW^\epsilon<C_\text{current}>{m^\setX_i, m^\setY_j}\) \;
  }
  \(C_\text{current} = C_\text{new}\)\;
}
Compute \( d = \dW^\epsilon<C_\text{current}>{\nu^\setX, \nu^\setY} \) \;
\KwRet{\(d\)}
\end{algorithm2e}

\begin{algorithm2e}[htb!]
\LinesNumbered
\SetKw{KwSave}{save for backwards pass}
\SetKwInput{KwStores}{Stores}
\TitleOfAlgo{Computation of depth-$\infty$ discounted WL distance}\label{alg:computation_simple}
\KwIn{ %
  $m^\setX$:  [n, n] float array \;
  $m^\setY$:  [m, m] float array \;
  $\nu^\setX$: [n] float array \;
  $\nu^\setY$: [m] float array \;
  $C$: [n, m] float array \;
  $\delta$: The discount parameter for the WL distance.\;
  $\epsilon$: Parameter for the Sinkhorn divergence\;
}
\KwOut{%
  $\dWL<C>[\delta]{\infty}{(\setX,m^\setX_\bullet, \nu^\setX), (\setY,m^\setY_\bullet,\nu^\setY)}$ \;
}
\KwStores{%
$f$: [n, m, n], $g$: [n, m, m] float arrays: dual solutions of the last step Sinkhorn computations\;
$P$ [n, m, n, m] float array :primal solutions (i.e., optimal matchings) for the last step Sinkhorn computation\;}
$C_\text{current} = C$ \;
\Repeat{$C_\text{current}$ converges}{
  $C_\text{new} = (0)_{0\leq i < n, 0 \leq j < m}$ \;
  $f = (0)_{0\leq i < n, 0 \leq j < m, 0\leq u < n}$  \;
  $g = (0)_{0\leq i < n, 0 \leq j < m, 0\leq l < m}$  \;
  $P = (0)_{0\leq i < n, 0 \leq j < m, 0\leq u < n, 0\leq l < m}$  \;
  \ForEach{$0\leq i < n, 0 \leq j < m$}{\label{algline:computation_simple_loop}
    $C_\text{new}[i, j] = \delta C[i,j] + (1-\delta) \dW^\epsilon<C_\text{current}>{m^\setX_i, m^\setY_j}$ \;
    $f[i,j], g[i, j], P[i, j] = \text{ Dual and primal solutions of the calculation above}$
  }
  $C_\text{current} = C_\text{new}$
}
Compute $d = \dW^\epsilon<C_\text{current}>{\nu^\setX, \nu^\setY}$\;
\KwSave{$f, g, P$} \;
\KwRet{$d$}
\end{algorithm2e}

Algorithms \ref{alg:computation_simple_k} and \ref{alg:computation_simple} are slightly simplified: we note that in both of them, the inner ``foreach" loop (line \ref{algline:computation_simple_k_loop} for \algorithmref{alg:computation_simple_k} and \ref{algline:computation_simple_loop} for \algorithmref{alg:computation_simple}) is embarassingly parallel. 
In practice, we use GPU acceleration to run the operations in this loop simultaneously.

\gls{sinkhorn}s are computed using the method from \citet{feydyInterpolatingOptimalTransport2019}.

Note also that for the depth-$\infty$ version, we take care of saving the primal and dual solutions of every optimal transport computation. These results will be used for the computation of the gradient in Algorithm~\ref{alg:backwardpass_simple}.

The complexity of one pass of computing $\dW^\epsilon<C_k>{m^\setX_i, m^\setY_j}$ is $O(n_s n m)$ where $n_s$ is the number of iterations necessary for the Sinkhorn computation to converge.

Thus, the total complexity of this algorithm is 
$O\left(n_s n_d (n m)^2\right)$, where $n_d$ is the number of iterations needed to converge.

\paragraph{Backward pass – computation of the gradient: $k$ finite} If the distance was computed for a finite (small) depth-$k$, we can use PyTorch’s automatic differentiation engine \citep{pytorch}: 
since all operations done are successive \glspl{sinkhorn}, we only implement 
the differentiation of a \gls{sinkhorn} according to
\citep[Proposition 4.6 and 9.2]{peyreComputationalOT2018}, 
and use the automatic differentiation engine  to apply the chain rule, to compute the full gradient of the distance.

Note that this approach only works for a small $k$: the subsequent chain rule applications induce risk of numerical error, as well as the usual problems associated with big differentiation graphs, e.g., high memory footprint, risk of gradient vanishing or explosion, etc.

\paragraph{Backward pass – computation of the gradient: $k=\infty$} If the distance was computed until convergence, we provide two things:
\begin{enumerate}
  \item A function that compute the backward pass of the \gls{sinkhorn}, like in the previous paragraph
  \item A function that computes the gradient of the $C^{\delta, (\infty)}$ matrix, defined in Algorithm~\ref{alg:backwardpass_simple}, against the parameters $m^\setX, m^\setY$ and $C$.
\end{enumerate}

Since the depth-$\infty$ discounted WL distance is computed as 
\[
  \dWL[\delta]{\infty}{\mX, \mY} = \dW<C^{\delta, (\infty)}(m^\setX, m^\setY, C)>{\nu^\setX, \nu^\setY},
\] provided with those two functions, the automatic differentiation engine is able to apply chain rule
to differentiate $\dWL[\delta]{\infty}{\mX, \mY}$ against all parameters $\nu^\setX, \nu^\setY, m^\setX, m^\setY$ and $C$.

More precisely, denoting by $\loss$ the value of the loss that we want to differentiate. Denote:

\begin{gather}
\nabla^{C^{\delta,(\infty)}} \loss_{ij} := \frac{\partial \loss}{\partial C^{\epsilon,\delta, (\infty)}_{ij}}\\
\nabla^{C} \loss_{kl} := \frac{\partial \loss}{\partial C_{kl}}, \\
\nabla^{m^\setX} \loss_{kk'} := \frac{\partial \loss}{\partial m^{\setX}_{kk'}}\\
\nabla^{m^\setY} \loss_{ll'} := \frac{\partial \loss}{\partial m^{\setY}_{ll'}}
\end{gather}

The backward pass of the depth-$\infty$ discounted WL distance should input $\nabla^{C^{\delta,(\infty)}} \loss$ and output $G^C, G^X$ and $G^Y$.

Using the notations of the proof of differentiation Section~\ref{proof:differentiation} 
$K := I_{nm} - (1 - \delta)P,$ then we have:
\begin{gather}
\nabla^{C} \loss = \Delta^T \nabla^{C^{\delta,(\infty)}} \loss = \delta (K^T)^{-1} \nabla^{C^{\delta,(\infty)}} \loss\\
\nabla^{m^\setX} \loss = \Gamma^T \nabla^{C^{\delta,(\infty)}} \loss = (1 - \delta) F^T (K^T)^{-1} \nabla^{C^{\delta,(\infty)}} \loss\\
\nabla^{m^\setY} \loss = \Theta^T \nabla^{C^{\delta,(\infty)}} \loss = (1 - \delta) G^T (K^T)^{-1} \nabla^{C^{\delta,(\infty)}} \loss
\end{gather}

Thus, we save some GPU power by applying above formulae,
and computing $(K^T)^{-1} \nabla^{C^{\delta,(\infty)}} \loss$ only once.
Note also that $(K^T)^{-1} \nabla^{C^{\delta,(\infty)}} \loss$ can be computed with linear equations solving 
primitives instead of matrix inversion, for more efficiency and stability.

\begin{algorithm2e}[htb!]
\SetAlgoLined
\LinesNumbered
\SetKw{KwRestore}{restore saved variables}
\TitleOfAlgo{Gradient computation for the depth-$\infty$ discounted WL distance}\label{alg:backwardpass_simple}
\KwIn{ %
  $\nabla^{C^{\delta,(\infty)}} \loss$: [n, m] float array, 
  value of the gradient 
  $(\frac{\partial \loss}{\partial C^{\delta,(\infty)}_{ij}})_{0\leq i < n, 0 \leq j < m}$ \;
}
\KwOut{%
 $\nabla^{m^\setX} \loss$ [n, n] float array, 
  value of the gradient 
  $(\frac{\partial \loss}{\partial m^{X}_{ij}})_{0\leq i < n, 0 \leq j < n}$ \;
 $\nabla^{m^\setY} \loss$ [m, m] float array, 
  value of the gradient 
  $(\frac{\partial \loss}{\partial m^{Y}_{ij}})_{0\leq i < m, 0 \leq j < m}$ \;
    $\nabla^{C} \loss$: [n, m] float array, 
  value of the gradient 
  $(\frac{\partial \loss}{\partial C})_{0\leq i < n, 0 \leq j < m}$ \;
}
\KwRestore{$f, g, P$ stored in \algorithmref{alg:computation_simple} or \algorithmref{alg:computation_simple_scheduling}}\;
Compute
  \(P := \left(P_{ij}^{kl}\right){}_{0 \leq i < n, 0 \leq j < m }^{ 0 \leq k < n, 0 \leq l < m}\)\;
  \(F := \left(F_{ij}^{kk'}\right){}_{0 \leq i < n, 0 \leq j < m }^{ 0 \leq k < n, 0 \leq k' < n} 
  = \left(f_{ij}^{k'}\1_{i = k}\right){}_{0 \leq i < n, 0 \leq j < m }^{ 0 \leq k < n, 0 \leq k' < m}\)\;
  \(G := \left(G_{ij}^{ll'}\right){}_{0 \leq i < n, 0 \leq j < m }^{ 0 \leq l < m, 0 \leq l' < m} 
  = \left(g_{ij}^{l'}\1_{i = l}\right){}_{0 \leq i < n, 0 \leq j < m }^{ 0 \leq l < m, 0 \leq l' < m}\)\;
$(\nabla^{C^{\delta,(\infty)}} \loss)$.reshape(n m) \;
$F$.reshape($n^2, nm$) \;
$G$.reshape($m^2, nm$) \;
$P$.reshape($nm, nm$) \;
$K := I_{nm} - (1 - \delta)P$ \;
$L := \left(K^T\right)^{-1} \nabla^{C^{\delta,(\infty)}} \loss$ \;
$\nabla^{C} \loss       = \delta L$\;
$\nabla^{m^\setX} \loss = (1-\delta) F L$\;
$\nabla^{m^\setY} \loss = (1-\delta) G L$\;
$\nabla^{C} \loss      $.reshape(n, m)\;
$\nabla^{m^\setX} \loss$.reshape(n, n)\;
$\nabla^{m^\setY} \loss$.reshape(m, m)\;
\KwRet{
  $
  \nabla^{C} \loss, 
  \nabla^{m^\setX} \loss, 
  \nabla^{m^\setY} \loss, 
  $
}
\end{algorithm2e}

\paragraph{Backward pass complexity} The matrix $K$ has size $nm \times nm$. 
The matrix $\nabla^{C^{\delta,(\infty)}} \loss$, is viewed as a vector of size $nm$. 
Computing $L := \left(K^T\right)^{-1} \nabla^{C^{\delta,(\infty)}} \loss$ thus has complexity $C_\text{solve}(nm)$ 
where $C_\text{solve}(k)$ is the complexity of the linear solver for $k$ equations with $k$ unknowns.
If the solver used is LAPACK\cite{lapack}, $C(k) = O(k^3)$. Theoretically the complexity is lower: from \citet{bunch1974},
we know that the complexity of solving this equation is the same as the complexity of a multiplication,
which is theoretically $C(k) = O(k^\omega)$, where $\omega < 2.371866$ \citep{duan2022faster}. So in theory the complexity for this step is $O((nm)^\omega)$ but in practice solvers will have $O((nm)^3)$.
Then the subsequent matrix multiplications have at most the same complexity (the complexity of matrix multiplications since F and G are smaller than $nm \times nm$)

Thus the total complexity of the backward pass is $O((nm)^\omega)$ theoretically but $O((nm)^3)$ in practice. 

\paragraph{Acceleration using Sinkhorn scheduling}\label{par:sinkhorn_scheduling_approx} The procedure in \algorithmref{alg:computation_simple} can in fact be accelerated using a trick related to \glspl{sinkhorn} and fixed point algorithms. 
Since $C^{\delta, (\infty)}$ can be defined as the unique fixed point of \cref{eq:wlreg_costmatrix} (from Proposition~\ref{prop:wlregnotconstant}), we are guaranteed to reach the right result as long as that result is a fixed point of \cref{eq:wlreg_costmatrix}.
In particular, it does not matter if some steps are approximative during the algorithm as long as the convergence is reached. Given this observation, we can accept some intermediate steps to be approximate.
Thus, to accelerate the algorithm, we can replace the first steps with a good and faster approximation of the right iteration. 
In practice one such way (mentioned, for example, by \citet{peyreComputationalOT2018}) is to cap the number of iterations in the Sinkhorn computation, and to use the result even if it has not fully converged.
In this way we obtain \algorithmref{alg:computation_simple_scheduling}, which is empirically faster than \algorithmref{alg:computation_simple}.

\begin{algorithm2e}[htb!]
\LinesNumbered
\SetKw{KwSave}{save for backwards pass}
\SetKwInput{KwStores}{Stores}
\TitleOfAlgo{Computation of depth-$\infty$ discounted  WL distance with Sinkhorn scheduling}\label{alg:computation_simple_scheduling}
\KwIn{ %
  $m^\setX$:  [n, n] float array \;
  $m^\setY$:  [m, m] float array \;
  $\nu^\setX$: [n] float array \;
  $\nu^\setY$: [m] float array \;
  $C$: [n, m] float array \;
  $\delta$: The discount parameter for the WL distance.\;
  $\epsilon$: Parameter for the Sinkhorn divergence\;
  sinkhorn\_update\_size: integer \;
}
\KwOut{%
$\dWL<C>[\epsilon, \delta]{\infty}{(\setX,m^\setX_\bullet, \nu^\setX), (\setY,m^\setY_\bullet, \nu^\setY) }$ \;
}
\KwStores{%
$f$: [n, m, n], $g$: [n, m, m] float arrays: dual solutions of the last step Sinkhorn computations\;
$P$ [n, m, n, m] float array :primal solutions (i.e., optimal matchings) for the last step Sinkhorn computation\;}

$C_\text{current} = C$ \;
$n_{\text{sinkhorn}} = 1$\;
\Repeat{$C_\text{current}$ has converged and all Sinkhorn computations have converged}{
  $C_\text{new} = (0)_{0\leq i < n, 0 \leq j < m}$ \;
  $f = (0)_{0\leq i < n, 0 \leq j < m, 0\leq u < n}$  \;
  $g = (0)_{0\leq i < n, 0 \leq j < m, 0\leq l < m}$  \;
  $P = (0)_{0\leq i < n, 0 \leq j < m, 0\leq u < n, 0\leq l < m}$  \;
  \ForEach{$0\leq i < n, 0 \leq j < m$}{\label{algline:computation_simple_scheduling_loop}
    $C_\text{new}[i, j] =  \delta C[i,j] + (1-\delta) \dW^\epsilon<C_\text{current}>{m^\setX_i, m^\setY_j}$ limiting the number of iterations to $n_{\text{sinkhorn}}$ \;
    $f[i,j], g[i, j], P[i, j] = \text{ Dual and primal solutions of the calculation above}$
  }
  $C_\text{current} = C_\text{new}$\;
  \If{one of the Sinkhorn computations did not converge}{%
    $n_{\text{sinkhorn}} +=$ sinkhorn\_update\_size \;
  }
}
Compute $d = \dW^\epsilon<C_\text{current}>{\nu^\setX, \nu^\setY }$\;
\KwSave{$f, g, P$} \;
\KwRet{$d$}
\end{algorithm2e}

\paragraph{Initialization} In the same vein, as long as we are computing the depth-$\infty$ distance, the initialization of the matrix $C_\text{current}$ does not matter. We let $C_0$ denote the initial value of that variable.
Empirically we find that \algorithmref{alg:computation_simple} and \algorithmref{alg:computation_simple_scheduling} converge the fastest when we select $C_0 = \delta C$, (compared to $C_0 = 0$ or $C_0 = C$ which are other sensible choices).
This choice of initialization relates to the procedure one obtains if 
instead of computing $\dWL[\delta]{k}{}$, i.e., the OTM distance related to $p_\delta^{k}$ 
as defined before \definitionref{corollary:probabilistic}, 
we compute the distance related to $p_\delta^{\prime k}$ where
$p_\delta^{\prime k}(t) = P(T_\delta^\infty = t | T_\delta^\infty < k)$
where $T_\delta^\infty \sim \mathcal{G}(\delta)$.

\subsection{Acceleration for Sparse Markov Chains}
Sometimes, the Markov chains encountered are sparse, i.e., the transition kernel matrices of Markov chains are sparse.
Whenever this happens, we exploit this to develop algorithms for faster computation of the discounted WL distance.

Let $\alpha\in\mathcal{P}(\setX)$. Then, we let 
  $\supp \alpha := \{ x \in \setX, \alpha_x >0\}$ denote the support of $\alpha$.
Given a Markov chain $\mX = (\setX, m^\setX_\bullet, \nu^\setX)$, we define $\supp_\mX x := \supp m^\setX_x$
for each $x\in\setX$. We further let $\deg_\mX x := \abs{\supp_\mX x}$ and let $d_\mX := \max_{x \in \setX} \deg_\mX x$.

In this section, we propose a modified version of Algorithm~\ref{alg:computation_simple_k}, 
that can compute $\dWL[\delta]{k}{\mX, \mY}$ 
in time $O(k l_s n m d_\mX d_\mY )$ (which is a performance boost when $d_\mX << n$ or $d_\mY << m$), where $l_s$ is the number of iterations needed for Sinkhorn to converge.

This accelerated version of the algorithm is based on the following observation: if $\alpha$ (resp. $\beta$) are probability measures on $\setX$ (resp. $\setY$)
\begin{equation}
  \dW<C>{\alpha, \beta} = \dW<C_{|\supp \alpha \times \supp \beta}>{\alpha_{|\supp \alpha}, \beta_{|\supp \beta}}
\label{eq:sparseot}\end{equation}
where $\alpha_{|\supp \alpha}$ (resp. $\beta_{|\supp \beta}$) denotes the distribution induced by $\alpha$ on its support, and $C_{|\supp \alpha \times \supp \beta}$ denotes the restriction of $C$ to $\supp \alpha \times \supp \beta$.
And the same holds true for the \gls{sinkhorn}:
\begin{equation}
  \dW^\epsilon<C>{\alpha, \beta} = \dW^\epsilon<C_{|\supp \alpha \times \supp \beta }>{\alpha_{|\supp \alpha}, \beta_{|\supp \beta}}
\label{eq:sparsesk}\end{equation}
\cref{eq:sparseot} and \cref{eq:sparsesk} are direct consequences of the definitions of \glspl{dW} and \glspl{sinkhorn}.

We use this simple observation to devise \algorithmref{alg:computation_accel_k}. 
Now, the computation of  
  $C_{l+1}[i, j] = \dW^\epsilon<C_{l | \supp m^\setX_i \times \supp m^\setY_j}>{m^\setX_{i  | \supp m^\setX_i}, 
    m^\setY_{j|\supp m^\setY_j}}$
can be done in time only $O(l_s d_\mX d_\mY)$, which is a substantial acceleration compared to the original time complexity $O(l_s nm)$  when the Markov chains are of low degree.

\begin{algorithm2e}[htb!]
\LinesNumbered
\TitleOfAlgo{Computation of depth-$k$ discounted WL distance with sparse Markov kernels}\label{alg:computation_accel_k}
\KwIn{ %
  \(m^\setX\):  [n, n] float array \;
  \(m^\setY\):  [m, m] float array \;
  \(\nu^\setX\): [n] float array \;
  \(\nu^\setY\): [m] float array \;
  \(C\): [n, m] float array \;
  \(\delta\): The discount parameter for the WL distance.\;
  \(\epsilon\): Parameter for the Sinkhorn divergence\;
  \(k\): depth
}
\KwOut{%
\(\dWL<C>[\epsilon, \delta]{k}{(\setX,m^\setX, \nu^\setX), (\setY,m^\setY, \nu^\setY) }\) \;
}
\(C_\text{current} = C\) \;
\ForEach{\(0 \leq l < k\)}{
  \(C_\text{new} = (0)_{0\leq i < n, 0 \leq j < m}\) \;
  \ForEach{\(0\leq i < n, 0 \leq j < m\)}{\label{algline:computation_accel_k_loop}
    $C_\text{new}[i, j] = \delta C[i,j] + (1-\delta) \dW^\epsilon<C_{\text{current} | \supp m^\setX_i \times \supp m^\setY_j}>{m^\setX_{i  | \supp m^\setX_i}, 
    m^\setY_{j|\supp m^\setY_j}, }$\label{line:computation} \;
  }
  \(C_\text{current} = C_\text{new}\)
}
Compute \(d = \dW^\epsilon<C_\text{current}>{\nu^\setX, \nu^\setY}\)\;
\KwRet{\(d\)}
\end{algorithm2e}

We could also continue until convergence in a similar way as \algorithmref{alg:computation_simple}. 
But the dual solutions of the computation at line~\ref{line:computation} cannot be directly used as the gradient for the whole distributions, thus we need to extend them using the technique from \citep[Proposition 2 and comments below]{feydyInterpolatingOptimalTransport2019}.
That operation requires recomputing one iteration of Sinkhorn with the full (non-restricted) distributions, which is $O(nm)$. Moreover, this needs to be done $n m$ times, ending in a $O((nm)^2)$ complexity. Additionally, the matrix inversions we do in \algorithmref{alg:backwardpass_simple} are also $O(( nm )^2)$. Thus, in this case, we cannot accelerate the backwards pass.

\section{Proofs and Technical Details}\label{sec:proofs}

\subsection{Preliminary Results}

In this section, we will present some lemmas that will be useful in subsequent proofs.

Let us start with a well-known alternative formulation for optimal transport
as a linear programming problem, 
in the finite case, that allows us to justify taking optimal matchings.

\begin{lemma}[Equation (2.11) in \citet{peyreComputationalOT2018}]\label{lm:otlinearprog}
Since $\setX$ and $\setY$ are two finite sets, without loss of generality, we identify them to $\{1\ldots n\}$ and $\{1\ldots m\}$, respectively. Then, the cost function $C: \{1\ldots n\} \times \{1\ldots m\}\to \R_+$ can be seen as a matrix in $\R_+^{n \times m}$. 
Then, the \gls{dW} can be expressed as the following linear program:
\begin{equation}
\dW<C>{\alpha, \beta} = \min_P \inner{P}{C} \label{eq:otdef}
\end{equation}
where the minimum is taken over the (compact) subspace $[0, 1]^{n\times m} $ in which each $P$ satisfies that $\sum_{i=1}^nP_{iy}=\beta(y)$ and $\sum_{j=1}^nP_{xj}=\alpha(x)$ for any $x=1,\ldots,n$ and $y=1,\ldots,m$, i.e., the compact space of distributions whose marginals are $\alpha$ and $\beta$. 
See \citep{peyreComputationalOT2018} for more background on these notions.

In particular, there always exists a coupling $(X, Y)$ that verifies the infimum in \cref{eq:wassersteindef}.
\end{lemma}

\begin{lemma}[Optimal transport is 1-lipschitzian in the cost matrix]\label{lemma:ot1lip}
Let $\setX$ and $\setY$ be finite sets such that $n:=|\setX|$ and $m=|\setY|$. Let $C_1, C_2 \in \R^{n\times m}_+$ denote two cost matrices between $\setX$ and $\setY$, and let $\alpha \in\mathcal{P}(\setX), \beta\in\mathcal{P}(\setY)$.
Then, for any $\epsilon\geq 0$, one has that
\begin{equation}
    \abs{\dW^\epsilon<C_1>{\alpha, \beta} - \dW^\epsilon<C_2>{\alpha, \beta}} \leq \norm{C_1 - C_2}_\infty. \label{eq:ot1lip}\end{equation}
\end{lemma}
\begin{proof}[proof of Lemma~\ref{lemma:ot1lip}]
 Without loss of generality, we assume that  $\dW^{\epsilon}<C_1>{\alpha,\beta}\geq \dW^\epsilon<C_2>{\alpha,\beta}$. 
   Let $(X^\epsilon,Y^\epsilon)$ be an optimal coupling for $\dW^\epsilon<C_2>{\alpha,\beta}$.
   Then, we have that
   \begin{align*}
      |\dW^\epsilon<C_1>{\alpha,\beta}-\dW^\epsilon<C_2>{\alpha,\beta}|&\leq \mathbb{E}(C_1(X^\epsilon,Y^\epsilon) - \epsilon H(X^\epsilon, Y^\epsilon)-\mathbb{E}C_2(X^\epsilon,Y^\epsilon) + \epsilon H(X^\epsilon, Y^\epsilon))\\
      &\leq \mathbb{E}|C_1(X^\epsilon,Y^\epsilon)-C_2(X^\epsilon,Y^\epsilon)|\\
      &\leq \|C_1-C_2\|_\infty.
   \end{align*}
   This concludes the proof.
\end{proof}

\begin{lemma}[Continuity of Banach fixed point (Corollary 1.4 in \cite{pata2019fixed})]\label{lm:fixed point}   
Let $Z$ and $\Lambda$ be metric spaces and assume that $Z$ is complete.
    Let $F:Z\times\Lambda\rightarrow{Z}$ be a continuous map. Assume that there exists $\alpha\in[0,1)$ such that for each $\lambda\in\Lambda$, $F_\lambda:Z\rightarrow Z$ is $\alpha$-Lipschitz. 
    Then, for each $\lambda\in\Lambda$, $F_\lambda$ has a unique fixed point $z(\lambda)$ and the map $\lambda\mapsto z(\lambda)$ is continuous.

\end{lemma}  

\begin{lemma}[Differentiability of Banach fixed point]\label{lm:fixed_diff}
    Let $Z$ and $\Lambda$ be differential manifolds,
    and assume that $Z$ is complete.
    Let $F:Z\times\Lambda\rightarrow{Z}$ be a $\mathcal{C}^1$ map. 
    Denote $\diff_{|z}$ its differential along $Z$ and $\diff_{|\lambda}$ its differential along $\Lambda$
Moreover, suppose that for some $\lambda$, 
    $\id_Z - \diff{|z}F_{|(z(\lambda), \lambda)}$ is invertible. 
Then the fixed point $z(\lambda)$ exists for any $\lambda\in\Lambda$ and is differentiable in $\lambda$ and more explicitly
    \begin{equation}
      d z = (\id - \diff_{|z} F)^{-1} \diff_{|\lambda} F
    \end{equation}

\end{lemma}
\begin{proof}
Without loss of generality, we prove the result when $Z\subseteq\R^n$ and $\Lambda\subseteq\R^m$ are open subsets in Euclidean spaces and the general result follows from taking charts in manifolds.
We define $G:Z\times \Lambda\rightarrow\R^n$ by letting $G(z,\lambda):=F(z,\lambda)-z$ for any $z\in Z$ and $\lambda\in\Lambda$.
Then, $G$ is also continuously differentiable. The differential (or Jacobian) of $G$ w.r.t. $z$ is computed as $\diff_{|z}G=\diff_{|z}F-\id$. By assumption, we know that $\diff_{|z}G$ is invertible and hence the implicit function theorem applies: there exists a unique differentiable function $z:U\rightarrow Z$ defined on a neighborhood of $\lambda$ such that $G(z(\lambda),\lambda)=0$ and $dz = -\diff_{|z}G^{-1}\diff_{|\lambda} G$.
This means that $F(z(\lambda),\lambda)=z(\lambda)$ and 
\begin{equation}
      d z = (\id - \diff_{|z } F)^{-1} \diff_{|\lambda} F,
    \end{equation}
which concludes the proof.
\end{proof}

We also provide the following gluing lemma for Markovian couplings, useful for the study of OTM distances:

  \begin{lemma}[Gluing lemma for Markovian couplings]\label{lm:gluing lemma}
    Let $(X_t, Z_t^1)_{t\in\N}\in\Pi(\mX,\mathcal{Z})$ and $( Z_t^2,Y_t)_{t\in\N}\in\Pi(\mathcal{Z},\mY)$ be Markovian couplings. 
    Then, there exists a (time inhomogeneous) Markov chain on $\setX\times\setZ\times\setY$
    \[(X'_t, Z'_t, Y'_t)_{t\in\N}\]
    so that
    $(X'_t, Z'_t)_{t\in\N}\sim (X_t, Z_t^1)_{t\in\N}$, $(Z'_t, Y'_t)_{t\in\N}\sim (Z^2_t, Y_t)_{t\in\N}$ and furthermore $(X'_t, Y'_t)_{t\in\N}$ is a Markovian coupling between $\mX$ and $\mY$.
  \end{lemma}
\begin{proof}
We let $\nu^{\bm{XZ}}:=\law((X_0,Z_0^1))$ and $\nu^{\bm{ZY}}:=\law((Z_0^2,Y_0))$.
    For any $t\in\N$, let $m^{\bm{XZ},(t)}_{xz}:=\Pr((X_{t+1},Z_{t+1}^1)|(X_t,Z_t^1)=(x,z))\in\mathcal{C}(m_x^{\setX},m_z^{\bm{Z}})$ for any $x\in\setX$ and $z\in\bm{Z}$.
    Similarly, let $m^{\bm{ZY},(t)}_{zy}:=\Pr((Z^2_{t+1},Y_{t+1})|(Z^2_t,Y_t)=(z,y))\in\mathcal{C}(m_z^{\bm{Z}},m_y^{\bm{Y}})$ for any $z\in\bm{Z}$ and $y\in\setY$.

    By the Gluing Lemma \citep{villani2009optimal} for probability measures, one has that
    \begin{itemize}
        \item there exists $\nu^{\bm{XZY}}\in\mathcal{P}(\setX\times \bm{Z}\times \setY)$ whose marginals on $\setX\times \bm{Z}$ and on $\setZ\times \bm{Y}$ coincide with $\nu^{\bm{XZ}}$ and $\nu^{\bm{ZY}}$, respectively, and furthermore, the marginal on $\setX\times \setY$, denoted by $\nu^{\bm{XY}}$, is a coupling between $\nu^\setX$ and $\nu^\setY$;
        \item for any $x\in\setX,y\in\setY$ and $z\in\bm{Z}$, there exists $m^{\bm{XZY},(t)}_{xzy}\in\mathcal{P}(\setX\times \bm{Z}\times \setY)$ whose marginals on $\setX\times \bm{Z}$ and on $\setZ\times \bm{Y}$ coincide  with $m^{\bm{XZ},(t)}_{xz}$ and $m^{\bm{ZY},(t)}_{zy}$, respectively, and furthermore, the marginal on $\setX\times \setY$, denoted by $\nu^{\bm{XY}}$, is a coupling between $m^\setX_x$ and $m^\setY_y$.
    \end{itemize}

By the Kolmogorov extension theorem \citep{kolmogorov2018foundations}, there exists a Markov chain $(X'_t, Z'_t, Y'_t)_{t\in\N}$ with initial distribution $\nu^{\bm{XZY}}$ and transition kernels at each step $t\in\N$ defined by: $\Pr((X'_{t+1},Z'_{t+1},Y'_{t+1})|(X'_t,Z'_t,Y'_t)=(x,z,y)):=m^{\bm{XZY},(t)}_{xzy}$ for any $x\in\setX,y\in\setY$ and $z\in\bm{Z}$.
By construction, one obviously has that $(X'_t, Z'_t)_{t\in\N}\sim (X_t, Z_t^1)_{t\in\N}$, $(Z'_t, Y'_t)_{t\in\N}\sim (Z^2_t, Y_t)_{t\in\N}$ and that $(X'_t, Y'_t)_{t\in\N}$ is a Markovian coupling between $\mX$ and $\mY$.
\end{proof}

We end this section with an alternative yet direct description for the expected value involved in the definition of the OTM distances. 
\begin{lemma}\label{lm:compute expected value}
Given any Markovian coupling $(X_t,Y_t)_{t\in\N}$ between $\mX$ and $\mY$, let $\nu^{\setX \setY} \in \mathcal{C}(\nu^\setX, \nu^\setY)$ denote its initial distribution and let $(m^{\setX\setY,(t)}_{\bullet\bullet})_{t\in\N} \in \mathcal{C}(m^\setX_\bullet, m^\setY_\bullet)^{\N_+}$ denote its Markov transition kernels at each step $t\in\N:=\{0,1,2,\ldots\}$, where $\mathcal{C}(m^\setX_\bullet, m^\setY_\bullet)$ denotes the space of all Markov transition kernels $m^{\setX\setY}_{\bullet\bullet}$ such that $m^{\setX\setY}_{ij}\in\mathcal{C}(m^\setX_i, m^\setY_j)$ for all $i\in\setX$ and $j\in\setY$. 
Then,
\begin{equation}\label{eq ET in distribution}
    \E\,C(X_T,Y_T) = \sum_{t=0}^\infty p(t)\sum_{i_0,j_0,\ldots,i_t,j_t}  C_{i_t,j_t}m^{\setX\setY,(t-1)}_{i_{t-1}j_{t-1},i_tj_t}m^{\setX\setY,(t-2)}_{i_{t-2}j_{t-2},i_{t-1}j_{t-1}}\cdots m^{\setX\setY,(0)}_{i_0j_0,i_1j_1}\nu^{\setX\setY}_{i_0j_0}
\end{equation}
where $m^{\setX\setY,(t)}_{ij,kl}:=m^{\setX\setY,(t)}_{ij}({k,l})$ is a shorthand for the transition probability.
\end{lemma}

\begin{proof}
    By properties of expected values, one has that for any give $t\in\N$,
    \begin{align*}  \E\,C(X_t,Y_t)&=\sum_{i_t,j_t}C_{i_tj_t}\Pr(X_t=i_t,Y_t=j_t)\\
    &=\sum_{i_{t-1},j_{t-1},i_t,j_t}C_{i_tj_t}\Pr(X_t=i_t,Y_t=j_t|X_{t-1}=i_{t-1},Y_{t-1}=j_{t-1})\Pr(X_{t-1}=i_{t-1},Y_{t-1}=j_{t-1})\\
    &=\sum_{i_{t-1},j_{t-1},i_t,j_t}C_{i_tj_t}m^{\setX\setY,(t-1)}_{i_{t-1}j_{t-1},i_tj_t}\Pr(X_{t-1}=i_{t-1},Y_{t-1}=j_{t-1})\\
    &=\sum_{i_0,j_0,\ldots,i_t,j_t}  C_{i_t,j_t}m^{\setX\setY,(t-1)}_{i_{t-1}j_{t-1},i_tj_t}m^{\setX\setY,(t-2)}_{i_{t-2}j_{t-2},i_{t-1}j_{t-1}}\cdots m^{\setX\setY,(0)}_{i_0j_0,i_1j_1}\nu^{\setX\setY}_{i_0j_0}.
    \end{align*}
The last equality can be proved inductively. Since $T$ is independent of $(X_t,Y_t)_{t\in\N}$, \cref{eq ET in distribution} follows directly from the calculation above.
\end{proof}

\subsection{Proofs and Technical Details from Section \ref{generalized-optimal-transport-Markov-distances}}\label{sec:OTM proofs}

\begin{proof}[Proof of Remark~\ref{rmk: optimal markovian coupling}]\label{proof:optimalmarkovcoupling}
  Similarly to \lemmaref{lm:otlinearprog} we remark that, under our assumption that 
  the spaces $\setX$ and $\setY$ are finite of sizes $n$ and $m$ respectively, 
  we write them as $\{1\ldots n\}$ and $\{1\ldots m\}.$
Given any Markovian coupling $(X_t,Y_t)_{t\in\N}$, by Lemma \ref{lm:compute expected value}, we have that the value $\E\, C(X_T,Y_T)$ is completely determined by the initial distribution 
  \[
    \nu^{\setX \setY} \in \mathcal{C}(\nu^\setX, \nu^\setY)
  \]
and Markov transition kernels at each step $t\geq 0$ of $(X_t,Y_t)_{t\in\N}$:
  \[
    (m^{\setX\setY,(t)}_{\bullet\bullet})_{t\in\N} \in \mathcal{C}(m^\setX_\bullet, m^\setY_\bullet)^{\N}
  \]
where  $\mathcal{C}(m^\setX_\bullet, m^\setY_\bullet)$ denotes the space of all Markov transition kernels $m^{\setX\setY}_{\bullet\bullet}$ such that $m^{\setX\setY}_{ij}\in\mathcal{C}(m^\setX_i, m^\setY_j)$ for all $i\in\setX$ and $j\in\setY$. More precisely,
\begin{equation}\label{eq: continuous expect}
    \E\,C(X_T,Y_T) = \sum_{t=0}^\infty p(t)\sum_{i_0,j_0,\ldots,i_t,j_t}  C_{i_t,j_t}m^{\setX\setY,(t-1)}_{i_{t-1}j_{t-1},i_tj_t}m^{\setX\setY,(t-2)}_{i_{t-2}j_{t-2},i_{t-1}j_{t-1}}\cdots m^{\setX\setY,(0)}_{i_0j_0,i_1j_1}\nu^{\setX\setY}_{i_0j_0}.
\end{equation}

Note that for any $\alpha\in\mathcal{P}(\setX)$ and $\beta\in\mathcal{P}(\setY)$, the set of all couplings $\mathcal{C}(\alpha,\beta)$ can be identified with a compact subset in $[0,1]^{n\times m}$ (see also \lemmaref{lm:otlinearprog}). 
Then, 
\[\mathcal{C}(\nu^\setX, \nu^\setY)\times \mathcal{C}(m^\setX_\bullet, m^\setY_\bullet)^{\N_+}=\mathcal{C}(\nu^\setX, \nu^\setY)\times(\Pi_{x\in\setX,y\in\setY} \mathcal{C}(m^\setX_x, m^\setY_y))^{\N_+}\]
is a countable Cartesian product of compact spaces and it is hence compact.
Moreover, the right-hand side of \cref{eq: continuous expect} is obviously a continuous function defined on $\mathcal{C}(\nu^\setX, \nu^\setY)\times \mathcal{C}(m^\setX_\bullet, m^\setY_\bullet)^{\N_+}$.
Therefore, the infimum of the right-hand side of \cref{eq ET in distribution} is attainable in $\mathcal{C}(\nu^\setX, \nu^\setY)\times \mathcal{C}(m^\setX_\bullet, m^\setY_\bullet)^{\N_+}$ and hence we conclude the proof.
\end{proof}

\begin{proof}[Proof of Proposition~\ref{prop:dwl_lower_bound}]
Let $p$ distribution on $\N$, 
\begin{align*}
  \dOTM<C>{p}{\mX, \mY} &= \inf_{(X_t, Y_t)_{t\in\N}} \E_{T\sim p, T\indep (X_t, Y_t)}\, C(X_T,Y_T),
                     \\ &= \inf_{(X_t, Y_t)_{t\in\N}} \sum_{t\in\N}p(t)C(X_t,Y_t),
                     \\ &\geq\sum_{k\in\N}p(t) \inf_{(X_t, Y_t)_{t\in\N}} C(X_k,Y_k),
                     \\ &\qquad=\sum_{k\in\N} p(k)\dWL{k}{\mX, \mY}.
                     \\ &\qquad= \E_{T\sim p}({\dWL{T}{\mX, \mY}})
\end{align*}
\end{proof}

\begin{proof}[Proof of Proposition~\ref{prop:dotc_upper_bound}]
Let $p$ be a distribution on $\N$, and $\mX$ and $\mY$ be stationary Markov chains.

Similarly to Remark~\ref{rmk: optimal markovian coupling}, one can prove that there exists a stationary Markovian coupling that realizes Eq.~\ref{eq:dOTC}: 
one can prove that the space of stationary Markovian couplings is compact 
(as a closed subset of the space of Markovian couplings) 
in the same sense as in the proof of Remark~\ref{rmk: optimal markovian coupling}, and the same compactness argument gives the existence of the optimal coupling.

Consequently, let $(\ol{X}_t, \ol{Y}_t)$ be a realization of that coupling.

\begin{align*}
  \dOTM<C>{p}{\mX, \mY} 
   &= \inf_{(X_t, Y_t)_{t\in\N}} \E_{T\sim p, T\indep (X_t, Y_t)}\, C(X_T,Y_T),
\\ &\leq \E_{T\sim p, T\indep (\ol{X}_t, \ol{Y}_t)}\, C(\ol{X}_T,\ol{Y}_T)
\\ &\qquad= \E_{T\sim p, T\indep (\ol{X}_t, \ol{Y}_t)}\, C(\ol{X}_0,\ol{Y}_0)
\\ &\qquad=  C(\ol{X}_0,\ol{Y}_0) 
\\ &\qquad= \dOTC<C>{\mX, \mY}
\end{align*}
\end{proof}

\begin{proof}[Proof of Proposition~\ref{thm:zerosets}]
We, in fact, prove that the following 4 statements are equivalent:
\begin{enumerate}\denselist
  \item\label{itm:zerosetotc} $\dOTC{\mX, \mY} = 0$;
  \item\label{itm:zerosetotm} $\dOTM{p}{\mX, \mY} = 0$;
  \item\label{itm:zeroCoupling} there exists a Markovian coupling $(X_t, Y_t)_{t\in\N} \in \Pi(\mX, \mY)$ so that $\forall t\geq 0, C(X_t, Y_t) = 0$ almost surely;
  \item\label{item:zerosetallotm} for all distributions $q$ over $\N$, $\dOTM{q}{\mX, \mY} = 0$.
\end{enumerate}
\begin{itemize}
\item[\ref{itm:zerosetotc} $\implies$ \ref{itm:zerosetotm}:] This is a direct consequence of Proposition~\ref{prop:dotc_upper_bound}: 
  if $\dOTC{\mX,\mY} = 0$, 
  then  
  \begin{equation*}
    0 \leq \dOTM{p}{\mX, \mY} \leq \dOTC{\mX,\mY}  = 0.
  \end{equation*}
Thus, $\dOTM{p}{\mX, \mY} = 0$.
\item[\ref{itm:zerosetotm} $\implies$ \ref{itm:zeroCoupling}:] %
  Suppose $\dOTM{p}{\mX, \mY} = 0$. 
  Then, by Remark~\ref{rmk: optimal markovian coupling} there exists a Markovian coupling $(X_t, Y_t)_{t\in\N}$ such that:
  \begin{equation*}
    0 = \E (C(X_T,Y_T)).
  \end{equation*}
  Then, suppose (by contradiction) that there is $t_0$ so that 
  $p({t_0}) C(X_{t_0}, Y_{t_0}) > s>0$ with positive probability $\alpha>0$.
  Then, $\E (C(X_T,Y_T))\geq \alpha\cdot s>0$. Contradiction.

  Thus, $p(t) C(X_t, Y_t) = 0$ almost surely for each $t\in\N$ and hence $C(X_t, Y_t) = 0$ almost surely (since $p(t) >0$).

\item[\ref{itm:zeroCoupling} $\implies$ \ref{item:zerosetallotm}:] %
 This holds obviously.
  
\item[\ref{item:zerosetallotm} $\implies$ \ref{itm:zerosetotc}:] %
  Assume \ref{item:zerosetallotm}.
  We know from Theorem~\ref{thm:regwlcv} that 
  \begin{equation*}
    \dOTC{\mathcal{X},\mathcal{Y}} = \lim_{\delta\rightarrow 0} \dWL[\delta]{\infty}{\mathcal{X},\mathcal{Y}} = \lim 0 = 0.
  \end{equation*}
\end{itemize}

\end{proof}

We finish this section with an interesting stability result of OTM distances with
respect to the choice of the distribution $p\in\mathcal{P}(\N)$, 
which will be useful for subsequent proofs.

\begin{lemma}\label{lm:convergence of dOTM}
Let $\{p_k\}_{k\in\N}\subseteq\mathcal{P}(\N)$ be such that $\lim_{k\rightarrow \infty}d_\mathrm{TV}(p_k,p)=0$ where $d_\mathrm{TV}$ denotes the total variation distance. Then for all $\mX, \mY$, one has that 
$\lim_{k\rightarrow\infty}\dOTM{p_k}{\mX, \mY} = \dOTM{p}{\mX, \mY}.$
\end{lemma}

\begin{proof}[Proof of Lemma \ref{lm:convergence of dOTM}]
Let $\epsilon > 0$, choose $N_0$ so that there exists a sequence of random variables $T_k\sim p_k$ for any $k\geq N_0$ and $T\sim p$ such that $\Pr(T_k\neq T) \leq \epsilon$.

Then let $(X_t, Y_t)_{t\in\N}$ (resp. $(X^k_t, Y^k_t)_{t\in\N}$) be an optimal coupling independent of $T$ (resp. independent of $T_k$) that verifies $\dOTM<C>{p}{\mX, \mY} = \E\, C(X_T, Y_T)$, (resp. $\dOTM<C>{p_k}{\mX, \mY} = \E\, C(X^k_{T_k}, Y^k_{T_k})$).
Then, one has that
\begin{align*}
\dOTM<C>{p}{\mX, \mY} \leq& \E\; C(X^k_T, Y^k_T) \\
= & \Pr(X = X_k) \expect{ C(X^k_T, Y^k_T) | T = T_k} + \Pr(X \neq X_k) \expect{ C(X^k_T, Y^k_T) | T \neq T_k} \\
\leq&  \expect{ C(X^k_T, Y^k_T) | T = T_k} + \epsilon \norm{C}_\infty\\
=&  \expect{ C(X^k_{T_k}, Y^k_{T_k}) | T = T_k} + \epsilon \norm{C}_\infty\\
 = & \frac{1}{1 - \epsilon}\left(\E\; C(X^k_{T_k}, Y^k_{T_k}) - \expect{ C(X^k_{T_k}, Y^k_{T_k}) | T \neq T_k} \Pr(X \neq X_k)\right) + \epsilon\norm{C}_\infty\\
 \leq& \frac{1}{1 - \epsilon}\dOTM<C>{p_k}{\mX, \mY} +  \epsilon\left(1 + \frac{1}{1-\epsilon}\right) \norm{C}_\infty\\
 \leq &  \dOTM<C>{p_k}{\mX, \mY} + 2\epsilon\dOTM<C>{p_k}{\mX, \mY} + \epsilon (1 + 1 + 2\epsilon) \norm{C}_\infty \text{ if } \epsilon \text{ is small enough}\\
 \leq & \dOTM<C>{p_k}{\mX, \mY} + 5\epsilon \norm{C}_\infty. 
\end{align*}
This concludes the proof.
\end{proof}

\subsubsection{The OTM distance is a pseudometric}\label{sec:OTM is a metric}
We first introduce some notation. Let $(\setX,d_{\setX})$ be a pseudometric space. We let $\mathcal{M}(\setX)$ denote the collection of all Markov chains $\mathcal{X}=(\setX,m_\bullet^{\setX},\nu^{\setX})$ with state space $\setX$. 
Then, $d_{\mathrm{OTM}}^p$ induces a map as follows 
\[d_{\mathrm{OTM}}^p:\mathcal{M}(\setX)\times \mathcal{M}(\setX)\rightarrow \R_+\]
sending $(\mX_1,\mX_2)$ to $\dOTM{p}{\mX_1,\mX_2;{d_{\setX}}}$ with $d_{\setX}$ being the cost function 

\begin{proposition}[OTM distances are metrics]\label{prop:OTM_is_distance}
If $\setX = \setY$ is a pseudometric space $(\setX, d_\setX )$ 
and the cost $C$ is the pseudometric distance on $\setX$ 
(i.e., $C(x, y) = d_\setX(x, y)$).
For all \(p \in \mathcal{P}(\N)\), $d_{\mathrm{OTM}}^p:\mathcal{M}(\setX)\times \mathcal{M}(\setX)\rightarrow \R_+$ defines a pseudometric on $\mathcal{M}(\setX)$.

When $d_{\setX}$ is a metric and $p$ is fully supported on $\N$, then $d_{\mathrm{OTM},p}$ \emph{is also a metric}.
\end{proposition}

In practice the assumption that $C$ is a pseudometric is respected for example in the framework of \citet{chen2022weisfeilerlehman}, where the states have labels in a common metric space.

One can also derive a slightly more general result, to relate to p-Wassertein distances:

\begin{proposition}\label{prop:alpha_OTM_is_distance}
Let $\alpha \in [1,\infty)$.
If $\setX = \setY$ is a pseudometric space $(\setX, d_\setX )$ 
and the cost $C$ is defined as $C := d_\setX^\alpha$, i.e.,  $C(x, y) = d_\setX(x, y)^\alpha$ for any $x,y\in \setX$.

Then, for any \(p \in \mathcal{P}(\N)\), the map $h:  \mathcal{M}(\setX)\times \mathcal{M}(\setX) \rightarrow \R_+$ sending $\mX, \mY\in\mathcal{M}(\setX)$ to $\left(\dOTM{p}{\mX,\mY;{d_{\setX}^\alpha}}\right)^{\frac1\alpha}$ defines a pseudometric on $\mathcal{M}(\setX)$.

When $d_{\setX}$ is a metric and $p$ is fully supported on $\N$, then $h$ \emph{is also a metric}.
\end{proposition}

\begin{proof}[Proof of Propositions \ref{prop:OTM_is_distance} and \ref{prop:alpha_OTM_is_distance}]
  \propositionref{prop:OTM_is_distance} is a direct consequence of \propositionref{prop:alpha_OTM_is_distance} by taking $\alpha = 1$. We thus only need to prove \propositionref{prop:alpha_OTM_is_distance}.
  Let $h$ like in \propositionref{prop:alpha_OTM_is_distance}.
We prove that $h$ is a pseudometric on $\mathcal{M}(\setX)$ through the following three steps.
  \begin{itemize}
    \item Given any Markov chain $\mX$, we let $X$ be any realization of $\mX$. Then, $(X,X)$ is a Markovian coupling between $\mX$ and itself. Hence,
      \[0\leq (\dOTM{p}{\mX, \mX;\dX^\alpha})^{\frac1\alpha} \leq  (\E\, \dX^\alpha(X_T,X_T))^{\frac1\alpha}=0\]

    \item \textbf{Symmetry}: Given any two Markov chain $\mX$ and $\mY$ on $\setX$, any Markovian coupling $(X_t, Y_t)_{t\in\N}$ between $\mX$ and $\mY$ naturally (and bijectively) gives rise to a Markovian coupling $(Y_t, X_t)_{t\in\N}$ between $\mY$ and $\mX$. Hence, 
      \begin{align*}
        \dOTM{p}{\mX, \mY;\dX^\alpha}  &= \inf_{\substack{(X_t, Y_t)_{t\in\N}}}
          \E\, \dX^\alpha(X_T,Y_T)
                                 \\ &= \inf_{\substack{(X_t, Y_t)_{t\in\N}}}
          \E\, \dX^\alpha(Y_T,X_T)
                                 \\ &= \dOTM{p}{\mY, \mX;\dX^\alpha}
          \\ h(\mX, \mY) &= h(\mY, \mX)
      \end{align*}

    \item \textbf{Triangle inequality}:
      The proof of the triangle inequality is based on \cref{lm:gluing lemma}. 
      It is similar to the proof of the triangle inequality 
      in \citet[Definition 6.1]{villani2009optimal}.

    Suppose $\mX, \mY$ and $\mZ$ are three Markov chains on $\setX$. Then,
    \begin{align*}
      h(\mX, \mY)^\alpha
       &= \dOTM{p}{\mX, \mY;d_\setX^\alpha} 
      \\&= 
      \inf_{\substack{
        (X_t,Y_t)_{t\in\N} 
  }} \E\;\dX^\alpha(X_T, Y_T) \\
    &\leq  \inf_{\substack{
        (X_t,Z_t,Y_t)_{t\in\N} 
    }} \E\;\dX^\alpha(X_T, Y_T) \\
    &\leq \inf_{\substack{
        (X_t,Z_t,Y_t)_{t\in\N}
    }} \expect{\left(\dX(X_T, Z_T) + \dX(Z_T, Y_T)\right)^\alpha}
    \\ &\leq   \inf_{\substack{
        (X_t,Z_t,Y_t)_{t\in\N}
    }}
    \left((\E\,{\dX^\alpha (X_T, Z_T)})^{\frac1\alpha} +
    (\mathbb{E}\,{\dX^\alpha(Z_T, Y_T)})^{\frac1\alpha} \right)^\alpha  
    \text{ (Minkowski)}
    \end{align*}
where we are infimizing over all Markov chains $(X_t,Z_t,Y_t)_{t\in\N}$ whose marginals $(X_t,Z_t)_{t\in\N},(Z_t,Y_t)_{t\in\N}$ and $(X_t,Y_t)_{t\in\N}$ are Markovian couplings. 

Now, by Lemma~\ref{lm:gluing lemma} and \cref{rmk: optimal markovian coupling} 
we can take $(X'_t, Y'_t, Z'_t)$ a Markov chain so that

\begin{align*}
  h(\mX, \mZ) &= (\E{\dX^\alpha(X'_T, Z'_T)})^{\frac1\alpha} 
  \\h(\mZ, \mY) &= (\E{\dX^\alpha(Z'_T, Y'_T)})^{\frac1\alpha}.
\end{align*}

Hence
\begin{align*}
h(\mX, \mY)^\alpha
    &\leq \inf_{\substack{(X_t,Z_t,Y_t)_{t\in\N}}}
    \left((\E{\dX^\alpha(X_T, Z_T)})^{\frac1\alpha} +
    (\E{\dX^\alpha(Z_T, Y_T)})^{\frac1\alpha} \right)^\alpha  
  \\&\leq
    \left((\E{\dX^\alpha(X'_T, Z'_T)})^{\frac1\alpha} +
    (\E{\dX^\alpha(Z'_T, Y'_T)})^{\frac1\alpha} \right)^\alpha  
  \\&\quad= \left(h(\mX, \mZ) + h(\mZ, \mY)\right)^\alpha 
  \\ h(\mX, \mY) &\leq  h(\mX, \mZ) + h(\mZ, \mY).
\end{align*}
\end{itemize}

Now, for the second part of the statement, assume that $p$ has full support on $\N$ and $d_{\setX}$ is a metric.  
Assume further that $\dOTM<\dX^\alpha>{p}{\mX, \mY}=0$. Let $(X_t,Y_t)_{t\in\N}$ be an optimal Markovian coupling ((cf. Remark \ref{rmk: optimal markovian coupling}) and let $T\sim p$ be an independent random variable. 
Then, 
\[\expect{\dX^\alpha(X_T,Y_T)}=\sum_{t=0}^\infty p(t)\expect{\dX^\alpha(X_t,Y_t)}=0.\]
This implies that $\expect{\dX^\alpha(X_t,Y_t)}=0$ for all $t\in\N$. Since $\dX$ is a metric, we have that $X_t=Y_t$ for all $t\in\N$ and this means that $\mX$ and $\mY$ are isomorphic to each other.
\end{proof}

\subsection{Proofs and Technical Details from Section \ref{sec:wldelta}}\label{seq:wldeltaproofs}

\begin{proof}[Proof of Proposition~\ref{prop:wlreginfty}]
Notice that $\lim_{k\rightarrow\infty}d_\mathrm{TV}(p_\delta^k,p_\delta^\infty)=0$. Then, the proof follows from Lemma \ref{lm:convergence of dOTM}.
\end{proof}

\begin{proof}[Proof of Proposition~\ref{prop:wlreg_recursive}]\label{proof:recursive} 
We prove by induction the property that, 
\begin{equation}
  C^{\delta, (k)}_{ij} = 
  \inf_{\substack{
  (X_t, Y_t)_{t\in\N} \in \Pi((\setX,m^{\setX}_\bullet,\delta_i), (\setY,m^{\setY}_\bullet,\delta_j)) }}
  \E \left( \sum_{t=0}^{k-1} \delta (1 - \delta)^t C(X_{t}, Y_{t}) 
  + (1 - \delta)^k C(X_{k}, Y_{k})\right)
\label{eq:recurrence1}\end{equation}
and furthermore, there exists an optimal Markovian coupling for each $i,j$ that shares the same transition kernels  $(m^{\setX\setY,(t),k}_{\bullet \bullet})_{t\geq 0}$, where $t$ denotes the step number.

The case when $k=0$ is trivial as by definition  $C^{\delta, (0)} = C$. 
Now, suppose that \cref{eq:recurrence1} holds true for some $k\in \N$. Then, let us prove that the equation holds for $k+1$.

By definition, we have that $C^{\delta, (k+1)}_{ij} = \delta C_{ij} + (1 - \delta)\, \dW<C^{\delta, (k)}>{m^{\setX}_i, m^{\setY}_j}$. 
We let $m_{ij}^{\setX\setY,(0),k+1}\in\mathcal{C}(m_i^\setX,m_j^\setY)$ be an optimal coupling for each $i\in\setX,j\in\setY$. 
Then, we expand upon $m^{\setX\setY,(0),k+1}_{\bullet\bullet}$ to define a new set of transition kernels $(m^{\setX\setY,(t),k+1}_{\bullet\bullet})_{t\geq 0}$ such that $m^{\setX\setY,(t),k+1}_{\bullet\bullet}:=m^{\setX\setY,(t-1),k}_{\bullet\bullet}$ for all $t>0$.
Now, for any $i,j$, let $(X_t, Y_t)_{t\in\N} \in \Pi((\setX,m^{\setX}_\bullet,\delta_i), (\setY,m^{\setY}_\bullet,\delta_j))$ be a Markovian coupling with the prescribed transition kernels $(m^{\setX\setY,(t),k+1}_{\bullet\bullet})_{t\geq 0}$.
Then, 
\begin{align*}%
  &C^{\delta, (k+1)}_{ij} \\
  &= \delta C_{ij} + (1 - \delta)\, \dW<C^{\delta, (k)}>{m^{\setX}_i, m^{\setY}_j} \\
         &= \delta C(X_0,Y_0) 
          +  (1 - \delta) \expect*{C^{\delta, (k)}(X_1,Y_1)} \\
         &= \delta C(X_0,Y_0) 
          +    (1 - \delta) \expect*{\expect*{C^{\delta, (k)}_{i'j'}\mid (X_1,Y_1)=(i',j')}} \\
         &= \delta C(X_0,Y_0) 
         +  
         \\
         &\quad(1 - \delta)\expect*{\expect*{
           \expect*{ \delta C_{i'j'}+
             \sum_{t=1}^{k-1} \delta (1 - \delta)^t
              C(X_{t+1}, Y_{t+1}) 
               + (1 - \delta)^k C(X_{k+1}, Y_{k+1})}\Large\mid(X_1,Y_1)=(i',j')}}\\
         &= \delta C(X_0,Y_0) 
         +  
         \\
         &\quad(1 - \delta)\expect*{
            \delta C({X_1,Y_1})+
             \sum_{t=1}^{k-1} \delta (1 - \delta)^t
              C(X_{t+1}, Y_{t+1}) 
               + (1 - \delta)^k C(X_{k+1}, Y_{k+1})}\\ 
    &= \expect*{
             \sum_{t=0}^{k} \delta (1 - \delta)^t
              C(X_{t}, Y_{t}) 
               + (1 - \delta)^{k+1} C(X_{k+1}, Y_{k+1})}.
\end{align*}
This concludes the induction step.

Hence, for all $k\in\N$ we have that
\begin{equation}
  C^{\delta, (k)}_{ij} = 
  \inf_{(X_t, Y_t)_{t\in\N} \in \Pi((\setX,m^{\setX}_\bullet,\delta_i), (\setY,m^{\setY}_\bullet,\delta_j))}
  \E C(X_{T^\delta_k}, Y_{T^\delta_k}).
\label{eq:recurrence3}\end{equation}
In this way, we have that
\begin{align*}
  \dW<C^{\delta, (k)}>{\nu^{\setX}, \nu^{\setY}}
  &= \inf_{X_0\sim \nu^{\setX}, Y_0 \sim\nu^{\setY} }\expect*{
    C^{\delta, (k)}(X_0,Y_0)} \\
    &=\inf_{X_0\sim \nu^{\setX}, Y_0 \sim\nu^{\setY} }\expect*{\expect*{
    C^{\delta, (k)}_{ij}\mid (X_0,Y_0)=(i,j)}}\\
  &= \inf_{
    (X_t, Y_t)_{t\in\N} \in \Pi((\setX,m^{\setX}_\bullet,\nu^\setX), (\setY,m^{\setY}_\bullet,\nu^\setY))}
    \E C(X_{T^\delta_k}, Y_{T^\delta_k})\\
    &= \dWL[\delta]{k}{\mX,\mY}.
\end{align*}
This concludes the proof.
\end{proof}

\begin{proof}[Proof of Proposition~\ref{prop:wlregnotconstant}]\label{proof:wlregnotconstant} 
This is actually related to general convergence results on finite discounted MDPs.
However, as our action space is infinite (see Section \ref{sec:algorithm and convergence} for the description of the relevant MDPs), we will provide a complete proof here. We also note that we use arguments similar to the one below for proving other results such as Theorem~\ref{thm:wlsinkhorncv}.

When $\delta>0$, the convergence of $C^{\delta,(k)}$ follows from the Banach fixed point theorem (see for example Lemma \ref{lm:fixed point}). Let us define $F:\R^{n \times m}\rightarrow \R^{n \times m}$ by sending $D\in\R^{n \times m}$ to $D'\in\R^{n \times m}$ such that for any $i,j$:
\[D'_{ij} := \delta C_{ij} + (1 - \delta) \dW<D>{ m^{\setX}_i, m^{\setX}_j}.\]
By Lemma~\ref{lemma:ot1lip} one has that $F$ is a  $(1 - \delta)$-Lipschitz function when considering $\infty$-norm on $\R^{n\times m}$.  
Since $1-\delta<1$ whenever $\delta>0$, the Banach fixed point theorem applies.
By definition, $C^{\delta, (k+1)} = F\left(C^{\delta, (k)} \right)$. 
Hence, using the Banach fixed point theorem, we conclude that $C^{\delta, (k)}$ converges to the unique fixed point \(C^{\delta, (\infty)}\) of $F$.

Then, by Lemma~\ref{lemma:ot1lip} again, one has that 
\[\lim_{k\rightarrow\infty}\dW<C^{\delta, (k)}>{\nu^{\setX}, \nu^{\setY}}=\dW<C^{\delta, (\infty)}>{\nu^{\setX}, \nu^{\setY}}.\]
Then, by Proposition \ref{prop:wlreginfty} and Proposition \ref{prop:wlreg_recursive} one has that
\[\dWL[\delta]{\infty}{\mX,\mY} =\dW<C^{\delta, (\infty)}>{\nu^{\setX}, \nu^{\setY}}.\]

Now, suppose that \(C^{\delta, (\infty)}\) is a constant matrix with value \(c\). Let
\(1\leq i\leq n\), \(1 \leq j \leq m\). Then, by definition of the fixed point, one has that
\begin{equation}
C^{\delta, (\infty)}_{ij} = 
\delta C_{ij} + (1 - \delta)\dW<C^{\delta, (\infty)}>{m^{\setX}_i, m^{\setY}_j}.
\label{eq:wlreg_fixpoint}\end{equation}
This implies that
$c = \delta C_{ij} + (1 - \delta) c$ and thus $C_{ij} = c$ for any $i,j$. Hence, this will be a contradiction unless $C$ is also a constant matrix.

Finally, the speed of convergence result is also obtained as a consequence of the fact that the application \(F\) is \((1 - \delta)\)-contracting. Thus, by \citet[Corollary 1.1]{pata2019fixed} we have that
\[
  \norm{C^{\delta, (k)} - C^{\delta, (\infty)}}_\infty 
  \leq \frac{(1 - \delta)^k}{\delta} 
\norm{C^{\delta, (1)} - C}_\infty\leq \frac{2(1 - \delta)^k}{\delta} 
\norm{C}_\infty.
\]
where the rightmost inequality follows from the fact that  $\norm{C^{\delta,(1)}}_\infty\leq \norm{C}_\infty$ which can be proved using an argument similar for proving \cref{eq:orderofmins}.
Finally, $|\dWL[\delta]{k}{\mX,\mY} - \dWL[\delta]{\infty}{\mX,\mY}| \leq \frac{2(1 - \delta)^k}{\delta} 
\norm{C}_\infty$ follows from Lemma \ref{lemma:ot1lip}.

The case when $\delta=0$ is dealt with in Proposition \ref{prop:Cconverges}.
\end{proof}

\begin{proof}[Proof of Proposition~\ref{prop: optimal coupling}]
By Proposition \ref{prop:wlregnotconstant}, we know that when $\delta>0$, \(C^{\delta,(k)}\) converges to the unique fixed point $C^{\delta,(\infty)}$ of \cref{eq:wlreg_costmatrix}. This implies that for any $i\in\setX$ and $j\in\setY$,
\[C^{\delta, (\infty)}_{ij} = \delta C_{ij} + (1 - \delta)\, \dW<C^{\delta, (\infty)}>{m^{\setX}_i, m^{\setY}_j}. \]
Define $m^{\setX\setY}_{\bullet\bullet}:\setX\times\setY\rightarrow \mathcal{P}(\setX\times\setY)$ by sending $(i,j)$ to an optimal coupling between $m_i^\setX$ and $m_j^\setY$ for the optimal transport problem in the equation above. 
We finally let $\nu^{\setX\setY}$ be an optimal coupling for $\dW<C^{\delta,(\infty)}>{\nu^\setX,\nu^\setY}$. 
For any $i\in\setX$ and $j\in\setY$, we construct a time homogeneous Markovian coupling $(X_t^{ij},Y_t^{ij})_{t\in\N}$ with initial distribution $\delta_i\otimes\delta_j$ and transition kernel $m^{\setX\setY}_{\bullet\bullet}$. 
Then, we let $D_{ij}:=\expect*{C(X^{ij}_{T_\delta^\infty},Y^{ij}_{T_\delta^\infty})}$. We have that
\begin{align*}
   D_{ij}=& \expect*{C(X^{ij}_{T_\delta^\infty},Y^{ij}_{T_\delta^\infty})}\\
   =&\expect*{\sum_{t = 0}^{\infty} \delta (1 - \delta)^t C(X_{t}^{ij},Y_{t}^{ij})}\\
   =&\delta C_{ij} + (1-\delta)\expect*{\sum_{t = 1}^{\infty} \delta (1 - \delta)^{t-1} C(X_{t}^{ij},Y_{t}^{ij})}\\
   =&\delta C_{ij} + (1-\delta)\sum_{t=1}^\infty \delta(1-\delta)^{t-1}\!\!\!\!\!\sum_{i_1,j_1,\ldots,i_t,j_t}\!\!  C_{i_t,j_t}m^{\setX\setY}_{i_{t-1}j_{t-1},i_tj_t}m^{\setX\setY}_{i_{t-2}j_{t-2},i_{t-1}j_{t-1}}\cdots m^{\setX\setY}_{ij,i_1j_1}\\
   =&\delta C_{ij} + \\
   &(1-\delta) \sum_{i_1,j_1}\left(\delta C_{i_1,j_1}+\sum_{t=2}^\infty\delta(1-\delta)^{t-1}\!\!\!\!\sum_{i_2,j_2,\ldots,i_t,j_t}\!\!\!\!\!  C_{i_t,j_t}m^{\setX\setY}_{i_{t-1}j_{t-1},i_tj_t}\cdots m^{\setX\setY}_{i_1j_1,i_2j_2}\right)m^{\setX\setY}_{ij,i_1j_1}\\
   =&\delta C_{ij}+(1-\delta)\E_{(i_1,j_1)\sim m^{\setX\setY}_{ij}}\,(D_{i_1j_1}).
\end{align*}
Here in the fourth equality we used \cref{eq ET in distribution}.
Then, $D$ is the fixed point (the uniqueness follows from an argument similar to the one for proving Proposition \ref{prop:wlregnotconstant}) of the equation 
\[D_{ij}=\delta C_{ij}+(1-\delta)\E_{m^{\setX\setY}_{ij}}\,(D_{i_1j_1}), \quad \forall i,j.\]
Notice by definition, $C^{\delta,(\infty)}$ is also a fixed point of the equation above. Then, $C^{\delta,(\infty)}=D$.
Hence, if one constructs a time homogeneous Markovian coupling $(X_t,Y_t)_{t\in\N}$ with initial distribution $\nu^{\setX\setY}$ and transition kernel $m^{\setX\setY}_{\bullet\bullet}$. 
Then,
\begin{align*}
   & \expect*{C(X_{T_\delta^\infty},Y_{T_\delta^\infty})}=\expect*{\expect*{C(X_{T_\delta^\infty},Y_{T_\delta^\infty})\mid (X_0,Y_0)=(i,j)}}\\
   &=\E_{(i,j)\sim\nu^{\setX\setY}}\left({D_{ij}}\right)=\E_{(i,j)\sim\nu^{\setX\setY}}\left({C^{\delta,(\infty)}_{ij}}\right)=\dWL[\delta]{\infty}{\mathcal{X},\mathcal{Y}}.
\end{align*}
This concludes the proof.
\end{proof}

\begin{proof}[Proof of Theorem~\ref{thm:regwlcv}]
Choose a sequence $\delta_n\rightarrow 0$ such that
\[\lim_{n\rightarrow\infty}\dWL<C>[\delta_n]{\infty}{\mathcal{X},\mathcal{Y}}=\liminf_{\delta\rightarrow 0}\dWL<C>[\delta]{\infty}{\mathcal{X},\mathcal{Y}}.\]
By Proposition~\ref{prop: optimal coupling}, $\dWL<C>[\delta_n]{\infty}{\mathcal{X},\mathcal{Y}}$ can be obtained by an optimal time homogeneous Markovian coupling, which is determined by a transition kernel matrix \(P_n \in \mathcal{T}\) and an initial distribution vector \(\pi_n \in \mathcal{P}\), where
\(\mathcal{T} = \left\{P \in [0,1]^{\abs{X}\abs{Y} \times \abs{X}\abs{Y}}, P\1 = \1 \right\}\) and 
\(\mathcal{P} = \left\{\pi \in [0,1]^{\abs{X}\abs{Y}}, \1^{\mathrm{T}} \pi = 1 \right\}\), where $\1$ is the vector containing all ones. $\mathcal{P}$ and $\mathcal{T}$ are both compact spaces.
 Up to a choice of subsequence, we assume that 
\begin{itemize}
    \item $P_n$ converges to $P$ in $\ell_\infty$ norm, where $P$ is itself a transition kernel matrix;
    \item The limit $\lim_{n\rightarrow \infty}\delta\sum_{t=0}^\infty(1-\delta)^t(P_n^t)^\mathrm{T}\pi_n$ exists and is denoted by $\mu$. Obviously, $\mu\in\mathcal{P}$.
\end{itemize}
Now, we have that
\begin{align*}
    P^\mathrm{T}\mu &= \lim_{n\rightarrow \infty}\delta_n\sum_{t=0}^\infty(1-\delta_n)^tP^\mathrm{T}(P_n^t)^\mathrm{T}\pi_n\\
    &=\lim_{n\rightarrow \infty}\delta_n\sum_{t=0}^\infty(1-\delta_n)^t(P_n^{t+1})^\mathrm{T}\pi_n\\
    &=\lim_{n\rightarrow \infty}\frac{\delta_n}{1-\delta_n}\sum_{i=1}^\infty(1-\delta_n)^t(P_n^{t})^\mathrm{T}\pi_n\\
    &=\lim_{n\rightarrow \infty}\delta_n\left(\sum_{t=0}^\infty(1-\delta_n)^t(P_n^{t})^\mathrm{T}-P_n^\mathrm{T}\right)\pi_n\\
    &=\lim_{n\rightarrow \infty}\delta_n\sum_{t=0}^\infty(1-\delta_n)^t(P_n^{t})^\mathrm{T}\pi_n-\lim_{n\rightarrow \infty}\delta_nP_n^\mathrm{T}\pi_n\\
    &=\mu-0=\mu.
\end{align*}
Note that we have used the fact that $\delta_n\rightarrow 0$ several times in the derivation above.
This means that $\mu$ is stationary w.r.t. the transition kernel $P$. Moreover, by assumption that $\nu_X$ and $\nu_Y$ are stationary, it is easy to check that $\mu$, regarded as a probability measure still denote by $\mu$, is a coupling: $\mu\in\mathcal{C}(\nu_X,\nu_Y)$.
In this way, there exists a time homogeneous Markovian coupling $(X_t,Y_t)_{t\in\N}$ with transition kernel matrix $P$ and with a stationary initial distribution $\mu$. Hence,
\begin{align*}
\liminf_{\delta\rightarrow 0}\dWL<C>[\delta]{\infty}{\mathcal{X},\mathcal{Y}}&=\lim_{n\rightarrow\infty}\dWL<C>[\delta_n]{\infty}{\mathcal{X},\mathcal{Y}}\\
&=\mathbb{E}\,C(X_0,Y_0)\geq \dOTC<C>{\mX,\mY}.
\end{align*}
On the contrary, we know from Proposition \ref{prop:dotc_upper_bound} that $\limsup_{\delta\rightarrow 0}\dWL<C>[\delta]{\infty}{\mathcal{X},\mathcal{Y}}\leq \dOTC<C>{\mX,\mY}$ and this concludes the proof. 
 \end{proof}

\subsection{Proofs and Technical Details from Section \ref{differentiation}}\label{sec:regularized discounted WL}

\subsubsection{The Definition and Basic Properties}
We first provide a precise definition of the entropy-regularized optimal transport, including its primal and dual solution, 
which will later be useful to compute its gradient.

\begin{definition}[Entropy-regularized optimal transport]\label{def:regOT}
Remember that, using the same notations as in \cref{eq:wassersteindef}, given $\epsilon\geq 0$, the ($\epsilon$-)regularized OT problem is defined as:
\begin{equation}
\dW^{\epsilon}<C>{\alpha,\beta}:=\min_{(X, Y)\in \mathcal{C}(\alpha, \beta)} \E\; C(X,Y) - \epsilon H(X, Y). \label{eq:sinkhorndef}
\end{equation}
Where $H$ denotes the entropy function, i.e., $H(X,Y) := -\sum_{i\in\setX,j\in\setY}P_{ij}\log(P_{ij})$, where $P_{ij}:=\Pr(X=i,Y=j)$.

The distribution $P\in \R_+^{\abs{\setX} \times \abs{\setY}}$ of the 
optimal coupling verifying the minimum is called the \emph{primal solution}.

Solving this problem is usually done using Sinkhorn’s algorithm, an iterative algorithm described by \citet[Chapter 4.2]{peyreComputationalOT2018} — or a variant of described in \citet[Chapter 4.4]{peyreComputationalOT2018}. The latter algorithm ends up computing as a byproduct the so-called "\emph{dual solutions}" 
$f\in \R^{\abs{\setX}}, g\in\R^{\abs{\setY}}$, 
which are the solutions to the following dual optimization problem:
\begin{equation}
  \argmax_{f\in \R^{\abs{\setX}}, g\in\R^{\abs{\setY}}} 
  \inner{f}{\alpha} + \inner{g}{\beta}
  - \epsilon  \inner{e^{-f/\epsilon}}{K e^{-g/\epsilon}},
  \label{eq:sinkhorndualdef}
\end{equation}
where the matrix $K$ is defined by $K_{ij} := e^{-C_{ij}/\epsilon}$.
\end{definition}

We then provide a precise definition of the entropy-regularized discounted WL distance.
\begin{definition}[\gls{dWLepsilondelta}]\label{def:entropy dWL}
Analogous to Proposition~\ref{prop:wlreg_recursive}, let $C$ denote the cost matrix, $\delta$ be the discount factor, and $\epsilon$ be the entropy-regularization parameter. We recursively define matrices $C^{\epsilon,\delta, (l)}$ for $l=0,\ldots,k$ as follows:
\begin{gather}
C^{\epsilon, \delta, (0)}_{ij} = C_{ij}\label{eq:wlregsink_costmatrix0}\\
C^{\epsilon, \delta, (l)}_{ij} = 
\delta C_{ij} + (1 - \delta)\, \dW^{\epsilon}<C^{\epsilon,\delta, (l-1)}_{ij}>{m^{\setX}_i, m^{\setY}_j} .\label{eq:wlregsink_costmatrix}
\end{gather}
Then, the $\delta$-discounted entropy-regularized WL distance of depth $k$ is defined as follows
\begin{equation}
\dWL<C>[\delta][\epsilon]{k}{\mX,\mY} = \dW^\epsilon<C^{\epsilon,\delta, (k)}>{\nu^{\setX}, \nu^{\setY}}.
\label{eq:wlregk-entropy}\end{equation}\end{definition}

We note that the matrices defined above satisfy some convergence properties similar to those in Proposition \ref{prop:wlregnotconstant} for $C^{\delta,(k)}$:

\begin{proposition}[Convergence of \(C^{\epsilon,\delta,(k)}\)]\label{prop:wlregnotconstant-sinkhorn} 
For any $\delta\in(0,1]$ and any $\epsilon\geq 0$, the matrix \(C^{\epsilon,\delta,(k)}\) converges as $k\rightarrow\infty$. In particular, \(C^{\epsilon,\delta,(k)}\) converges to the unique fixed point $C^{\epsilon,\delta,(\infty)}$ of \cref{eq:wlregsink_costmatrix}.
Moreover, $C^{\epsilon,\delta,(k)}$ converges at rate   
\[\norm{C^{\epsilon,\delta,(k)} - C^{\epsilon,\delta,(\infty)}}_\infty \leq \frac{(1 - \delta)^k}{\delta} 
(2\norm{C}_\infty+\epsilon \log(nm)),\]
where $n:=|\setX|$ and $m:=|\setY|$.
\end{proposition}

\begin{proposition}\label{prop:limit regularized WL}
The limit $\dWL[\delta][\epsilon]{\infty}{\mX,\mY} :=\lim_{k\rightarrow\infty}\dWL[\delta][\epsilon]{k}{\mX,\mY}$ exists and in fact,
\begin{equation}
\dWL[\delta][\epsilon]{\infty}{\mX,\mY} =  \dW^\epsilon<C^{\epsilon,\delta, (\infty)}>{\nu^{\setX}, \nu^{\setY}}.
\label{eq:wlregksink-infty}\end{equation}
Furthermore, one has that
\[|\dWL[\delta][\epsilon]{k}{\mX,\mY} - \dWL[\delta][\epsilon]{\infty}{\mX,\mY}| \leq \frac{(1 - \delta)^k}{\delta} 
(2\norm{C}_\infty+\epsilon \log(nm)).\]
\end{proposition}
The proofs of the two propositions above follow essentially the same arguments in the proof of Proposition \ref{prop:wlregnotconstant} for the case when $\delta>0$. The extra term $\epsilon nm$ is from the fact that the entropy function satisfies that $|H|\leq \log(nm)$. We omit the proofs here.

\subsubsection{Proofs} 
\begin{proof}[Proof of Theorem \ref{thm:wlsinkhorncv}]
\arxivonly{\item}
\paragraph{Convergence}
When $k$ is finite, the proof follows from the convergence of \gls{sinkhorn} to regular optimal transport \cite[Property 1]{cuturi2013sinkhorn}

When $k=\infty$, the proof is based on the property of Banach fixed points in Lemma~\ref{lm:fixed point}. Consider the space $Z=\R^{n\times m}_{\geq 0}$ endowed with $\ell^\infty$ distance where $n:=|\setX|$ and $m:=|\setY|$. Let $\Lambda = [0,\infty)$. 
Let $F:Z\times \Lambda\rightarrow Z$ be defined by 
\[F(A,\epsilon):=\delta C + (1-\delta)\left( \dW^{\epsilon}<A>{m^{\setX}_i, m^{\setY}_j}\right)_{1\leq i\leq n,1\leq j\leq m}.\]

Now, consider $A,B\in Z$. Then,
\begin{align*}
   \|F(A,\epsilon)-F(B,\epsilon)\|_\infty&=(1-\delta)\left\|\left( \dW^{\epsilon}<A>{m^{\setX}_i, m^{\setY}_j}-\dW^{\epsilon}<B>{m^{\setX}_i, m^{\setY}_j}\right)_{1\leq i\leq n,1\leq j\leq m}\right\|_\infty\\
   &= (1-\delta)\max_{i,j}\left| \dW^{\epsilon}<A>{m^{\setX}_i, m^{\setY}_j}-\dW^{\epsilon}<B>{m^{\setX}_i, m^{\setY}_j}\right|.
\end{align*}

Hence, using Lemma~\ref{lemma:ot1lip}, we have that for any given $i$ and $j$,
\[|\dW^{\epsilon}<A>{m^{\setX}_i, m^{\setY}_j}-\dW^{\epsilon}<B>{m^{\setX}_i, m^{\setY}_j}|\leq \|A-B\|_\infty.\]
Therefore,
\begin{align*}
    \|F(A,\epsilon)-F(B,\epsilon)\|_\infty&=(1-\delta)\left\|\left( \dW^{\epsilon}<A>{m^{\setX}_i, m^{\setY}_j}-\dW^{\epsilon}<B>{m^{\setX}_i, m^{\setY}_j}\right)_{1\leq i\leq n,1\leq j\leq m}\right\|_\infty\\
    &\leq (1-\delta)\|A-B\|_\infty.
\end{align*}
Now, by Lemma \ref{lm:fixed point}, Proposition \ref{prop:wlregnotconstant-sinkhorn} and \citep[Proposition 4.1]{peyreComputationalOT2018}, one has that $C^{\epsilon,\delta,(\infty)}$, as the fixed point for $F_\epsilon$, is continuous w.r.t. $\epsilon$. Hence, by Lemma \ref{lemma:ot1lip} and Proposition \ref{prop:limit regularized WL}, one has that 
\[\dWL[\delta][\epsilon]{\infty}{\mX,\mY}=\dW^\epsilon<C^{\epsilon,\delta, (\infty)}>{\nu^{\setX}, \nu^{\setY}} \xrightarrow{\epsilon \to 0} \dW<C^{\delta, (\infty)}>{\nu^{\setX}, \nu^{\setY}}=\dWL[\delta]{\infty}{\mX,\mY}.\]

\paragraph{Convergence rate}
To prove the bound, denote by $\phi(\epsilon)$ the fixed point of $A \mapsto F(A, \epsilon)$ for any $\epsilon\in [0,\infty)$.

We first note that for any probability measures $\alpha\in P(\setX)$ and $\beta\in P(\setY)$ and any cost matrix $C:X\times Y\to \mathbb{R}$, we have that 
$$|\dW^\epsilon{\alpha,\beta;C}-\dW{\alpha,\beta;C}|\leq \epsilon\log{nm}.$$
To see this, let $(X,Y)$ be an optimal coupling for $\dW^\epsilon{\alpha,\beta;C}$, then
\begin{align*}
  |\dW^\epsilon{\alpha,\beta;C}-\dW{\alpha,\beta;C}|\leq \E C(X,Y)-(\E C(X,Y)-\epsilon H(X,Y))=\epsilon H(X,Y)\leq \epsilon\log{nm}.
\end{align*}

Then
\begin{align*}
  \norm{F(A, \epsilon) - F(A, 0)}_\infty
    &= (1-\delta)\norm{(\dW^\epsilon<A>{m^\setX_i, m^\setY_j}-\dW<A>{m^\setX_i, m^\setY_j})_{1\leq i\leq n,1\leq j\leq m}}_\infty\\
    &\leq \epsilon(1-\delta)\log{nm}.
\end{align*}

Thus
\begin{align*}
  \norm{\phi(0)-\phi(\epsilon)}_\infty &= \norm{F(\phi(0),0)-F(\phi(\epsilon),\epsilon)}_\infty\\
   &\leq \norm{F(\phi(0),0)-F(\phi(\epsilon),0)}_\infty
    + \norm{F(\phi(\epsilon),0)-F(\phi(\epsilon),\epsilon)}_\infty\\
   &\leq (1-\delta)\norm{\phi(0)-\phi(\epsilon)}_\infty + \epsilon(1-\delta)\log{nm}.
\end{align*}

Therefore, $\|\phi(0)-\phi(\epsilon)\|_\infty\leq \frac{\epsilon(1-\delta)}{\delta}\log{nm} $.

Hence we have the following non-asymptotic bound between the regularized OTM distance and the original OTM distance.

\begin{align*}
    &\abs{\dWL[\delta][\epsilon]{\infty}{\mX, \mY}-\dWL[\delta]{\infty}{\mX, \mY}}
    =\abs{\dW^\epsilon<\phi(\epsilon)>{\nu^\setX,\nu^\setY} - \dW<\phi(0)>{\nu^\setX,\nu^\setY}}
 \\ &\leq \abs{\dW^\epsilon<\phi(\epsilon)>{\nu^\setX,\nu^\setY} - \dW^\epsilon<\phi(0)>{\nu^\setX,\nu^\setY}}
 + \abs{\dW^\epsilon<\phi(0)>{\nu^\setX,\nu^\setY} - \dW<\phi(0)>{\nu^\setX,\nu^\setY}}
 \\ &\leq \norm{\phi(0)-\phi(\epsilon)}_\infty + \epsilon\log{nm}
 \\ &\leq \frac{\epsilon(1-\delta)}{\delta}\log{nm} + \epsilon\log{nm}
 \\ &\qquad= \frac{\epsilon(1-\delta + \delta) }{\delta}\log{nm}
 \\ &\qquad= \frac\epsilon\delta\log{nm}.
\end{align*}
\end{proof}

\begin{proof}[Proof of Theorem \ref{thm:computation of gradients}]\label{proof:differentiation} 
Consider the space $Z=\R^{n\times m}_{+}$ endowed with $\ell^\infty$ distance where $n:=|\setX|$ and $m:=|\setY|$. Let $\Lambda = M_{n}\times M_{m}\times \R^{n\times m}_+$, where \(M_k = \{ M \in \R^{k\times k}_+:\, \forall i, \sum_j M_{ij} = 1\}\).
Let $F:Z\times \Lambda\rightarrow Z$ be defined by 
\[F(A,M^1,M^2,C):=\delta C + (1-\delta)( \dW^{\epsilon}<A>{M^1_i, M^2_j})_{1\leq i\leq n,1\leq j\leq m}.\]
Recall that $C^{\epsilon,\delta, (\infty)}_{ij}$ satisfies the following equation:
\begin{equation}\label{eq:fixed point in proof}
C^{\epsilon,\delta, (\infty)}_{ij} = 
\delta C_{ij} + (1 - \delta) \dW^\epsilon<C^{\epsilon,\delta, (\infty)}>{m^{\setX}_i, m^{\setY}_j}.
\end{equation}
In other words, $C^{\epsilon,\delta, (\infty)}$ is a fixed point of $F$. 
By Lemma~\ref{lm:fixed_diff}, one has that $C^{\epsilon,\delta, (\infty)}$ is differentiable on the interior of its definition domain (which is a manifold without boundary).
We could also use that lemma directly to compute the gradient, but, for clarity, we will still do the computations in coordinates
to show how the gradient is computed in practice.

We differentiate $C^{\epsilon,\delta, (\infty)}_{ij}$ on both sides of \cref{eq:fixed point in proof} below  (using Einstein summation
convention):
\begin{align*}
\Delta_{ij}^{kl} = & \frac{\partial C^{\epsilon,\delta,(\infty)}_{ij}}{\partial C_{kl}} = \delta \1_{(i, j) = (k, l)} \nonumber + (1 - \delta) \frac{\partial}{\partial C_{kl}} 
    \dW<C^{\epsilon,\delta,(\infty)}>{m^{\setX}_i, m^{\setY}_j}\\
= & \delta \1_{(i, j) = (k, l)} + (1 - \delta) \frac{\partial \dW<C^{\epsilon,\delta,(\infty)}>{m^{\setX}_i, m^{\setY}_j} }{\partial C^{\delta,(\infty)}_{\alpha\beta}}
  \frac{\partial C^{\delta,(\infty)}_{\alpha\beta}}{\partial C_{kl}} \nonumber \\
= & \delta \1_{(i, j) = (k, l)} \nonumber  + (1 - \delta) P^{\alpha\beta}_{ij} \Delta^{kl}_{\alpha\beta}. 
\end{align*}\label{eq:fixpointdiff1}
By identifying tensors with
\(nm\times nm\)-square matrices via flattening together the dimensions (resp
codimensions), one has that
\[\Delta =  \delta I_{nm} + (1 - \delta) P \Delta.\]
Hence,
\begin{equation*} ( I_{nm} -  (1 - \delta) P) \Delta  =  \delta I_{nm}, 
\end{equation*}
and therefore, we have that
\begin{equation*}\Delta  =  \delta ( I_{nm} -  (1 - \delta) P)^{-1} .\label{eq:fixpointdiff3}\end{equation*}
Here the matrix \(K =  I_{nm} - (1 - \delta) P\) is invertible
because it is strictly diagonally dominant: For any $1\leq i\leq n$ and $1\leq j\leq m$, one has that
\begin{align*}
K_{ij}^{ij} = & 1 - (1 - \delta)P_{ij}^{ij}= \sum_{k, l} P_{ij}^{kl} - (1 - \delta)P_{ij}^{ij} \nonumber \\
= & \sum_{(k, l) \neq (i, j)} P_{ij}^{kl} + \delta P_{ij}^{ij} \nonumber >\sum_{(k, l) \neq (i, j)} P_{ij}^{kl} = \sum_{(k, l) \neq (i, j)} \abs{K_{ij}^{kl}},
\end{align*}
where in the second equality we used the fact that \(\sum_{k, l} P_{ij}^{kl} = 1\) since $P_{ij}$ represents a
coupling.

We apply the same method for calculating \(\Gamma\): differentiating the fixpoint
equation (cf. \cref{eq:fixed point in proof}) on both sides, we have that
\begin{align*}
\Gamma_{ij}^{kk'} = & \frac{\partial C^{\epsilon,\delta,(\infty)}_{ij}}{\partial m^{\setX}_{kk'}} \\
= & (1-\delta) \left(
\frac{\partial\dW<C^{\epsilon,\delta,(\infty)}>{m^{\setX}_i, m^{\setY}_j}}{\partial m^{\setX}_{kk'}} 
+ \frac{\partial\dW<C^{\epsilon,\delta,(\infty)}>{m^{\setX}_i, m^{\setY}_j}}{\partial C^{\epsilon,\delta,(\infty)}_{\alpha\beta}}
\frac{\partial C^{\epsilon,\delta,(\infty)}_{\alpha\beta}}{\partial m^{\setX}_{kk'}}
\right) \\
= &(1-\delta)(\1_{k = i} f_{ij}^{k'}  + P_{ij}^{\alpha\beta} \Gamma_{\alpha\beta}^{kk'}). 
\end{align*}\label{eq:fixpointdiff4}
Hence, $
\Gamma = (1 - \delta) (F + P\Gamma)$ and thus
 $( I_{nm} -  (1 - \delta) P) \Gamma =  (1 - \delta) F$. By invertibility of $K=I_{nm} -  (1 - \delta) P$ again, one has that
\begin{equation*}\Gamma =  (1 - \delta)  ( I_{nm} -  (1 - \delta) P)^{-1} F .\end{equation*}
This concludes the proof.
\end{proof}

\printglossaries

\end{document}